\documentclass[journal,romanappendices]{IEEEtran}
\ifCLASSINFOpdf
\else
\fi

\usepackage{dsfont} % for indicator function
\usepackage{bbm} % for indicator function - uses Type 3 fonts

\usepackage{cite}
\usepackage[dvips]{graphicx} % other class options: dvipsone or dvipdf
\graphicspath{{../figures/}} % the path(s) where the graphic files are
\DeclareGraphicsExtensions{.eps}  % the extensions of graphic files

\usepackage[cmex10]{amsmath}
\usepackage{amsfonts}
\usepackage{amssymb}
\usepackage{makeidx}
\usepackage{pifont}
\usepackage{latexsym}
\usepackage[mathcal]{euscript}
\usepackage[vcentermath]{youngtab}
\usepackage{amsthm}
\usepackage{multirow}
\usepackage{array} 

\usepackage[caption=false,font=footnotesize]{subfig}

\usepackage{color}
\usepackage{pstricks}
\usepackage{comment}
\usepackage{wrapfig}
\usepackage{enumitem}

\usepackage{url}
\usepackage[%
hyperfigures=true,%
breaklinks=true,%
colorlinks=true,%
citecolor=green,%
linkcolor=blue%
]{hyperref}

	\newtheoremstyle{myplain}% name
	  {}%      Space above, empty = `usual value'
	  {}%      Space below
	  {\itshape}% Body font
	  {}%         Indent amount (empty = no indent, \parindent = para indent)
	  {\bfseries}% Thm head font
	  {}%        Punctuation after thm head
	  {5pt plus 1pt minus 1pt}% Space after thm head: \newline = linebreak
	  {}%         Thm head spec

	\newtheoremstyle{mydefinition}% name
	  {}%      Space above, empty = `usual value'
	  {}%      Space below
	  {\normalfont}% Body font
	  {}%         Indent amount (empty = no indent, \parindent = para indent)
	  {\bfseries}% Thm head font
	  {}%        Punctuation after thm head
	  {5pt plus 1pt minus 1pt}% Space after thm head: \newline = linebreak
	  {}%         Thm head spec
  
	\theoremstyle{myplain}

	\newtheorem{problem}{Problem}
	
	\newtheorem{property}{Property}
	\newtheorem{theorem}{Theorem}
	\newtheorem{lemma}{Lemma}
	\newtheorem{proposition}{Proposition}

	\theoremstyle{mydefinition}
	\newtheorem{definition}{Definition}

	\renewenvironment{proof}[1][\textit{Proof}]{\begin{trivlist}
		\item[\hskip \labelsep {\bfseries #1}]}{\end{trivlist}}
    \newcommand{\prl}[1] 		{\left(#1\right)}
    \newcommand{\brl}[1] 		{\left[ #1 \right]}
    \newcommand{\crl}[1] 		{\left\{#1\right\}}
    
    \newcommand{\R}             {\mathbb{R}} % Reals
 	\newcommand{\N}         		{\mathbb{N}} % Natural Numbers
 	\newcommand{\Sp}        		{\mathbb{S}} % Sphere
 	\newcommand{\C}      		{\mathbb{C}} % Complex Numbers
  	\newcommand{\card}[1]		{\left| #1 \right|} % cardinality
	\newcommand{\norm}[1] 		{\left \| #1 \right \|}

    \newcommand{\vectprod}[2]    { \tr{#1} #2}
    \newcommand{\tr}[1]         {{#1}^\mathrm{T}}
     
	\newcommand{\argmax}{\operatornamewithlimits{arg\ max}} % with space

	\newcommand{\mat}[1] 		{\mathrm{\mathbf{#1}}}

	\newcommand{\vect}[1] 		{\mathrm{#1}}
	\newcommand{\vectbf}[1]		{\mathrm{\mathbf{#1}}}

    \newcommand{\threevecT} [3] { \left[ \begin{array}{@{}c@{}@{}c@{}@{}c@{}} 
          #1, & #2, & #3 \end{array} \right] }

    \newcommand{\refeqn}[1]			{(\ref{#1})}
    \newcommand{\reffig}[1]			{Fig. \ref{#1}}
    \newcommand{\refsec}[1]			{Section \ref{#1}}
    \newcommand{\refapp}[1]			{Appendix \ref{#1}}
	\newcommand{\reftab}[1]			{Table \ref{#1}}
	\newcommand{\refthm}[1]			{Theorem \ref{#1}}
	\newcommand{\refprop}[1]		{Proposition \ref{#1}}
	
	\newcommand{\reflem}[1]			{Lemma \ref{#1}}
	\newcommand{\refdef}[1]			{Definition \ref{#1}}

	\newcommand{\refprob}[1]			{Problem \ref{#1}}
	\newcommand{\refP}[1]            {Property \ref{#1}}

    \newcommand{\ldf}				{  \: {\mathbf :\! = }  \: }

    \newcommand{\sqz}[1]            {\! #1 \!}

\newcommand{\bigO}[1] 			{\mathrm{O}\prl{#1}}

\newcommand{\indexset}			{J} % Leaf Labels

\newcommand{\radius}            {r}
\newcommand{\radiusCL}[1]       {\radius_{#1}} % Radius of a partial configuration
\newcommand{\radiusDD}[1]       {\radius_{#1}} % Radius of consensus ball

\newcommand{\confspace}[1]       {\mathtt{Conf}\!\prl{#1}}
\newcommand{\cRdJr}  			{\confspace{\R^d\!,\indexset,\mathbf{\radius}}}
\newcommand{\cRdJ} 				{\confspace{\R^d,\indexset}}
\newcommand{\state}             {x}
\newcommand{\stateA}            {x}
\newcommand{\stateB}            {y}
\newcommand{\stateC}            {z}
\newcommand{\stateD}            {w}

\newcommand{\bintreespace}     		{\mathcal{BT}}
\newcommand{\bintreetopspace}     	{\bintreespace}
\newcommand{\tree}        			{\tau}
\newcommand{\treeA}              	{\sigma}
\newcommand{\treeB}             		{\tau}
\newcommand{\treeC}             		{\gamma}

\newcommand{\HC}                		{\mathtt{HC}} % Hierarchical Clustering
\newcommand{\HCkmeans}               {\HC_{\text{2-means}}} %Hierarchical 2-means
\newcommand{\cluster}[1] 			{\mathcal{C}\prl{#1}} % Cluster Set
\newcommand{\parentCL}[1]			{\mathrm{Pr}\prl{ #1}} % Parent Cluster
\newcommand{\childCL}[1]		    		{\mathrm{Ch}\prl{#1}} % Children Cluster
\newcommand{\ancestorCL}[1]     		{\mathrm{Anc}\prl{#1}} % Ancestor Clusters
\newcommand{\descendantCL}[1]   		{\mathrm{Des}\prl{#1}} % Descendant Clusters
\newcommand{\complementLCL}[2]   	{{#1}^{- #2}} % Locally Complement Cluster
\newcommand{\stratum}[1] 		{\mathfrak{S}\!\prl{#1} } % Hierarchical Stratum
\newcommand{\stratumI}[1]       	{\mathring{\mathfrak{S}}\!\prl{#1}} 

\newcommand{\ctrd}[1]			{\mathrm{c}\prl{#1}}
\newcommand{\ctrdsym}			{\mathrm{c}}
\newcommand{\ctrdmid}[2]			{\mathrm{m}_{#1}\prl{#2}}
\newcommand{\ctrdsep}[2]			{\mathrm{s}_{#1}\prl{#2}}
\newcommand{\sepmag}[2]		    {\eta_{#1}\prl{#2}}
\newcommand{\sepmagsym}[1]		{\eta_{#1}}

\newcommand{\subpolicyset}[2]      	{\mathcal{SP}_{#1}\prl{#2}}
\newcommand{\subpreparesgraph}      	{\mathcal{PG}}
\newcommand{\subpreparesgraphS}     	{\widehat{\mathcal{PG}}}
\newcommand{\subpreparesedgeset}   	{\mathcal{E}_{\mathcal{PG}}}
\newcommand{\subpreparesedgesetS}   	{\widehat{\mathcal{E}}_{\mathcal{PG}}}
\newcommand{\f}                  {f_{\tree,\vectbf{\stateB}}}
\newcommand{\fhat}               {\hat{f}_{\tree,\vectbf{\stateB}}}
\newcommand{\fA}                 {f_A}
\newcommand{\fS}                 {f_S}
\newcommand{\fH}                 {f_H}
\newcommand{\V}                  {V}
\newcommand{\falpha}[2]          {\alpha_{#1}\prl{#2}}
\newcommand{\fphi}[2]            {\phi_{#1}\prl{#2}}
\newcommand{\fpsi}[2]			{\psi_{#1}\prl{#2}}
\newcommand{\fbeta}[2]            {\beta_{#1}\prl{#2}}
\newcommand{\domain}				{\mathcal{D}}
\newcommand{\setA}               {\mathcal{D}_A}
\newcommand{\setAhat}            {\widehat{\mathcal{D}}_A}
\newcommand{\setAhatI}           {\mathring{\widehat{\mathcal{D}}}_A}
\newcommand{\setH}               {\mathcal{D}_H}
\newcommand{\setHI}              {\mathring{\mathcal{D}}_H}
\newcommand{\setB}               {\mathcal{D}_B}
\newcommand{\priority}			{\mathtt{priority}}
\newcommand{\purt}    			{\mathcal{\indexset}} % a partition of the index set J
\newcommand{\purtset} 			{\mathcal{P}} % Partition set
\newcommand{\visitedcluster}     {\mathcal{V}_{\tree}\prl{\purt}} 
\newcommand{\h}[1]               {h_{#1}}
\newcommand{\hhat}[1]            {\hat{h}_{#1}}
\newcommand{\p}                  {p}
\newcommand{\phat}               {\hat{p}}
\newcommand{\setG}               {\mathcal{G}}

\newcommand{\portal}            		{\mathtt{Portal}}
\newcommand{\portalconf}[1]     		{{\mathtt{Port}}_{#1}}

\newcommand{\portalcenter}[1]   		{{\mathtt{Ctr}}_{#1}}
\newcommand{\portalscale}[1]    		{{\mathtt{Scl}}_{#1}}
\newcommand{\scaleconst}        		{\zeta}
\newcommand{\portalmerge}[1]    		{{\mathtt{Mrg}}_{#1}}
\newcommand{\portalseparate}[1]    	{{\mathtt{Sep}}_{#1}}
\newcommand{\NT}        				{\mathtt{NT}} % Napoleon Transformation
\newcommand{\NToff}[1]  				{\mathtt{Noff}_{#1}}  % Napoleon Offset
\newcommand{\adjgraph}[1]      	{\mathcal{A}_{#1}}
\newcommand{\adjedgeset}       	{\mathcal{E}_{\mathcal{A}}}
\newcommand{\NNI}              	{\text{NNI}}
\newcommand{\NNIgraph}      	   	{\mathcal{N}}
\newcommand{\NNIedgeset}       	{\mathcal{E}_{\mathcal{N}}}

\newcommand{\flow}              	{\varphi}
\newcommand{\lie}				{\mathcal{L}}

\begin{document}
%
% paper title
% can use linebreaks \\ within to get better formatting as desired
% Do not put math or special symbols in the title.
\title{Coordinated Robot Navigation \\ via Hierarchical Clustering}
%\title{Navigation of Euclidean Spheres \\ via Hierarchical Clustering}
%
%
% author names and IEEE memberships
% note positions of commas and nonbreaking spaces ( ~ ) LaTeX will not break
% a structure at a ~ so this keeps an author's name from being broken across
% two lines.
% use \thanks{} to gain access to the first footnote area
% a separate \thanks must be used for each paragraph as LaTeX2e's \thanks
% was not built to handle multiple paragraphs
%

\author{Omur Arslan, Dan P. Guralnik  and Daniel E. Koditschek,
\thanks{The authors are with the Department of Electrical and Systems Engineering, University of Pennsylvania, Philadelphia, PA 19104. 
E-mail: \{omur, guralnik, kod\}@seas.upenn.edu}%
\thanks{This work was supported in part by AFOSR under the CHASE MURI FA9550-10-1-0567 and in part by ONR under the HUNT MURI N00014–07–0829.}%
\thanks{A preliminary version of this paper is presented in the conference paper \cite{arslan_guralnik_kod_WAFR2014} for point particles and a certain choice of hierarchical clustering. 
In this paper, we propose a general hierarchical navigation framework for a broad class of clustering methods and disk-shaped robots.}
}

\markboth{ESE Technical Report --- \today}{}
%\markboth{IEEE Transactions on Robotics,~Vol.~??,~No.~?,~Month~2015 --- DRAFT~\today}%
%{}%{Arslan \MakeLowercase{\textit{et al.}}: Hierarchical Navigation of Euclidean Spheres}
%
%\markboth{Journal of \LaTeX\ Class Files,~Vol.~11, No.~4, December~2012}%
%{Shell \MakeLowercase{\textit{et al.}}: Bare Demo of IEEEtran.cls for Journals}
% The only time the second header will appear is for the odd numbered pages
% after the title page when using the twoside option.
% 
% *** Note that you probably will NOT want to include the author's ***
% *** name in the headers of peer review papers.                   ***
% You can use \ifCLASSOPTIONpeerreview for conditional compilation here if
% you desire.

% If you want to put a publisher's ID mark on the page you can do it like
% this:
%\IEEEpubid{0000--0000/00\$00.00~\copyright~2015 IEEE}
% Remember, if you use this you must call \IEEEpubidadjcol in the second
% column for its text to clear the IEEEpubid mark.

% use for special paper notices
%\IEEEspecialpapernotice{(Invited Paper)}

% make the title area
\maketitle

% As a general rule, do not put math, special symbols or citations
% in the abstract or keywords.
\begin{abstract}
We introduce the use of hierarchical clustering for relaxed, deterministic coordination and control of multiple robots. 
Traditionally an unsupervised learning method, hierarchical clustering offers a formalism for  identifying and representing spatially cohesive and segregated robot groups at different resolutions by relating the continuous space of configurations to the combinatorial space of trees. 
We formalize and exploit this relation, developing computationally effective reactive algorithms for navigating through the combinatorial space in concert with geometric realizations for a particular choice of hierarchical clustering method.  
These constructions yield computationally effective  vector field planners for both hierarchically invariant as well as transitional navigation in the configuration space.
We apply these methods to the centralized coordination and control of $n$ perfectly sensed and actuated Euclidean spheres in a $d$-dimensional ambient space (for arbitrary $n$ and $d$). 
Given a desired configuration supporting a desired hierarchy, we construct a hybrid controller  which is quadratic in $n$ and algebraic  in $d$ and prove that its execution brings all but a measure zero set of  initial configurations to the desired goal with the guarantee of no collisions along the way.
\end{abstract}

% Note that keywords are not normally used for peerreview papers.
\begin{IEEEkeywords}
multi-agent systems, navigation functions, formation control, swarm robots, configuration space, coordinated motion planning, hierarchical clustering, cohesion, segregation.
\end{IEEEkeywords}

% For peer review papers, you can put extra information on the cover
% page as needed:
% \ifCLASSOPTIONpeerreview
% \begin{center} \bfseries EDICS Category: 3-BBND \end{center}
% \fi
%
% For peerreview papers, this IEEEtran command inserts a page break and
% creates the second title. It will be ignored for other modes.
\IEEEpeerreviewmaketitle

%%%%%%%%%%%%%%%%%%%%%%%%%%%%%%%%
%%%%%%%%%%%%%%%%%%%%%%%%%%%%%%%%
\section{Introduction}
\label{sec.Introduction}
%%%%%%%%%%%%%%%%%%%%%%%%%%%%%%%%
%%%%%%%%%%%%%%%%%%%%%%%%%%%%%%%%

% The very first letter is a 2 line initial drop letter followed
% by the rest of the first word in caps.
% 
% form to use if the first word consists of a single letter:
% \IEEEPARstart{A}{demo} file is ....
% 
% form to use if you need the single drop letter followed by
% normal text (unknown if ever used by IEEE):
% \IEEEPARstart{A}{}demo file is ....
% 
% Some journals put the first two words in caps:
% \IEEEPARstart{T}{his demo} file is ....
% 
% Here we have the typical use of a "T" for an initial drop letter
% and "HIS" in caps to complete the first word.

\IEEEPARstart{C}{}ooperative, coordinated action and sensing can promote efficiency,
robustness, and flexibility in achieving  complex tasks such
as search and rescue, area exploration, surveillance and reconnaissance,
and warehouse management \cite{parker_SHR2008}.  
Despite significant progress in the analysis of how local rules can yield such global spatiotemporal patterns \cite{jadbabaie_lin_morse_TAC2003,nabet_leonard_couzin_levin_JNS2009,chazelle_JACM2014}, there has been strikingly less work on their specification. 
With few exceptions, the engineering literature on multirobot systems relies on task representations expressed in terms of rigidly imposed configurations --- either by absolutely targeted positions, or relative distances --- missing the intuitively substantial benefit of ignoring fine details of individual positioning, to focus control effort instead on the presumably far coarser properties of the collective pattern that matter.   
We seek a more relaxed means of specification that is sensitive to spatial distribution at multiple scales (as influencing the intensity of interactions among individuals and with their environment \cite{okubo_levin_2001}) and the identities of neighbors (as determining the capabilities of heterogeneous teams \cite{anderson_franks_BE2001}) while affording, nevertheless,  a well-formed deterministic characterization of pattern.

\IEEEpubidadjcol

\begin{figure}[tb]
\centering
\includegraphics[width=0.455\textwidth]{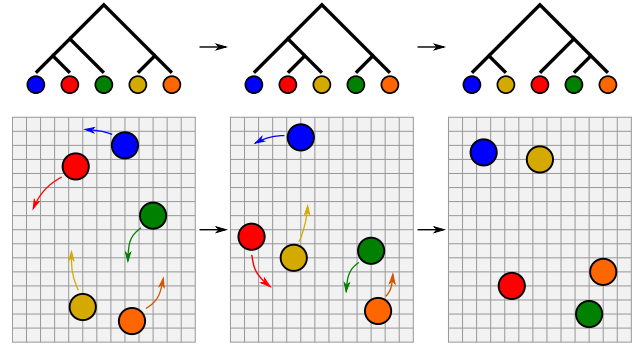} 
\vspace{-2mm}
\caption{Moving from one spatial distribution to another is generally carried through rearrangements of robot groups (clusters) at different resolution corresponding to transitions between different cluster structures (hierarchies).}
\label{fig.GroupTransitions}
\vspace{-2mm}
\end{figure}

We are led to the notion of hierarchical clustering. 
We reinterpret this classical method for unsupervised learning \cite{jain_dubes_1988} as a formalism for the specification and reactive implementation of collective mobility tasks expressed with respect to successively refined partitions of the agent set in a manner depicted in \reffig{fig.GroupTransitions}. 
There, we display three different configurations of five planar disks whose relative positions are specified by three distinct trees that represent differently nested clusters of relative proximity. 
The first configuration exhibits three distinct clusters at a  resolution in the neighborhood of  2  units of distance:  the red and the blue disks;  the yellow and the orange disks; and the solitary green disk. 
At a coarser  resolution, in the neighborhood of  4 units of distance,  the green disk has merged into the subgroup including the red and the blue disks to comprise one of only two clusters discernible at this scale, the other formed by the orange and the yellow disks.  
It is intuitively clear that this hierarchical arrangement of subgroupings will persist under significant variations in the position of each individual disk. 
It is similarly clear that the second and third configurations (and significant variations in the positions of the individual disks of both) support the very differently nested clusters represented by the second and third trees, respectively. 
In this paper, we introduce a provably correct and computationally effective machinery for specifying, controlling  invariantly to, and passing between such hierarchical clusterings at will.

As an illustration of its utility, we use this formalism to solve a specific instance of the reactive motion planning problem suggesting how the new ``relaxed'' hierarchy-sensitive layer of control can be merged with a task entailing a traditional rigidly specified goal pattern. 
Namely, for a collection of $n$  disk robots in  $\R^d$ we presume that a target hierarchy has been specified along with a goal configuration that supports it, and that the robot group is controlled by a centralized source of perfect, instantaneous information about each agent's position that can command exact instantaneous velocities for each disk. 
We present an algorithm resulting in a purely reactive hybrid dynamical system \cite{back_guckenheimer_myers_HS1993} guaranteed to bring the disk robots to both the hierarchical pattern as well as the rigidly specified instance from (almost) arbitrary initial conditions with no collisions of the disks along the way. 
Stated formally  in \reftab{tab.HNCAlgorithm}, the correctness of this algorithm  is guaranteed by \refthm{thm.HierNavAlg} whose proof appeals to the resolution of various constituent problems summarized in \reftab{tab.ProblemSolution}.
The construction is computationally effective: the number of discrete transitions grows in the worst case with the square of the number of robots, $n$; each successive discrete transition can be computed reactively (i.e., as a function of the present configuration) in time that grows linearly with the number of robots; and the formulae that define each successive vector field and guard condition are rational functions (defined by quotients of polynomials over the ambient space of degree less than $3$) entailing terms whose number grows quadratically with the number of robots.

\begin{table}[tb]
\caption{Constituent Problems of Hierarchical Robot Navigation}
\centering
\vspace{-2mm}
\begin{tabular}{|@{\hspace{1mm}}c@{\hspace{1mm}}|@{\hspace{1mm}}c@{\hspace{1mm}}|@{\hspace{1mm}}c@{\hspace{1mm}}|@{\hspace{1mm}}l@{\hspace{1mm}}|}
\hline
Problem & Solution & Theorem & Description \\
\hline
\ref{prob.HierarchyInvariant} &  \reftab{tab.HierInvNav} &  \ref{thm.HierInvNav} &   Hierarchy invariant vector field planner\\
\hline
\ref{prob.Transition} &  \reftab{tab.NNIControl}  &  \ref{thm.Transition} & Reactive navigation across  hierarchies \\
\hline
\ref{prob.Portal} &  Eqn.\refeqn{eq.portalconf} &   \ref{thm.Portal} & Cross-hierarchy geometric realization \\
\hline
\end{tabular}
\label{tab.ProblemSolution}
\end{table}

This paper is organized as follows. 
We review in the next section the  relevant background  literature: first  on reactive multirobot motion planning to relate the difficulty and importance of our sample problem to the state of the art in this field; next on the role of hierarchy in configuration spaces as explored both in biology and engineering.  
Because the notion of hierarchical clustering is a new abstraction for motion planning  we devote \refsec{sec.Abstraction} to a presentation of the key background technical ideas: first we review the relevant topological properties of configuration spaces; next the relevant topological properties of tree spaces; and, finally, prior work establishing properties of certain functions and relations between them. 
Because we feel that the specific motion planning problem we pose and solve represents a mere illustration of the larger value of this abstraction for multirobot systems we devote \refsec{sec.HierNav} to a presentation of some of the more generic tools from which our particular construction is built: first we introduce the notion of hierarchy invariant navigation; next we discuss the combinatorial  problem of hierarchy rearrangement as a graph navigation problem; and finally we interpret a subgraph of that combinatorial space as a ``prepares'' graph \cite{Burridge_Rizzi_Koditschek_1999} for  the hierarchy-invariant cover of configuration space. 
In \refsec{sec.HierNavSphere} we pose and solve the specific motion planning problem using the concepts introduced in \refsec{sec.Abstraction} and the tools introduced in \refsec{sec.HierNav}. 
\refsec{sec.Simulation} offers some numerical studies of the resulting algorithm.  
We conclude in \refsec{sec.Conclusion} with a summary of the major technical results that yield the specific contribution followed by some speculative remarks bearing on the likelihood that recent extensions of these ideas presently in progress \cite{arslan_koditschek_2014} might afford a distributed reformulation, thus addressing the first (and better explored) remarkable biological inspiration for multirobot systems.

%%%%%%%%%%%%%%%%%%%%%%%%%%%%%%%%%%%%%
%%%%%%%%%%%%%%%%%%%%%%%%%%%%%%%%%%%%%
\section{Related Literature}
\label{sec.Literature}
%%%%%%%%%%%%%%%%%%%%%%%%%%%%%%%%%%%%%
%%%%%%%%%%%%%%%%%%%%%%%%%%%%%%%%%%%%%

\subsection{Multirobot Motion Planning}

\subsubsection{Complexity}

The intrinsic complexity of multibody configurations impedes computationally effective generalizations of single-robot motion planners \cite{lavalle_PlanningAlgorithms2006,choset_etal_PrinciplesOfRobotMotion2005}. 
Coordinated motion planning of thick bodies in a compact space is computationally
hard. 
For example, moving planar rectangular objects within a rectangular box is PSPACE-hard \cite{hopcroft_schwartz_sharir_IJRR1984} and motion planning for finite planar disks in a polygonal environment is strongly NP-hard \cite{spirakis_yap_IPL1984}.  
Even determining when and how the configuration space of noncolliding spheres in a unit box is connected entails an encounter with the ancient sphere packing problem \cite{baryshnikov_bubenik_kahle_IMRN2013}.
Within the domain of reactive or vector field motion planning, it has proven deceptively hard to determine exactly this line of intractability.
Consequently, this intrinsic complexity for coordinated vector field planners is generally mitigated by either assuming objects move in an unbounded (or sufficiently large) space \cite{Tanner_Boddu_TRO2012,Dimarogonas_Loizou_Kyriakopoulos_Zavlanos_2006}, as we do in \refsec{sec.HierNavSphere}, or simply assuming conditions sufficient to guarantee connectivity between initial and goal configurations \cite{karagoz_bozma_kod_TRO2014,loizou_RSS2014}.  
On the other hand, more relaxed versions entailing (perhaps partially) homogeneous (unlabeled) specifications for interchangeable individuals have yielded computationally efficient planners in the recent literature \cite{solovey_halperin_IJRR2014,adlerEtAl_EuroCG2014,turpin_michael_kumar_ICRA2013,turpin_michael_kumar_AFR2013},
and we suspect that the cluster hierarchy abstraction may be usefully applicable to such partially labeled settings.

\subsubsection{Reactive Multirobot Motion Planning}

Since the problem of reactively navigating groups of disks was first introduced to robotics  \cite{Whitcomb_Koditschek_1991,Whitcomb_Koditschek_Cabrera_ICRA1992}, most research into vector field planners has embraced the navigation function paradigm \cite{rimon_koditschek_TRO1992}. 
A recent review of this two decade old literature is provided in \cite{Tanner_Boddu_TRO2012},  where a combination of intuitive and analytical results yields a nonsmooth centralized planner for achieving goal configurations specified up to rigid transformation.  
As noted in \cite{Tanner_Boddu_TRO2012}, the multirobot generalization of a single-agent navigation function is challenged by the violation of certain assumptions inherited from the original formulation \cite{rimon_koditschek_TRO1992}.
One such assumption is that obstacles are ``isolated'' ( i.e. nonintersecting).   
In the multirobot case, every robot encounters others as mobile obstacles, and any collision between more than two robots breaks down the isolated obstacle assumption \cite{Tanner_Boddu_TRO2012}.
In some approaches, the departure from isolated interaction has been addressed by encoding all possible collision scenarios,  yielding controllers with terms growing super-exponentially in the number of robots,  even when the workspace is not compact \cite{Dimarogonas_Loizou_Kyriakopoulos_Zavlanos_2006}.  
In contrast, our recourse to the hierarchical representation of configurations affords a computational burden growing merely quadratically in the number of agents. 
In  \cite{karagoz_bozma_kod_TRO2014}, the problem is circumvented by allowing critical points on the boundary (with no damage to the obstacle avoidance and convergence guarantees), but, as mentioned above, very conservative assumptions about the degree of separation  between agents at the goal state are required.  
In contrast, our recourse to hierarchy allows us to handle arbitrary (non-intersecting) goal configurations,  albeit our reliance upon the homotopy type of the underlying space presently precludes the consideration of a compact configuration space as formally allowed in \cite{karagoz_bozma_kod_TRO2014}.\footnote{ We conjecture  that a compact  configuration space with a free-space goal point satisfying the conditions of \cite{karagoz_bozma_kod_TRO2014} has the same homotopy type as the unbounded case we treat here.}

Another limitation of navigation function approaches is the requirement of proper parameter tuning to eliminate local minima. 
Some effort has been given to automatic adaptation of this parameter \cite{loizou_RSS2014}, and,  in principle, the original results of \cite{rimon_koditschek_TRO1992} guarantee that any monotone increasing scheme must  eventually resolve the issue of local minima,  however, this is numerically unfavorable (the Hessian of the resulting field becomes stiffer) and incurs substantial performance costs (transients must slow as the tuning parameter increases).\footnote{It bears mention in passing that partial differential equations (e.g.,  harmonic potentials \cite{loizou_CDC2011}) yield self-tuning navigation functions but these are intrinsically numerical constructions that forfeit the reactive nature of the closed form vector field planners under discussion here. 
}
In contrast, our recourse to hierarchy removes the need for any comparable tuning parameter. 

Many of the concepts and some  of the technical constructions we develop here were presented in preliminary form in the conference paper  \cite{arslan_guralnik_kod_WAFR2014}, building on the initial  results of the conference paper \cite{arslanEtAl_Allerton2012}.  
This presentation gives a unified view of  the detailed results (with some tutorial background) and contributes a major new extension by generalizing the construction of \cite{arslan_guralnik_kod_WAFR2014} from  point particles  to thickened disks of non-zero radius (necessitating a more involved version of the hierarchy invariant fields in \refsec{sec.HierInvNav}).
  
\subsection{The Use of Hierarchies as Organizational Models}

\subsubsection{Hierarchy in Configuration Space}

That a hierarchy of proximities might play a key role in computationally efficient coordinated motion planning had already been hinted at in early work on this problem \cite{Liu_Kuroda_Naniwa_Noborio_Arimoto_1989,Liu_Arimoto_Noborio_1991,peasgood_clark_mcphee_TRO2008}.
Partial hierarchies that limit the combinatorial growth of complexity have been explicitly applied algorithmically to organize and simplify the systematic enumeration of cluster adjacencies in the configuration space \cite{ayanian_kumar_koditschek_RR2011}. 
Moreover, hierarchical discrete abstraction methods are successfully applied for scalable steering of a large number of robots as a group all together by controlling the group shape \cite{belta_kumar_TRO2004}, and  also find applications for congestion avoidance  in swarm navigation \cite{santos_chaimowicz_IROS2011}. 
While the utility of hierarchies and expressions for manipulating them are by no means new to this problem domain, we believe that the explicit formal connection \cite{Baryshnikov_Guralnik} we exploit between the topology of configuration space \cite{Fadell_Husseini_2001} and the topology of tree space \cite{felsenstein2004} through the hierarchical clustering relation \cite{jain_dubes_1988} is entirely new.
% and opens up promising ground that deserves to be broadly explored within the coordinated motion planning literature. 

%%%%%%%%%%%%%%%%%%%%%%%%%%%%%%%%%%%%
%%%%%%%%%%%%%%%%%%%%%%%%%%%%%%%%%%%%
\subsubsection{Hierarchy in Biology and Engineering}
%%%%%%%%%%%%%%%%%%%%%%%%%%%%%%%%%%%%
%%%%%%%%%%%%%%%%%%%%%%%%%%%%%%%%%%%%

Biology offers spectacularly diverse examples of animal spatial organization ranging from self-sorting in cells \cite{steinberg_Science1963}, tissues and organs \cite{turing_PTBS1952,gierer_meinhardt_K1972}, and groups of individuals \cite{deneubourg_etal_SAB1991,ame_rivault_deneubourg_AB2004,halverson_skelly_caccone_MEC2009} to more patterned teams \cite{obrien_JEMBE1989,kareiva_odell_AN1987,anderson_franks_BE2001,parrish_edelsteinkeshet_Science1999}, all the way through strategic group formations in vertebrates \cite{bednarz_Science1988,strubin_steinegger_bshary_E2011}, mammals \cite{gazda_etal_RSBS2005,scheel_packer_AB1991,stander_BES1992,nituch_schaefer_maxwell_E2008}, and primates
\cite{watts_mitani_IJP2002,crofoot_gilby_PNAS2012} hypothesized to increase
efficacy in foraging \cite{obrien_JEMBE1989,kareiva_odell_AN1987}, hunting \cite{watts_mitani_IJP2002,bednarz_Science1988,gazda_etal_RSBS2005,scheel_packer_AB1991}, logistics and construction \cite{anderson_franks_BE2001,parrish_edelsteinkeshet_Science1999}, 
predator avoidance \cite{tien_levin_rubenstein_EER2004,hammer_hammer_CJFAS2000}, and even to stabilize whole ecologies \cite{fryxell_mosser_sinclair_packer_Nature2007} --- all consequent upon the collective ability to target, track, and transform geometrically structured patterns of mutual location in response to environmental stimulus. 
These formations are remarkable for at least two reasons. 
First, their global structure seems to arise from local signaling and response amongst proximal individuals coupled to specific physical environments \cite{flierl_etal_JTB1999}, in a manner that might be posited as a  paradigm for  generalized emergent intelligence \cite{couzin_Nature2007}.  
Second, these formations  appear to resist familiar rigid prescriptions governing absolute or relative location, instead giving wide latitude for individual autonomy and detailed positioning  
(intuitively, a necessity for negotiating fraught, highly dynamic  interactions  such as arise in, say, hunting  \cite{gazda_etal_RSBS2005,stander_BES1992}), while, nevertheless, supporting the underlying coarse, deterministic ``deep structure'' as a dynamical invariant.
It is this second remarkable attribute of
biological swarms that inspires the present paper.

This profusion of pattern formation in biology has inspired a commensurate interest in robotics, yielding a growing literature on group coordination behaviors \cite{kumar_garg_kumar_ACC2008,santos_pimenta_chaimowicz_ICRA2014,santos_chaimowicz_R2014,huang_etal_ICRA2014} motivated by the intuition that the heterogeneous action and sensing abilities of a group of robots might enable a comparably diverse range of  complex tasks beyond the capabilities of a single individual. 
For example, group coordination via splitting and merging behaviours creates effective strategies for obstacle avoidance \cite{ogren_ICRA2004}, congestion control \cite{santos_chaimowicz_IROS2011}, shepherding \cite{chaimowicz_kumar_DARS2007}, area exploration \cite{chaimowicz_kumar_DARS2007,dames_schwager_kumar_rus_CDC2012}, and maintaining persistent and coherent groups while adapting to the environment \cite{huang_etal_ICRA2014}. 
In almost all of the robotics work in this area, formation tasks are given based upon rigid specifications taking either the form of explicit formation or relative distance graphs, with few exceptions including the ``shape'' abstraction of \cite{belta_kumar_TRO2004} or applications in  unknown environments such as area coverage and exploration \cite{schwager_rus_slotine_IJRR2009}.
Alternatively, hierarchical clustering offers an interesting means of ensemble task encoding and control; and  it seems likely that the ability to specify organizational structure in the precise but flexible terms that hierarchy permits will add a useful tool to the robot motion planner's toolkit.

\begin{table}[h]
\centering
\caption{Principal symbols used throughout this paper}
\vspace{0mm}
\begin{tabular}{@{\hspace{1mm}}l@{\hspace{3mm}}l@{\hspace{1mm}}}
\hline
$\indexset$, $\vectbf{\radius}$ & Sets of labels and disk radii [\ref{sec.ConfSpace}]\\
$\cRdJr$ &   The conf. space of labelled, noncolliding disks \refeqn{eq.confspace}\\
$\bintreetopspace_{\indexset}$  & The space of binary trees   [\ref{sec.Hierarchy}]   \\
$\HC$  & Hierarchical clustering  [\ref{sec.ConfHierarchy}]\\
$\HCkmeans$  & Iterative 2-means clustering  [\ref{sec.HierNavSphere}]\\
$\stratum{\tree}$ & The stratum of a tree, $\tree \in \bintreetopspace_{\indexset}$, \refeqn{eq.HierarchicalStratum}  \\
$\portal\prl{\treeA,\treeB}$ & Portal configurations of a pair, $\prl{\treeA,\treeB}$, of trees \refeqn{eq.portal} \\
$\portalconf{\treeA,\treeB}$ & Portal map [\ref{sec.HierPortal}] \\
$\adjgraph{\indexset} = \prl{\bintreetopspace_{\indexset}, \adjedgeset}$ & The adjacency graph of trees [\ref{sec.TreeGraph}] \\
$\NNIgraph_{\indexset} = \prl{\bintreetopspace_{\indexset}, \NNIedgeset}$ & The NNI-graph of trees [\ref{sec.TreeGraph}] \\
\hline
\end{tabular}
\vspace{-2mm}
\end{table}

%%%%%%%%%%%%%%%%%%%%%%%%%%%%%%%%%%%%%%
%%%%%%%%%%%%%%%%%%%%%%%%%%%%%%%%%%%%%%
\section{Hierarchical Abstraction}
\label{sec.Abstraction}
%%%%%%%%%%%%%%%%%%%%%%%%%%%%%%%%%%%%%%
%%%%%%%%%%%%%%%%%%%%%%%%%%%%%%%%%%%%%%

This section describes how we relate multirobot configurations to abstract cluster trees via hierarchical clustering methods and how we define connectivity in tree space. 

%%%%%%%%%%%%%%%%%%%%%%%%%%%%%%%%%%%%%%
%%%%%%%%%%%%%%%%%%%%%%%%%%%%%%%%%%%%%%
\subsection{Configuration Space}
\label{sec.ConfSpace}
%%%%%%%%%%%%%%%%%%%%%%%%%%%%%%%%%%%%%%
%%%%%%%%%%%%%%%%%%%%%%%%%%%%%%%%%%%%%%

For ease of exposing fundamental technical concepts, we restrict our attention to groups of Euclidean spheres in a $d$-dimensional ambient space, but many concepts introduced herein can be generalized to any metric space.

Given an index set, $ \indexset=[n]\ldf \crl{1, \ldots, n} \subset \N$, a \emph{heterogeneous multirobot configuration}, $\vectbf{\state}= \prl{ \vect{\state}_j }_{j \in \indexset}$, is a labeled nonintersecting placement of $\card{\indexset}=n$ distinct Euclidean spheres,\footnote{Here, $\card{A}$ denotes the cardinality of set $A$.} where $i$th sphere is centered at $\vect{x}_i \in \R^d$ and has radius $r_i \geq 0$. 
We find it convenient to identify the \emph{configuration space} \cite{Fadell_Husseini_2001} with the set of distinct labelings, i.e., the injective mappings of $\indexset$ into $\R^d$,
and, given a vector of nonnegative radii, $\mathbf{r} \ldf \prl{r_j}_{j \in \indexset} \in \prl{\R_{\geq 0}}^{\indexset}$,  we will find it convenient to denote our ``thickened'' subset of this configuration space as\footnote{Here, $\R$ and $\R_{\geq 0}$ denote the set of real numbers and its subset of nonnegative real numbers, respectively; and $\R^d$ is the $d$-dimensional Euclidean space.}
\begin{align}\label{eq.confspace}
\hspace{-2.5mm}\cRdJr \sqz{\sqz{\ldf}} \crl{ \hspace{-0.4mm}  \vectbf{\state} \sqz{\in} {(\R^d)}^{\!\indexset}  \Big | \! \norm{\vect{\state}_i \sqz{-} \vect{\state}_j} \sqz{>} r_i \sqz{+} r_j,  \forall i \sqz{\neq} j \sqz{\in} \indexset \!} \!\!,\!\!\!
\end{align}
where $\norm{.}$ denotes the standard Euclidean norm on $\R^d$.

%%%%%%%%%%%%%%%%%%%%%%%%%%%%%%%%%%%%%%
%%%%%%%%%%%%%%%%%%%%%%%%%%%%%%%%%%%%%%
\subsection{Cluster Hierarchies}
\label{sec.Hierarchy}
%%%%%%%%%%%%%%%%%%%%%%%%%%%%%%%%%%%%%%
%%%%%%%%%%%%%%%%%%%%%%%%%%%%%%%%%%%%%%

\begin{figure}[tb]

 \centering
 \begin{tabular}{@{}c@{}@{}c@{}}
 \begin{tabular}{c} 
 \includegraphics[width=0.21\textwidth]{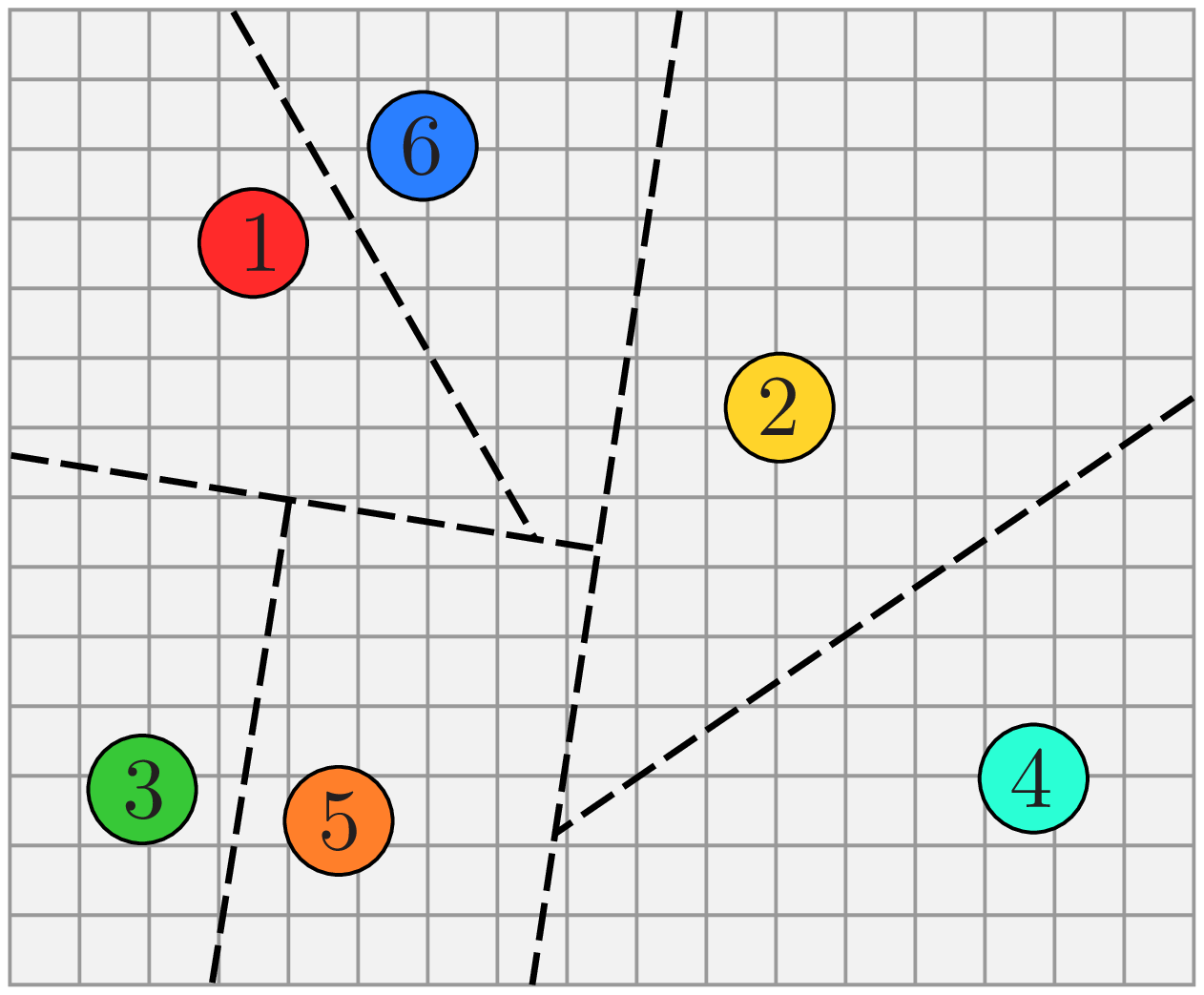} 
 \end{tabular}&
 \begin{tabular}{c}
 \includegraphics[width=0.21\textwidth]{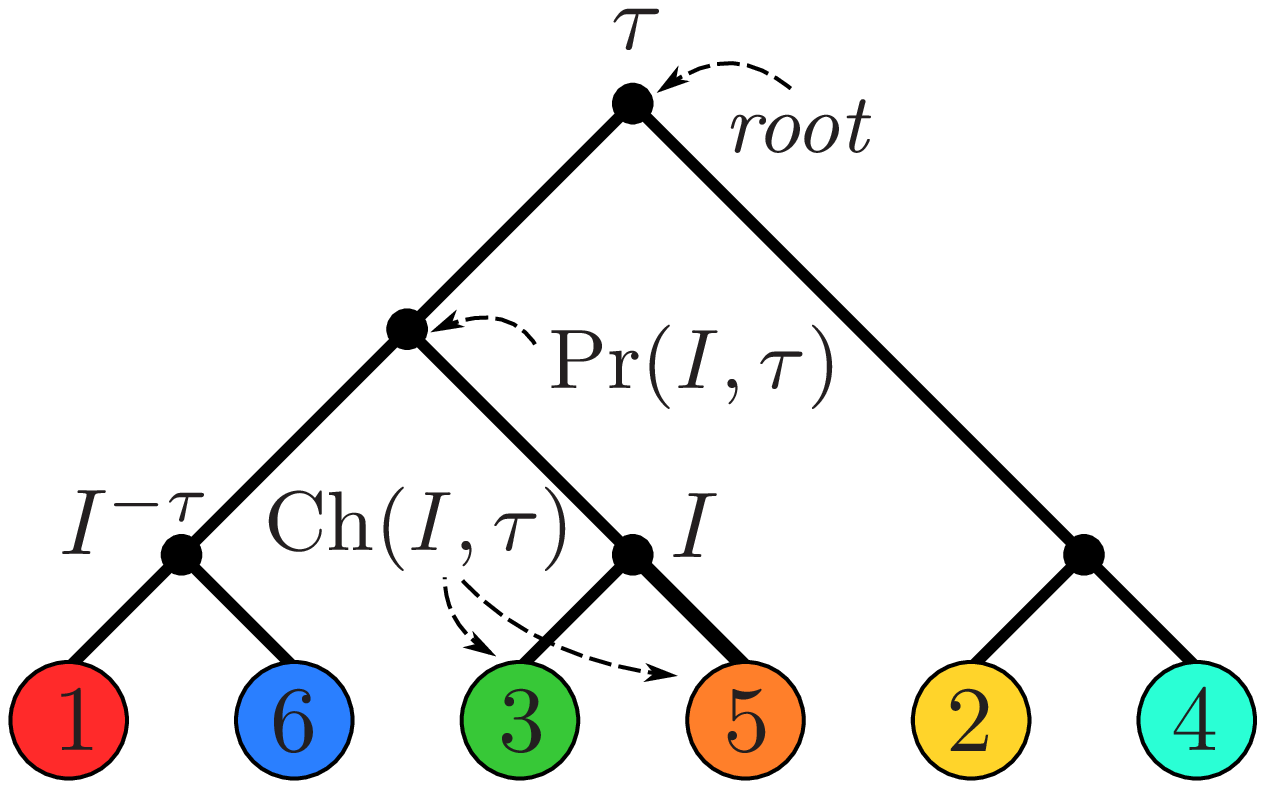}
 \end{tabular} \\[-1mm] 
 \scalebox{0.7}{(a)} & \scalebox{0.7}{(b)}
 \end{tabular}
 \vspace{-1mm}
\caption{An illustration of (a) a heteregeneous configuration of unit disks in $\confspace{\R^2,\brl{6}, \vectbf{1}}$ and (b) its iterative 2-mean clustering \cite{savaresi_boley_ICDM2001} hierarchy  $\tree$ in $\bintreetopspace_{\brl{6}}$, where the dashed lines in (a) depict the separating hyperplanes  between clusters, and (b) illustrates hierarchical cluster relations: parent - $\parentCL{I,\tree}$, children - $\childCL{I,\tree}$, and  local complement (sibling) -  $\complementLCL{I}{\tree}$ of cluster $I$ of the rooted binary  tree, $\tree \in \bintreetopspace_{\brl{6}}$. 
%Filled and unfilled circles represent interior and leaf nodes, respectively. 
An interior node is referred by its cluster, the list of leaves below it; for example, $I = \crl{3,5}$. 
Accordingly the cluster set of $\tree$ is  $\cluster{\tree} = \crl{\big. \!\crl{1}\!, \crl{2}\!, \ldots, \crl{6}\!, \crl{1,6}\!,\crl{3,5}\!, \crl{2,4}\!, \crl{1,3,5,6}\!, \crl{1,2,3,4,5,6}\!}$.
  }
\label{fig.ConfigurationHierarchy}
\vspace{-1mm} 
 \end{figure}

A rooted semi-labelled tree $\tree$ over a fixed finite index set $\indexset$, illustrated in \reffig{fig.ConfigurationHierarchy}, is  a directed acyclic graph $G_{\tree} = \prl{V_{\tree}, E_{\tree} }$,
whose leaves, vertices of degree one, are bijectively labeled by $\indexset$ and interior vertices all have out-degree at least two; and all of whose edges in $E_{\tree}$ are directed away from a vertex designated to be the \emph{root} \cite{billera_holmes_vogtmann_aap2001}. 
A rooted tree with all interior vertices of out-degree two is said to be { \em binary } or, equivalently, { \em non-degenerate}, and  all other trees are
said to be \emph{degenerate}.  
In this paper $\bintreetopspace_{\indexset}$ denotes the set of rooted nondegenerate trees over leaf set $\indexset$.

A rooted semi-labelled tree $\tree$ uniquely determines (and henceforth will be interchangeably termed) a \emph{cluster hierarchy} \cite{mirkin_1996}. 
By definition, all vertices of $\tree$  can be reached from the root through a directed path in $\tree$. 
The \emph{cluster} of a vertex $v \in V_{\tree}$ is defined to be the set of leaves reachable from $v$ by a directed path in $\tree$. 
Let $\cluster{\tree}$ denote the set of all vertex clusters of $\tree$.
%Correspondingly, the cluster set $\cluster{\tree}$ of $\tree$ is defined to be the set of all its vertex clusters.
% $\cluster{v}$, $v\in V_{\tree}$. 

For every cluster $I \in \cluster{\tree}$ we recall the standard notion of parent (cluster) $\parentCL{I,\tree}$ and lists of children $\childCL{I,\tree}$,  ancestors $\ancestorCL{I,\tree}$ and descendants $\descendantCL{I,\tree}$ of $I$ in $\tree$ --- see \cite{arslanEtAl_Allerton2012} for explicit definitions of cluster relations.
Additionally, we find it useful to define the \emph{local complement (sibling)} of cluster $I \in \cluster{\tree}$ as $\complementLCL{I}{\tree} \ldf \parentCL{I,\tree} \setminus I$.

%%%%%%%%%%%%%%%%%%%%%%%%%%%%%%%%%%%%%%
%%%%%%%%%%%%%%%%%%%%%%%%%%%%%%%%%%%%%%
\subsection{Configuration Hierarchies}
\label{sec.ConfHierarchy}
%%%%%%%%%%%%%%%%%%%%%%%%%%%%%%%%%%%%%%
%%%%%%%%%%%%%%%%%%%%%%%%%%%%%%%%%%%%%%

A \emph{hierarchical clustering}\footnote{Although clustering algorithms generating degenerate hierarchies  are available, many  standard hierarchical clustering methods return binary clustering trees as a default, thereby  avoiding  commitment to some ``optimal" number of clusters \cite{jain_dubes_1988,witten_frank_hall_DM2011}.} $\HC \subset \cRdJr \times \bintreetopspace_{\indexset}$ is a relation from the configuration space $\cRdJr$ to the abstract space of binary hierarchies $\bintreetopspace_{\indexset}$ \cite{jain_dubes_1988}, an example depicted in \reffig{fig.ConfigurationHierarchy}. 
In this paper we will only be interested in clustering
methods that can classify all possible configurations (i.e. for which $\HC$ assigns some tree to every configuration), and so we need:
\begin{property}\label{P.HCmultifunction}
$\HC$ is a multi-function.
\end{property}

\noindent Most standard divisive and agglomerative hierarchical clustering methods exhibit this property, but generally fail to be functions because choices may be required between different but equally valid cluster splitting or merging decisions \cite{jain_dubes_1988}.

\begin{figure}[t]
 \centering
 \includegraphics[width=0.45\textwidth]{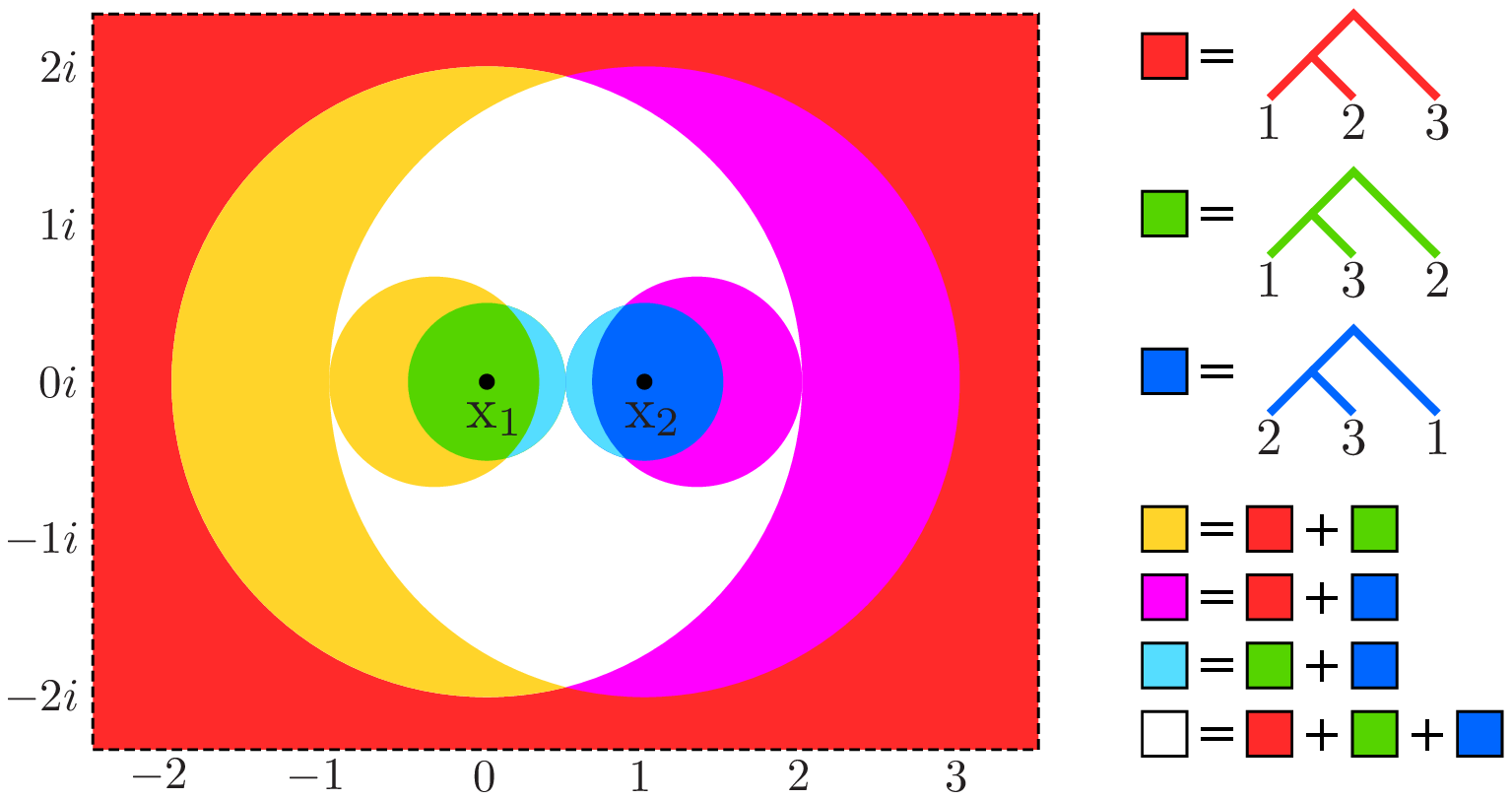} 
 \vspace{-1mm}
 \caption{The Quotient Space  $\confspace{\C,\brl{3}, \vectbf{0}}/\sim$, where for any $\vectbf{\stateA}, \vectbf{\stateB} \in \confspace{\C,\brl{3}, \vectbf{0}}$,  $\vectbf{\stateA} \sim \vectbf{\stateB} \Longleftrightarrow \frac{\vect{\stateA}_3 - \vect{\stateA}_1}{\vect{\stateA}_2 - \vect{\stateA}_1} = \frac{\vect{\stateB}_3 - \vect{\stateB}_1}{\vect{\stateB}_2 - \vect{\stateB}_1}$.
Here, point particle configurations are quotiented out by translation, scale and rotation, and so $\vect{\stateA}_1 = 0 + 0i$, $\vectbf{\stateA}_2 = 1 + 0i$ and $\vect{\stateA}_3 \in \C \setminus \crl{\vect{\stateA}_1 ,\vect{\stateA}_{2}}$.
 Regions are colored according the associated cluster hierarchies results from their iterative 2-mean clustering \cite{savaresi_boley_ICDM2001}.
 For instance, any configuration in the white region supports all  hierarchies in $\bintreetopspace_{\brl{3}}$. }
\label{fig.QuotientSpace} 
\vspace{-1mm}
\end{figure}

Given such an $\HC$, for any $\vectbf{\state} \in \cRdJr$ and $\tree \in \bintreetopspace_{\indexset}$, we say $\vectbf{x}$ \emph{supports} $\tree$ if and only if $\prl{\vectbf{\state},\tree} \in \HC$.
The \emph{stratum} associated with a binary hierarchy $\tree \in \bintreetopspace_{\indexset}$, denoted by $\stratum{\tree} \subset \cRdJr$, 
is  the set of all configurations $\vectbf{x} \in \cRdJr$ supporting the same tree $\tree$ \cite{arslanEtAl_Allerton2012}, 
\begin{align} \label{eq.HierarchicalStratum}
\stratum{\tree} \ldf \crl{\vectbf{x} \in \cRdJr \Big| \prl{\vectbf{x},\tree} \in \HC},
\end{align}
and this yields a tree-indexed cover of the configuration space.
For purposes of illustration, we depict in \reffig{fig.QuotientSpace} the strata of $\confspace{\C,\brl{3}, \vectbf{0}}$ --- a space that represents a group of three point particles on the complex plane.\footnote{Here, $\vectbf{0}$ and $\vectbf{1}$ are, respectively, vectors of all zeros and ones with the appropriate sizes.}

The restriction to binary trees precludes  combinatorial tree degeneracy \cite{billera_holmes_vogtmann_aap2001} and we will  avoid configuration degeneracy by imposing:
\begin{property}\label{P.HCopeninterior}
Each stratum of $\HC$ includes an open subset of configurations, i.e. for every $\tree \in \bintreetopspace_{\indexset}$, $\stratumI{\tree} \neq \emptyset$.\footnote{Here, $\mathring{A}$ denotes the interior of set $A$.}  
 \end{property}
\noindent Once again, most standard hierarchical clustering methods respect this assumption: they generally all agree (i.e. return the same result) and are robust to small perturbations of a configuration whenever all its clusters are compact and well separated \cite{witten_frank_hall_DM2011}.

%For purposes of illustration, we depict in \reffig{fig.ConfHierarchyQuotientSpace}(c) the strata of $\confS{\C}{\brl{3}}$ --- a space that represents a swarm of three particles on the plane. 
%Notice that the intersection of any pair of hierarchical strata of three particles   contains configurations at which particles form an equilateral triangle.
%Such configurations supporting a pair of hierarchies shall be later referred to \emph{portal} configurations enabling instantaneous transitions between adjacent hierarchies.

Given any two multirobot configurations supporting the same cluster hierarchy, moving between them while maintaining the shared cluster hierarchy (introduced later as \refprob{prob.HierarchyInvariant}) requires:
\begin{property}\label{P.HCconnected}
Each stratum of $\HC$ is connected. 
\end{property}
\noindent For an arbitrary clustering method this requirement is generally not trivial to show, but when configuration clusters of $\HC$ are linearly separable, one can characterize the topological shape of each stratum to verify this requirement, as we do in \refsec{sec.HierStrata}.

\begin{figure}[tb]
\centering
\includegraphics[width=0.45\textwidth]{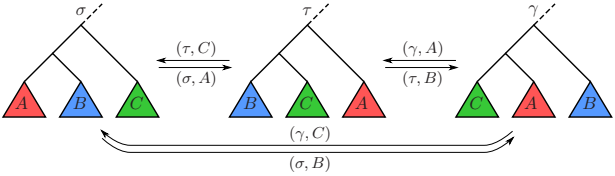}
\vspace{-2mm} 
\caption{An illustration of NNI moves between binary trees: each arrow is labeled by a source tree and an associated cluster defining the move.
}
\label{fig.NNIMove}
\vspace{0mm}
\end{figure}

%%%%%%%%%%%%%%%%%%%%%%%%%%%%%%%%%%%%%%
%%%%%%%%%%%%%%%%%%%%%%%%%%%%%%%%%%%%%%
\subsection{Graphs On Trees}
\label{sec.TreeGraph}
%%%%%%%%%%%%%%%%%%%%%%%%%%%%%%%%%%%%%%
%%%%%%%%%%%%%%%%%%%%%%%%%%%%%%%%%%%%%%

After establishing the relation between multirobot configurations and cluster hierarchies, the final step of our proposed abstraction is to determine the connectivity of tree space.

\smallskip

Define the \emph{adjacency graph } $\adjgraph{\indexset} = \prl{\bintreetopspace_{\indexset}, \adjedgeset}$ to be the 1-skeleton of the nerve \cite{Hatcher_2002} of the 
$\cRdJr$-cover induced by $\HC$. 
That is to say, a pair of hierarchies, $\treeA, \treeB \in \bintreetopspace_{\indexset}$, is connected with an edge in $\adjedgeset$ if and only if their strata intersect,  $\stratum{\treeA} \cap \stratum{\treeB} \neq \emptyset$.   
To enable navigation between structurally different multirobot configurations later (\refprob{prob.Transition}), we need: 
\begin{property}
The adjacency graph  is connected.
\end{property}
\noindent Although the adjacency graph is a critical building block of our abstraction, as \reffig{fig.QuotientSpace} anticipates, $\HC$ strata generally have complicated shapes, making it usually hard to compute the complete adjacency graph.

Fortunately, the computational biology literature \cite{felsenstein2004} offers an alternative notion of adjacency that turns out to be both   feasible and nicely compatible with our needs, yielding a computationally effective, connected subgraph of the adjacency graph, $\adjgraph{\indexset}$, as follows.

The \emph{Nearest Neighbor Interchange (NNI)} move at a cluster $A \in \cluster{\treeA}$ on a binary hierarchy $\treeA \in \bintreetopspace_{\indexset}$,  as illustrated in \reffig{fig.NNIMove}, swaps cluster $A$ with its parent's sibling $C = \complementLCL{\parentCL{A,\treeA}}{\treeA}$ to yield another binary hierarchy $\treeB \in \bintreetopspace_{\indexset}$ \cite{robinson_jct1971,moore_goodman_barnabas_jtb1973}. 
Say that $\treeA,\treeB \in \bintreetopspace_{\indexset}$ are \emph{NNI-adjacent} if and only if one can be obtained from the other by a single NNI move.
Moreover, define the \emph{NNI-graph} $\NNIgraph_{\indexset} = \prl{\bintreetopspace_{\indexset}, \NNIedgeset}$ to have vertex set $\bintreetopspace_{\indexset}$, with two trees connected by an edge in $\NNIedgeset$ if and only if they are NNI-adjacent, see \reffig{fig.NNIGraph}.
An important contribution of this paper will be to show how the NNI-graph  yields a computationally effective subgraph of the adjacency graph (\refthm{thm.Portal}). 

\begin{figure}[tb]
\centering
\includegraphics[width=0.43\textwidth]{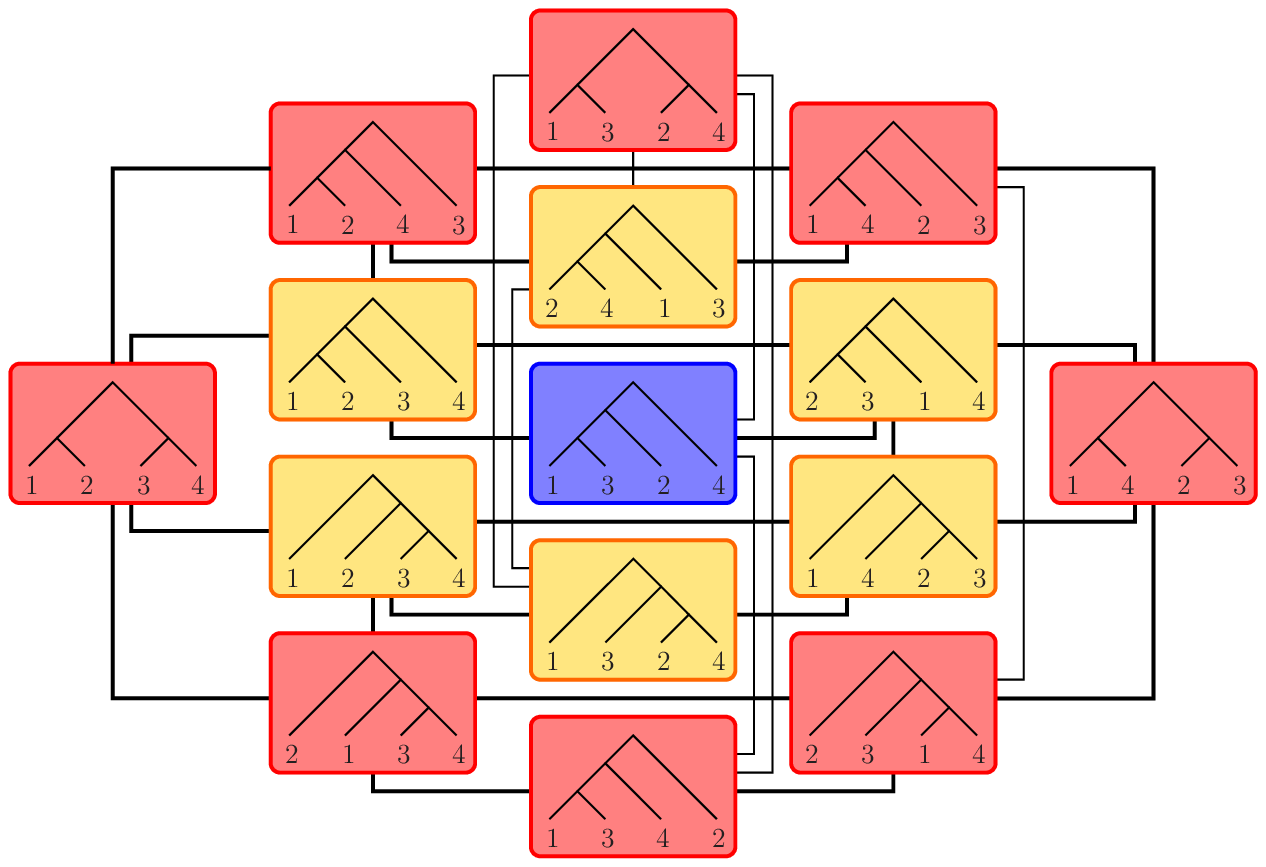} 
\vspace{-2mm}
\caption{The NNI Graph: a graphical representation  of the space of rooted binary trees, $\bintreetopspace_{\indexset}$, with NNI connectivity, where $\indexset = \brl{4} = \crl{1,2,3,4}$.
}
\label{fig.NNIGraph}
\end{figure}

%%%%%%%%%%%%%%%%%%%%%%%%%%%%%%%%%%%%%%
%%%%%%%%%%%%%%%%%%%%%%%%%%%%%%%%%%%%%%
\section{Hierarchical Navigation Framework}
\label{sec.HierNav}
%%%%%%%%%%%%%%%%%%%%%%%%%%%%%%%%%%%%%%
%%%%%%%%%%%%%%%%%%%%%%%%%%%%%%%%%%%%%%

Hierarchical abstraction introduced in \refsec{sec.Abstraction} intrinsically suggests a two-level navigation strategy for coordinated motion design: (i) at the low-level perform finer adjustments on configurations using hierarchy preserving vector fields, (ii) and at the high-level resolve structural conflicts between configurations using  a discrete transition policy in tree space; and the connection between these two levels are established through ``portals'' --- open sets of configurations supporting two adjacent hierarchies.
In this section we abstractly describe the generic components of our navigation framework and we show how they are put together.

\subsection{Generic Components of Hierarchical Navigation}

\subsubsection{Hierarchy Preserving Navigation}

For ease of exposition we restrict attention to  first order (completely actuated single integrator) robot dynamics, and we will be interested in smooth closed loop feedback laws (or hybrid controllers composed from them) that result in complete flows,\footnotemark 
{\small
\begin{align}\label{eq.SystemModel}
\dot{\vectbf{x}} = f \prl{\vectbf{x}},
\end{align}
}%
where $f:\cRdJr \rightarrow \prl{\R^d}^\indexset$ is a vector field over $\cRdJr$ \refeqn{eq.confspace}.

\noindent
\footnotetext{A long prior robotics literature motivates the utility of this fully actuated ``generalized damper'' dynamical model \cite{Lozano-Perez_Mason_Taylor_1984}, and provides methods for ``lifts'' to controllers for second order plants \cite{Koditschek_1987,Koditschek_JDSMC1991a} as well.}

Denote by $\flow^t$ the \emph{flow} \cite{Arnold_1973} on $\cRdJr$ induced by the vector field $f$.
For a choice of hierarchical clustering $\HC$, the class of hierarchy-invariant vector fields maintaining the robot group in a specified hierarchical arrangement of clusters, $\tree \in \bintreetopspace_{\indexset}$, is defined as \cite{arslanEtAl_Allerton2012},
{\small
\begin{align}
\hspace{-1mm}\mathcal{F}_{\HC}(\tree) \sqz{\ldf} \! \crl{ \! f\sqz{:} \cRdJr \!\sqz{\rightarrow} \prl{\R^d}^{\!\indexset} \! \Big |
\flow^t\!\prl{\big. \stratum{\tree}\!} \!\sqz{\subset} \stratum{\tree}\!, t\sqz{>} 0 \!}\!\!.\! \!\!
\end{align}
}%
Hierarchy preserving navigation, the low-level component of our framework, uses the vector fields of $\mathcal{F}_{\HC}\prl{\tree}$ to invariantly retract almost all of a stratum onto any designated goal configuration.\footnote{It is important to remark that, instead of a single goal configuration, a more general family of problems can be parametrized by a set of goal configurations sharing a certain homotopy model comprising a set of appropriately nested spheres; and for such a general case one can still construct an exact retraction within our framework.} 
Thus, we require the availability of such a construction, summarized as:

\begin{problem}\label{prob.HierarchyInvariant}
For any $\tree \in \bintreetopspace_{\indexset}$ and $\vectbf{\stateB} \in \stratum{\tree}$ associated with $\HC$ construct  a control policy, $f_{\tree, \vectbf{\stateB}}$, using the hierarchy invariant vector fields of $\mathcal{F}_{\HC}\prl{\tree}$ whose closed loop asymptotically results in a retraction, $R_{\tree, \vectbf{\stateB}}$, of $\stratum{\tree}$, possibly excluding a set of measure zero\footnote{Recall from \cite{Farber_2003} that a continuous motion planner in a configuration space $X$ exists if and only if $X$ is contractible. Hence, if a hierarchical stratum is  non-contractible (\refthm{thm.StrataHomotopy}), the domain of such a vector field planner described in \refprob{prob.HierarchyInvariant} must exclude at least a set of measure zero.},  onto $\crl{\vectbf{\stateB}}$. 
\end{problem}
\noindent Key for purposes of the present application is  the observation that any hierarchy-invariant field $f \in \mathcal{F}_{\HC}\prl{\tree}$ must leave $\cRdJr$ invariant as well, and thus avoids any self-collisions of the agents along the way. 
%There are likely to be many alternative approaches to such results; and working with the 2-means hierarchical clustering \cite{savaresi_boley_ICDM2001},  we construct in \cite{arslanEtAl_Allerton2012} such a family of hierarchy preserving  control policies for particle configurations, and \refsec{sec.HierInvNav} extends that construction to thickened disk configurations. 

\medskip

%%%%%%%%%%%%%%%%%%%%%%%%%%%%%%%%%%%%%%%%%%%%%%%%%%%%%%%
%%%%%%%%%%%%%%%%%%%%%%%%%%%%%%%%%%%%%%%%%%%%%%%%%%%%%%%
\subsubsection{Navigation in the Space of Binary Trees}
\label{sec.TreeNav}
%%%%%%%%%%%%%%%%%%%%%%%%%%%%%%%%%%%%%%%%%%%%%%%%%%%%%%%
%%%%%%%%%%%%%%%%%%%%%%%%%%%%%%%%%%%%%%%%%%%%%%%%%%%%%%%

Whereas the controlled deformation retraction, $R_{\tree, \vectbf{\stateB}}$, above generates paths ``through" the strata, we will also want to navigate ``across" them along the adjacency graph (which will be later in \refsec{sec.HierNavSphere} replaced with the NNI-graph --- a computationally efficient, connected subgraph). 
Thus, we further require a construction of a discrete feedback policy in $\bintreetopspace_{\indexset}$ that recursively generates paths in the adjacency graph toward any specified destination tree from all other trees in $\bintreetopspace_{\indexset}$ by reducing a ``discrete Lyapunov function'' relative to that destination, which we  summarize as follows:
\begin{problem}\label{prob.Transition}
 Given any $\tree \in \bintreetopspace_{\indexset}$ construct recursively a closed loop  discrete dynamical system  in the  adjacency graph, taking the form of a deterministic discrete transition rule, $g_{\tree}$, with global attractor at $\tree$  endowed with a discrete Lyapunov function relative to the attractor $\tree$.  
\end{problem}
\noindent Such a recursively generated choice of next hierarchy will play the role of a discrete feedback policy used to define the reset map of our hybrid dynamical system. 
%Once again, there are many alternative ways of constructing such a discrete transition rule, for example, using  standard graph search algorithms, like the A* or Dijkstra's algorithm \cite{Cormen_2009}; and we recently develop in \cite{ArslanEtAl_NNITechReport2013}  such an efficient recursive procedure, summarized in \refsec{sec.NNINav}, to find paths  joining any given pair of trees in the NNI graph --- a subgraph of the adjacency graph (\refthm{thm.Portal}).

\medskip

%%%%%%%%%%%%%%%%%%%%%%%%%%%%%%%%%%%
%%%%%%%%%%%%%%%%%%%%%%%%%%%%%%%%%%%
\subsubsection{Hierarchical Portals}
\label{sec.HierPortal} 
%%%%%%%%%%%%%%%%%%%%%%%%%%%%%%%%%%%
%%%%%%%%%%%%%%%%%%%%%%%%%%%%%%%%%%%

Here, we relate the (combinatorial) topology of hierarchical clusters to the (continuous) topology of configurations by defining ``portals'' --- open sets of configurations supporting two adjacent hierarchies. 
\begin{definition} \label{def.portal}
The \emph{portal},  $\portal\prl{\treeA,\treeB}$, of a pair of  hierarchies, $\treeA,\treeB \in \bintreetopspace_{\indexset}$,  is the set of all configurations supporting interior strata of both trees,
\begin{align}
\portal\prl{\treeA,\treeB} \ldf \stratumI{\treeA} \cap  \stratumI{\treeB}. \label{eq.portal}
\end{align}
\end{definition}

\noindent Namely, portals are geometric realizations in the configuration space of the edges of the adjacency graph on trees, see \reffig{fig.QuotientSpace}. 
To realize discrete transitions in tree space via hierarchy preserving navigation in the configuration space, we need a portal map that takes an edge of the adjacency graph, and returns a target configuration in the associated portal, summarized as: 
\begin{problem} \label{prob.Portal}
Given an edge $\prl{\treeA, \treeB} \in \adjedgeset$ of the adjacency graph $\adjgraph{\indexset} = \prl{\bintreetopspace_{\indexset}, \adjedgeset}$, construct a geometric realization map $\portalconf{\prl{\treeA,\treeB}}: \stratum{\treeA} \rightarrow \portal\prl{\treeA, \treeB}$  that takes a configuration supporting $\treeA$, and  returns a target configuration supporting both trees $\treeA$ and $\treeB$. 
\end{problem}
\noindent A portal map will serve the role of a dynamically computed ``prepares graph'' \cite{Burridge_Rizzi_Koditschek_1999} for the sequentially composed local controllers whose correct recruitment solves the reactive coordinated motion planning problem (\refthm{thm.HierNavAlg}).

%%%%%%%%%%%%%%%%%%%%%%%%%%%%%%%%%%%%%%%%%
%%%%%%%%%%%%%%%%%%%%%%%%%%%%%%%%%%%%%%%%%
\subsection{Specification and Correctness of the Hierarchical Navigation Control (HNC) Algorithm}
%%%%%%%%%%%%%%%%%%%%%%%%%%%%%%%%%%%%%%%%%
%%%%%%%%%%%%%%%%%%%%%%%%%%%%%%%%%%%%%%%%%

Assume the selection of a goal configuration $\vectbf{\stateB} \in \stratum{\treeB}$ and a hierarchy $\treeB \in \bintreetopspace_{\indexset}$ that $\vectbf{\stateB}$ supports. 
Now, given (almost) any initial configuration $\vectbf{\stateA} \in \stratum{\treeA}$ for some hierarchy $\treeA\in \bintreetopspace_{\indexset}$ that $\vectbf{\stateA}$ supports, 
\reftab{tab.HNCAlgorithm} presents the HNC algorithm. 

\begin{table}[htb]
\vspace{-2mm}
\caption{The HNC Algorithm}
\label{tab.HNCAlgorithm}
\centering 
\vspace{-2mm}
\begin{tabular}{|p{0.45\textwidth}@{\hspace{2mm}}|}
\hline
\vspace{-2mm}

For (almost) any initial  $\vectbf{\stateA} \in \stratum{\treeA}$ and $\treeA \in \bintreetopspace_{\indexset}$, and  desired $\vectbf{\stateB} \in \stratum{\treeB}$ and $\treeB \in \bintreetopspace_{\indexset}$,
\begin{enumerate}
\item (Hybrid Base Case) if  $\vectbf{\stateA} \in \stratum{\treeB}$ then apply stratum-invariant dynamics, $f_{\treeB, \vectbf{\stateB}}$ (\refprob{prob.HierarchyInvariant}).

\item (Hybrid Recursive Step) else, 

\begin{enumerate}
\item invoke the discrete transition rule $g_{\treeB}$ (\refprob{prob.Transition}) to propose  an adjacent tree, $\treeC \in \bintreetopspace_{\indexset}$, with lowered discrete Lyapunov value.

\item Choose local configuration goal, $\vectbf{\stateC} := \portalconf{\prl{\treeA,\treeC}} \prl{\vectbf{\stateA}}$ (\refprob{prob.Portal}).

\item Apply the stratum-invariant continuous controller $f_{\treeA,\vectbf{\stateC}}$ (\refprob{prob.HierarchyInvariant}).

\item If the trajectory  enters $\stratum{\treeB}$ then go to step 1; else, the trajectory must enter $\stratum{\treeC}$ in finite time in which case terminate $f_{\treeA, \vectbf{\stateC}}$,  reassign $\treeA \leftarrow \treeC$, and go to step 2a).
\end{enumerate}
\vspace{-3.5mm}
\end{enumerate}
\\
\hline

\end{tabular}

\end{table}

\begin{figure}[h]
\centering
\includegraphics[width=0.4\textwidth]{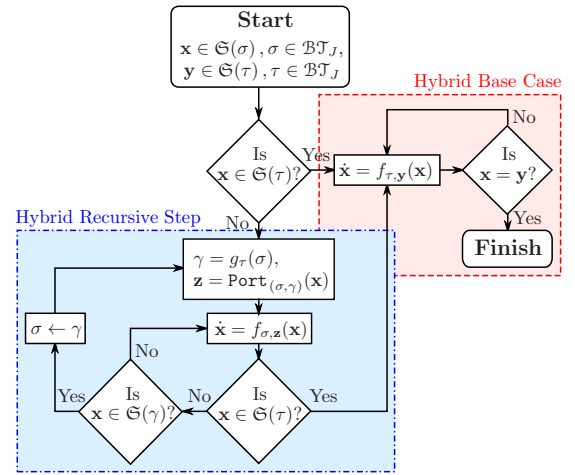} 
\vspace{-2mm}
\caption{Flowchart of the hybrid vector field planner.% guaranteed to bring almost any initial configuration, $\vectbf{\stateA} \in \stratum{\treeA}$ for arbitrary $\treeA \in\bintreetopspace_{\indexset}$ to a specified desired configuration, $\vectbf{\stateB} \in \stratumO{\treeB}$ supporting a designated hierarchy, $\treeB \in \bintreetopspace_{\indexset}$. 
}
\label{fig:NNIHierNavFlowChart}
\end{figure}

\begin{theorem} \label{thm.HierNavAlg}
The HNC Algorithm in \reftab{tab.HNCAlgorithm} defines a hybrid dynamical system whose execution brings almost every initial configuration, $\vectbf{\stateA} \in \cRdJ$, in finite time to an arbitrarily small neighborhood of $\vectbf{\stateB} \in \stratum{\treeB}$ with the guarantee of no collisions along the way. 
\end{theorem}
\begin{proof}
In the base case, 1) the conclusion follows directly from the construction of \refprob{prob.HierarchyInvariant}: the flow $f_{\treeB, \vectbf{\stateB}}$  keeps the state in $\stratum{\treeB}$, approaches a neighborhood of $\vectbf{\stateB}$ (which is an asymptotically stable equilibrium state for that flow) in finite time.

In the inductive step, a) The NNI transition rule $g_{\tree}$  guarantees a decrement in the Lyapunov function after  a transition from   $\treeA$ to  $\treeC$ (\refprob{prob.Transition}), and  a new local policy $f_{\treeA,\vectbf{\stateC}}$ is automatically deployed with a local goal configuration $\vectbf{\stateC} \in \portal\prl{\treeA,\treeC}$ found in b).   
Next, the  flow $f_{\treeA,\vectbf{\stateC}}$ in c) is guaranteed to keep the state in $\stratum{\treeA}$ and approach $\vectbf{\stateC}\in \portal\prl{\treeA,\treeC}$ asymptotically from almost all initial configurations. 
If the base case is not triggered in d), then the state enters arbitrarily small neighborhoods of $\vectbf{\stateC}$ and, hence, must eventually reach $\portal\prl{\treeA,\treeC} \subset \stratum{\treeC}$ in finite time, triggering a return to 2a). 
Because the dynamical transitions $g_{\tree}$ initiated from any hierarchy in $ \bintreetopspace_{\indexset}$ reaches $\tree$ in finite steps (\refprob{prob.Transition}), it must eventually  trigger the base case.  \qed
\end{proof}

%%%%%%%%%%%%%%%%%%%%%%%%%%%%%%%%%
%%%%%%%%%%%%%%%%%%%%%%%%%%%%%%%%%
\section{Hierarchical Navigation of Euclidean Spheres via Bisecting K-means Clustering}
\label{sec.HierNavSphere}
%%%%%%%%%%%%%%%%%%%%%%%%%%%%%%%%%
%%%%%%%%%%%%%%%%%%%%%%%%%%%%%%%%%

We now confine our attention to 2-means divisive hierarchical clustering \cite{savaresi_boley_ICDM2001}, $\HCkmeans$, and 
demonstrate a construction of our hierarchical navigation framework for coordinated navigation of Euclidean spheres via $\HCkmeans$.

%%%%%%%%%%%%%%%%%%%%%%%%%%%%%%%%
%%%%%%%%%%%%%%%%%%%%%%%%%%%%%%%%
\subsection{Hierarchical Strata  of \emph{$\HCkmeans$}}
\label{sec.HierStrata}
%%%%%%%%%%%%%%%%%%%%%%%%%%%%%%%%
%%%%%%%%%%%%%%%%%%%%%%%%%%%%%%%%

Iterative 2-means clustering, $\HCkmeans$, is a divisive method that recursively constructs a cluster hierarchy of a configuration in a top-down fashion \cite{savaresi_boley_ICDM2001}.
Briefly, this method splits each successive (partial) configuration by applying 2-means clustering, and successively continues with each subsplit until reaching singletons.
By construction, complementary configuration clusters of $\HCkmeans$ are linearly separable by a hyperplane defined by the associated cluster centroids\footnote{In the context of self-sorting in heterogeneous swarms \cite{kumar_garg_kumar_ACC2008}, two groups of robot swarms  are said to be \emph{segregated} if their configurations are linearly separable;  and in this regard configuration hierarchies of $\HCkmeans$ represent spatially cohesive and segregated swarms groups at different resolutions.}, as illustrated in \reffig{fig.ConfigurationHierarchy}; and the stratum of $\HCkmeans$ associated with a binary hierarchy $\tree \in \bintreetopspace_{\indexset}$ can be characterized  by the intersection inverse images,
\begin{align}
\label{eq.ClosedStrata}
\stratum{\tree} &=   \bigcap_{I \in \cluster{\tree} \setminus \crl{\indexset}} \bigcap_{i \in I}\sepmagsym{i,I,\tree}^{-1}[0,\infty),  
\end{align}
of the scalar valued  ``separation" function, $\sepmagsym{i, I, \tree} : \cRdJr \rightarrow \R$  \cite{arslanEtAl_Allerton2012} returning the distance of agent $i$ in cluster $I \in \cluster{\tree} \setminus \crl{\indexset}$ to the perpendicular bisector of the centroids of complementary clusters $I$ and $\complementLCL{I}{\tree}$: \footnote{Here,  $\tr{\mat{A}}$ denotes the transpose of $\mat{A}$.} 
%
%This function returns the distance of agent $i$ in cluster $I \in \cluster{\tree} \setminus \crl{\indexset}$  to the separating hyperplane  that is perpendicular to the separation vector, $\ctrdsep{I,\tree}{\vectbf{\state}}$, between centroids of complementary clusters $I$ and $\complementLCL{I}{\tree}$ and passes through the midpoint, $\ctrdmid{I,\tree}{\vectbf{\state}}$, of their centroids,\footnote{Here,  $\tr{\mat{A}}$ denotes the transpose of $\mat{A}$.}
%
\begin{align}
\sepmag{i,I,\tree}{\vectbf{\state}} &\ldf
\vectprod{ \prl{\big. \vect{\state}_i - \ctrdmid{I,\tree}{\vectbf{\state}}}}{\frac{\ctrdsep{I,\tree}{\vectbf{\state}}}{\norm{\ctrdsep{I,\tree}{\vectbf{\state}}}}}, \label{eq.SepMag} 
\end{align}
where the associated ``cluster functions'' of a partial configuration, $\vectbf{x}|I = \prl{\vect{x}_i}_{i \in I}$, are defined as  
\begin{align} 
\ctrd{\vectbf{\state}|I} & \ldf \frac{1}{|I|}  \sum_{i \in I} \vect{\state}_i, \label{eq.centroid} 
\\
\ctrdsep{I,\tree}{\vectbf{\state}}  &\ldf \ctrd{\vectbf{\state}| I}  - \ctrd{\vectbf{\state} | \complementLCL{I}{\tree}} , \label{eq.centroidseparation}
\\
\ctrdmid{I,\tree}{\vectbf{\state}}  & \ldf \frac{\ctrd{\vectbf{\state} |I} + \ctrd{\vectbf{\state} |\complementLCL{I}{\tree}}}{2}.  \label{eq.centroidmidpoint}
\end{align}

\smallskip 

We now follow \cite{Baryshnikov_Guralnik} in defining terminology and expresssions leading to the characterization of the homotopy type of the stratum, $\stratum{\tree},$ associated with a nondegenerate hierarchy.  
The proofs of our formal statements all follow the same pattern as established in \cite{Baryshnikov_Guralnik}, and we omit them to save space here.
\begin{definition} \label{def.NarrowStandardConf}
A configuration $\vectbf{x} \in \cRdJr$ is \emph{narrow} relative to
the  split,  $\crl{I, J \setminus I}$, if
\begin{equation}
\max_{A \in \crl{I, J \setminus I}} \radiusCL{}\prl{\vectbf{x}|A} \; <  \;  \frac{1}{2}\norm{\big. \ctrd{\vectbf{x}| I} - \ctrd{\vectbf{x}| J \setminus I} },
\end{equation}
where the radius of a cluster, $A \subset J$,  is defined to be\footnote{Recall from \refsec{sec.ConfSpace} that $r_i$ denotes the radius of $i$th sphere for any $i \in \indexset$. }
\begin{equation} \label{eq.ConfRadius} %\label{eq.CentroidalDisplacement}
\radiusCL{}\prl{\vectbf{x}|A} := \max_{a \in A} \prl{\big.\norm{\vect{x}_a - \ctrd{\vectbf{x}|A}} + r_a}.
\end{equation} 
Say that $\vectbf{x} \in \stratum{\tree}$ is a \emph{standard} configuration relative to the nondegenerate hierarchy, $\tree \in \bintreetopspace_{\indexset}$, if
it is narrow relative to each local split, $\childCL{I, \tree}$, of
every cluster,  $I \in \cluster{\tree}$. 
\end{definition}

\begin{figure}[b] 
\centering
\begin{tabular}{@{}c@{\hspace{2mm}}c@{}}
\includegraphics[width = 0.24\textwidth]{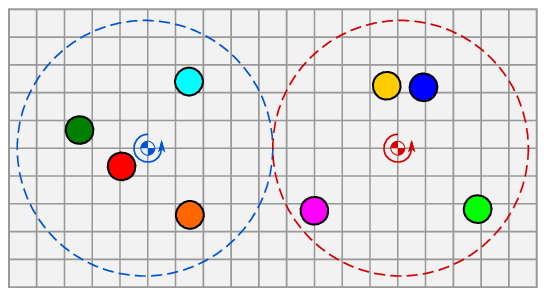}
 &
\includegraphics[width = 0.24\textwidth]{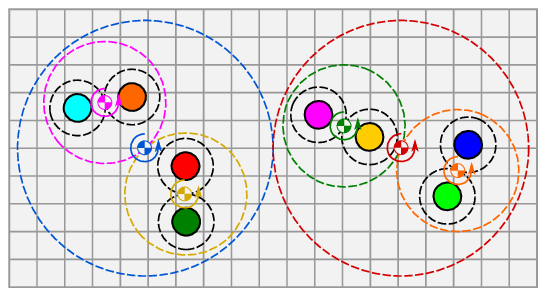}
\end{tabular} 
\vspace{-3mm}
\caption{An illustration of (left) narrow and (right) standard disk configurations, where arrows and dashed circles indicate clusters that can be rigidly rotated around their centroids while preserving their clustering structures.}
\label{fig.NarrowStandard}
\end{figure}

\begin{proposition} \label{prop.RotationAroundStandard}
If $\vectbf{x} \in \stratum{\tree}$ is a standard configuration then for each cluster, $I\in \cluster{\tree}$, any rigid rotation of the partial configuration, $\vectbf{x}|I$, around its centroid, $\ctrd{\vectbf{x}|I}$, as illustrated in \reffig{fig.NarrowStandard}, preserves the supported hierarchy $\tree$.
\end{proposition}

\begin{proposition}\label{prop.StrongDeformationRetraction}
For any finite label set $ \indexset\subset \N$  and non-degenerate tree $\tree \in \bintreetopspace_{\indexset}$, there exists a strong deformation retraction\footnotemark   
\begin{equation}
R_{\tree}: \stratum{\tree} \times [0, 1] \rightarrow \stratum{\tree}
\end{equation}
of $\stratum{\tree}$ onto the subset of standard configurations of $\stratum{\tree}$.
\end{proposition}
\footnotetext{In \cite{Baryshnikov_Guralnik} authors study point particle configurations, and they construct a strong deformation retraction onto standard configurations by shrinking clusters around cluster centroids; and one can obtain similar result for thickened spheres by properly expanding  cluster configurations  instead of shrinking.}
%
%\begin{proof}
%The proof follows a similar approach to \cite{guralnik_baryshnikov} for the bar%ycentric clustering in .
%\end{proof}
These two observations now yield the key insight reported in 
\cite{Baryshnikov_Guralnik}.

\begin{theorem} \label{thm.StrataHomotopy}
The set of configurations $\vectbf{x} \in \cRdJr$ supporting  a non-degenerate tree  has the homotopy type of $(\Sp^{d-1})^{\card{J}-1}$.  
\end{theorem}
%
%\begin{proof}
%The proof uses the same argument in \cite{guralnik_baryshnikov} for the barycen%tric clustering.
%\end{proof}
To gain an intuitive appreciation, one can restate this result as follows: two configurations in $\stratum{\tree}$ are topologically equivalent  if and only if the corresponding  separating hyperplane normals of configuration clusters are the same.\footnote{Note that a binary hierarchy over the leaf set $\indexset$ has $\card{\indexset} -1$ interior nodes, i.e. nonsingleton clusters \cite{robinson_jct1971}.}
Hence, navigation in a hierarchical stratum is carried out by aligning separating hyperplane normals, illustrated in \reffig{fig.HomotopyIntution}; and using this geometric intuition,  we construct in \cite{arslanEtAl_Allerton2012} a family of hierarchy preserving  control policies for point particle configurations, and in the following  we extend that construction to thickened disk configurations.

\begin{figure}[h]
\centering
\includegraphics[width = 0.46\textwidth]{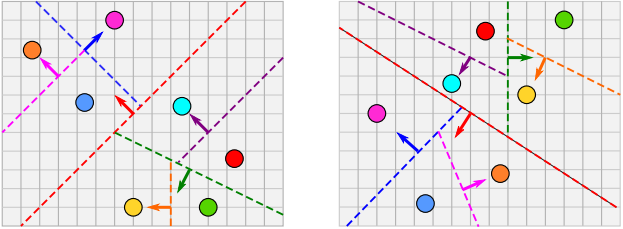}
\vspace{-3mm}
\caption{The topological shape of a hierarchical stratum intuitively suggests that global navigation in a hierarchical stratum is accomplished by aligning separating hyperplanes of configurations.}
\label{fig.HomotopyIntution} 
\end{figure}

\begin{theorem}\label{thm.HC2meansProperties}
Iterative 2-means clustering \emph{$\HCkmeans$} is a multi-function, and  each of its stratum, $\stratum{\tree}$ associated with $\tree \in \bintreetopspace_{\indexset}$, is connected and has an open interior.
\end{theorem}
\begin{proof}
It is well known that $k$-means clustering is a multi-function generally yielding different $k$-partitions of any given data, and so is $\HCkmeans$ (\refP{P.HCmultifunction})\cite{jain_dubes_1988, witten_frank_hall_DM2011}.
Further, it follows from \refdef{def.NarrowStandardConf} and \refprop{prop.StrongDeformationRetraction} that standard configurations in $\stratum{\tree}$ is open (\refP{P.HCopeninterior}), and \refthm{thm.StrataHomotopy} guarantees the connectedness  of $\stratum{\tree}$ (\refP{P.HCconnected}). \qed  
\end{proof}

\subsection{Hierarchy Preserving Navigation}
\label{sec.HierInvNav}
%%%%%%%%%%%%%%%%%%%%%%%%%%%%%%%%%%%%%
%%%%%%%%%%%%%%%%%%%%%%%%%%%%%%%%%%%%%

We now introduce a recursively defined vector field  for navigation in a hierarchical stratum and list its invariance and stability properties.

Suppose that some desired configuration, $\vectbf{\stateB} \in \stratum{\tree}$ has been selected,  supporting some desired nondegenerate tree, $\tree \in \bintreetopspace_{\indexset}$.
Our dynamical planner takes the form of a centralized hybrid controller,  $\f : \stratum{\tree} \rightarrow \prl{\R^d}^{\card{\indexset}}$, defining a hierarchy-invariant  vector  field whose flow  in $\stratum{\tree}$ yields the desired goal configuration, $\vectbf{\stateB}$,  recursively defined according to logic presented in \reftab{tab.HierInvNav}. 
Throughout this section, the tree $\tree$ and the goal configuration $\vectbf{\stateB}$ are fixed, and we therefore suppress all mention of these terms wherever convenient, in order to compress the notation. 
For example, for any $\vectbf{x} \in \stratum{\tree}$, $I \in \cluster{\tree}$  and $i \in I$ we use the shorthand  $\sepmag{i,I}{\vectbf{x}} = \sepmag{i,I,\tree}{\vectbf{x}}$ \refeqn{eq.SepMag}, $\ctrdsep{I}{\vectbf{x}} = \ctrdsep{I,\tree}{\vectbf{x}}$ \refeqn{eq.centroidseparation}, $\ctrdmid{I}{\vectbf{x}} = \ctrdmid{I,\tree}{\vectbf{x}}$ \refeqn{eq.centroidmidpoint} and so on. 

\begin{table}[t]
\caption{The Hierarchy-Preserving Navigation Vector Field}
\label{tab.HierInvNav}
\vspace{-1mm}
\centering 
\begin{tabular}{|p{0.45\textwidth}@{\hspace{2mm}}|}
\hline
\vspace{-1mm}
For any initial  $\vectbf{\stateA} \in \stratum{\tree}$ and desired $\vectbf{\stateB} \in \stratum{\tree}$, supporting $\tree \in \bintreetopspace_{\indexset}$, the hierarchy preserving vector field, $\f: \stratum{\tree} \rightarrow \prl{\R^d}^{\indexset}$,
\begin{align}
\f\prl{\vectbf{x}} \ldf \fhat\prl{\vectbf{x}, \vectbf{0}, \indexset}, \nonumber
\end{align}
is recursively computed using the post-order traversal\footnotemark \; of $\tree$ starting at the root cluster $\indexset$ with the zero control input $\vectbf{0} \in \prl{\R^d}^{\indexset}$
as follows: 
for any $\vectbf{u} \in \prl{\R^d}^{\indexset}$ and $I \in \cluster{\tree}$,
\vspace{-2mm}

\begin{tabular}{@{\hspace{-1mm}}c@{\hspace{3mm}}c@{\hspace{2mm}}c}
\begin{tabular}{@{}p{0.001\textwidth}@{}}
\\
$\rotatebox{90}{\hspace{-6mm}Base Cases}\left \{ \begin{array}{c}
\\[0.5mm]
\\[0.5mm]
\\[0.5mm]
\\[0.5mm]
\end{array} \right.$ 
\\[0.5mm]
$\rotatebox{90}{\hspace{-4.5mm}Recursion}\left \{ \begin{array}{c}
\\[0.5mm]
\\[0.5mm]
\\[0.5mm]
\\[0.5mm]
\\[0.5mm]
\\[0.5mm]
\end{array} \right.$
\end{tabular}
&
\begin{tabular}{@{}p{0.25\textwidth}@{}}

\begin{enumerate}[itemsep=0.5mm]
 \item \textbf{function} $\hat{\vectbf{u}} = \fhat\prl{\vectbf{x}, \vectbf{u}, I}$  
 \item \quad \textbf{if} $\vectbf{x} \in \setA\prl{I}$ \refeqn{eq.SetA},
 \item \quad \quad $\hat{\vectbf{u}} \leftarrow \fA\prl{\vectbf{x}, \vectbf{u}, I}$ \refeqn{eq.AttractiveField},
 \item \quad \textbf{else if} $\vectbf{x} \not \in \setH\prl{I}$ \refeqn{eq.SetR},
 \item \quad \quad $\hat{\vectbf{u}} \leftarrow  \fS\prl{\vectbf{x}, \vectbf{u}, I} $   \refeqn{eq.SeparationField},
 \item \quad \textbf{else} 
 \item \quad  \quad $\crl{I_L, I_R} \leftarrow \childCL{I,\tree}$,
 \item \quad \quad $\hat{\vectbf{u}}_L \leftarrow \fhat\prl{\vectbf{x}, \vectbf{u}, I_L}$,
 \item \quad \quad $\hat{\vectbf{u}}_R \leftarrow \fhat\prl{\vectbf{x}, \hat{\vectbf{u}}_L, I_R}$,
 \item \quad \quad $\hat{\vectbf{u}} \leftarrow  \fH\prl{\vectbf{x}, \hat{\vectbf{u}}_R, I} $ \refeqn{eq.SplitPreservingField},
 \item \quad \textbf{end}
 \item \textbf{return} $\hat{\vectbf{u}}$
 \vspace{-3mm}
\end{enumerate}
\end{tabular}
&
\begin{tabular}{@{}l@{}}
\\[0.6mm]
\\[0.6mm]
\% Attracting Field \\[0.6mm]
\\[0.6mm]
\% Split Separation Field \\[0.6mm]
\\[0.6mm]
\\[0.6mm]
\% Recursion for Left Child\\[0.6mm]
\% Recursion for Right Child\\[0.6mm]
\% Split Preserving Field\\[0.6mm]
\\
\end{tabular}
\end{tabular}

\\
\hline
\end{tabular}
\end{table}
\footnotetext{ The recursion step at any nonsingleton cluster $I \in \cluster{\tree}$ in \reftab{tab.HierInvNav}.8)-\ref{tab.HierInvNav}.10) updates the vector field $\f$  in a bottom-up fashion, first for the children clusters of $I$ and then for cluster $I$ itself, yielding a specific order in which the clusters of $\tree$ are visited; and such a tree traversal is formally referred to as the post-order tree traversal of $\tree$ \cite{Cormen_2009}.
}

In brief, the hierarchy invariant vector field $\f$ recursively detects partial configurations  whose separating hyperplanes are ``sufficiently aligned'' with the desired ones, as specified in \refeqn{eq.SetA} and illustrated in \reffig{fig.SufficientlyAligned}, and that can be directly moved towards the desired configurations, using a family of attracting fields $\fA$ \refeqn{eq.AttractiveField}, with no collisions along the way.
Once the partial configurations associated with sibling clusters $I$ and $\complementLCL{I}{\tree}$ of $\tree$ are in the domains of their associated attracting fields, $\f$  rotates these
partial configurations while preserving the hierarchy so that
their separating hyperplane is also asymptotically aligned.
Hence, $\f$  asymptotically aligns the separating hyperplanes of clusters of $\tree$ in a bottom-up fashion; and once the separating hyperplanes of all clusters of $\tree$ are ``sufficiently aligned'', $\f$ drives asymptotically  each disk directly towards its desired location.  
We now present and motivate  its constituent formulae as follows. 

\begin{figure}[h]
\vspace{2mm}
\centering
\includegraphics[width=0.47\textwidth]{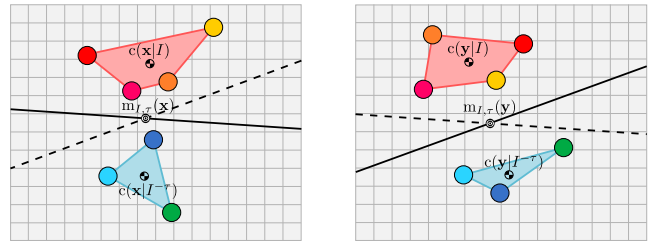}
\vspace{0mm} 
\caption{An illustration of ``sufficiently aligned'' separating hyperplanes of complementary clusters $I$ and $\complementLCL{I}{\tree}$ of $\tree$. 
Both of the current (left) and desired (right) partial configurations are linearly separable by each other’s separating hyperplane, and such an alignment condition needs to be satisfied at each level of the subtrees rooted at $I$ and $\complementLCL{I}{\tree}$ so that the partial configurations $\vectbf{x}|I$ and $\vectbf{x}|\complementLCL{I}{\tree}$ are steered by the associated attracting fields.}
\label{fig.SufficientlyAligned}
\end{figure}

\smallskip

The hierarchy-invariant vector field, $\f$, in \reftab{tab.HierInvNav}.2) \& \ref{tab.HierInvNav}.3) recursively detects partial configurations, $\vectbf{x} | I$ associated with cluster $I \in \cluster{\tree}$, that can be safely driven toward the goal formation in $\stratum{\tree}$  using a family of attracting controllers, $\fA : \stratum{\tree} \times \prl{\R^d}^{\indexset} \times \cluster{\tree} \rightarrow \prl{\R^d}^{\indexset}$,  defined in terms of the negated gradient field of $\V\prl{\vectbf{x}} \ldf \frac{1}{2}\norm{\vectbf{x} - \vectbf{\stateB}}_2^2$: for any $j \in \indexset$,
\begin{align}\label{eq.AttractiveField}
\fA\prl{\vectbf{x}, \vectbf{u}, I}_{j} 
\sqz{\ldf} \left \{ 
\begin{array}{@{}c@{}l}
 - (\vect{x}_j \sqz{-} \vect{\stateB}_j) & \text{, if } j \in I, \\
\vect{u}_j & \text{, else,}  
\end{array}
\right .  \!\!\!\!
\end{align}
where $\vectbf{u} \in \prl{\R^d}^{\indexset}$ is a desired (velocity) control input  specifying the motion of complementary cluster $\indexset \setminus I$.

To avoid intra-cluster collisions along the way  and preserve (local) clustering hierarchy, for any $I \in \cluster{\tree}$ the set of configurations in the domain  of the attracting field, $\fA$, is restricted to
\begin{align} \label{eq.SetA}
\hspace{-2mm}\setA\!\prl{I} \!\sqz{\ldf}\! \left\{\Big. \right. \!\! \vectbf{x} \sqz{\in} \stratum{\tree} \! \Big | \,  & \lie_{\overrightarrow{\vectbf{y}}}  \scalebox{1}{$\frac{1}{2}$}\norm{\vect{x}_i - \vect{x}_j}^2 \geq \prl{r_i \sqz{+} r_j}^2\!\!, \, \forall i \sqz{\neq}\! j \sqz{\in} I \!, \nonumber \\
& \hspace{-26mm} \lie_{\overrightarrow{\vectbf{y}}} \vectprod{\prl{\vect{x}_k \sqz{-} \ctrdmid{K\!}{\vectbf{x}}\!}\!}{\ctrdsep{K\!}{\vectbf{x}}} \sqz{\geq} 0, \forall k \sqz{\in} K\!, K  \sqz{\in} \descendantCL{I,\tree} \!\!  \left . \Big. \right \} \! \!,  \!  \! \! 
\end{align}
where $\descendantCL{I,\tree}$ is the set of descendants of $I$ in $\tree$. 
Here,  $\lie_{\overrightarrow{\vectbf{y}}}f$ denotes the Lie derivative of a scalar-valued function $f$ along a constant vector field $\overrightarrow{\vectbf{y}}$ which assigns the same vector $\vectbf{y}$ to every point in its domain,  and  one can simply verify that
\begin{align}
\lie_{\overrightarrow{\vectbf{y}}} \scalebox{1}{$\frac{1}{2}$}\norm{\vect{x}_i - \vect{x}_j}^2 &= \vectprod{\prl{\vect{x}_i - \vect{x}_j}}{\prl{\vect{y}_i - \vect{y}_j}}, \label{eq.SetACond1}
\\
\lie_{\overrightarrow{\vectbf{y}}} \vectprod{\prl{\vect{x}_k \sqz{-} \ctrdmid{K\!}{\vectbf{x}}\!}\!}{\ctrdsep{K\!}{\vectbf{x}}}   & = \vectprod{\prl{\vect{y}_k \sqz{-} \ctrdmid{K\!}{\vectbf{y}}\!}\!}{\ctrdsep{K\!}{\vectbf{x}}} \nonumber 
\\
& \hspace{9mm} + \vectprod{\prl{\vect{x}_k \sqz{-} \ctrdmid{K\!}{\vectbf{x}} \!}\!}{\ctrdsep{K\!}{\vectbf{y}}} \!. \!\! \label{eq.SetACond2}
\end{align}
Note that \refeqn{eq.SetACond1} quantifies the safety of a resulting trajectory of $\fA$, and to avoid  collision between any pair of disks, $i$ and $j$, \refeqn{eq.SetACond1} should be no less than the square of sum of their radii, $\prl{\radius_i + \radius_j}^2$, as required in \refeqn{eq.SetA}; and \refeqn{eq.SetACond2} quantifies the preservation of (local) clustering hierarchy and  should be nonnegative for hierarchy invariance.  
Also observe that since a singleton cluster  contains no pair of distinct indices, and has an empty set of descendants, the predicate in \refeqn{eq.SetA} is always true for these ``leaf'' node cases and we have $\setA\prl{I} = \stratum{\tree}$ for any singleton cluster $I \in \cluster{\tree}$. 
Further, one can simply verify that  $\vectbf{\stateB} \in \setA\prl{I}$ for any $I \in \cluster{\tree}$.

\smallskip

If a partial configuration,  $\vectbf{x}|I$, is not contained in the domain of the associated attracting field, i.e.  $\vectbf{x} \not \in \setA\prl{I}$, to avoid  inter-cluster collisions the failure of the condition in  \reftab{tab.HierInvNav}.4) ensures sibling clusters, $\childCL{I,\tree}$, will be separated by a certain distance, specified as:   
{\small
\begin{align} \label{eq.SetR}
\hspace{-2mm}\setH\!\prl{I} \sqz{\sqz{\ldf}} \crl{  \vectbf{x} \sqz{\in} \stratum{\tree} \! \Big |    \sepmag{k,K\!}{\vectbf{x}} \sqz{\geq} r_k \!\sqz{+} \alpha,  \forall k \sqz{\in} K, K \sqz{\in}  \childCL{I,\tree}  \! }\!,  \! \!\!
\end{align}
}%
where $\sepmag{k, K\!}{\vectbf{x}}$ \refeqn{eq.SepMag} returns the perpendicular distance of $k$th agent to the separating hyperplane of cluster $K \in \cluster{\tree}$, and  $\alpha > 0$ is a safety margin guaranteeing that the clearance between any pair of disks in complementary clusters, $\childCL{I,\tree}$, is at least $2\alpha$ units.
%, i.e. $\norm{\vect{x}_i - \vect{x}_j}_2 \geq r_i + r_j + 2\alpha$ for any $\vectbf{x} \in \setR_{\tree}\prl{I}$ and $i \in I_L$, $j \in I_R$.
Observe that   $\setH\prl{I} = \stratum{\tree}$ for any singleton cluster $I \in \cluster{\tree}$ because such leaf clusters of a
binary tree have no children, i.e. $\childCL{I,\tree} = \emptyset$.   

\smallskip

While the disks move in $\setH\prl{I}$ based on a desired control (velocity) input $\vectbf{u} \in \prl{\R^d}^{\indexset}$, \reftab{tab.HierInvNav}.10) guarantees the maintenance of the safety margin between children clusters $\childCL{I,\tree}$ by employing an additive repulsive field, $\fH : \stratum{\tree} \times \prl{\R^d}^{\indexset} \times \cluster{\tree} \rightarrow \prl{\R^d}^{\indexset}$, as follows:    
\begin{align} \label{eq.SplitPreservingField}
\fH\prl{\vectbf{x}, \vectbf{u}, I}_{j} 
\ldf \vect{u}_j + 2 \falpha{I}{\vectbf{x}, \vectbf{u}}  \frac{\card{\complementLCL{K}{\tree}}}{\card{I}}  \frac{\ctrdsep{K\!}{\vectbf{x}}}{\norm{\ctrdsep{K\!}{\vectbf{x}}}} ,\!\!
\end{align}
for all $j \sqz{\in} K$ and $K \sqz{\in} \childCL{I, \tree}$; otherwise, $\fH\prl{\vectbf{x}, \vectbf{u}, I}_{j} \!\ldf \vect{u}_j $, 
where $\falpha{I}{\vectbf{x}, \vectbf{u}}$ is a scalar valued function describing the strength of the repulsive field, 
\begin{align}\label{eq.Alpha}
\falpha{I}{\vectbf{x}, \vectbf{u}} \sqz{\ldf} \max_{\substack{k \in K \\ K \in \childCL{I,\tree}}} \fphi{k,K\!}{\vectbf{x}} \cdot \fpsi{k,K\!}{\vectbf{x}, \vectbf{u}}.
\end{align}
Here, for each individual $k$ in cluster $K \in \childCL{I,\tree}$, $\fphi{k,K}{\vectbf{x}}$ is  exponential damping on the repulsion strength $\fpsi{k,K}{\vectbf{x}, \vectbf{y}}$, in which the amplitude envelop exponentially decays to zero after a certain safety margin $\beta > \alpha$,
\begin{align}
\hspace{-2mm}\fphi{k,K\!}{\vectbf{x}} & \sqz{\ldf} \max \prl{\frac{{e^{-\prl{\sepmag{k,K}{\vectbf{x}}  - r_k - \alpha}} \sqz{-} e^{-\prl{\beta - \alpha}}}}{{1 - e^{-\prl{\beta - \alpha}}}}, 0\!}\!\!,\\
\hspace{-2mm}\fpsi{k,K\!}{\vectbf{x}, \vectbf{u}} & \sqz{\ldf} \max  \prl{\!\! \!\Big. - \!\!\prl{\big.\sepmag{k,K\!\!}{\vectbf{x}}  \sqz{-} r_k \sqz{-} \alpha} \!\sqz{-} \lie_{\overrightarrow{\vectbf{u}}}\sepmag{k,K\!\!}{\vectbf{x}} \!, 0 \!}\!\!, \!\!\! \!
\end{align}
where
\begin{align}\label{eq.sepmagLie}
\lie_{\overrightarrow{\vectbf{u}}}\sepmag{k,K\!}{\vectbf{x}} &= \frac{\vectprod{\prl{\vect{u}_k \sqz{-} \ctrdmid{K\!}{\vectbf{u}}\!}\!}{\ctrdsep{K\!}{\vectbf{x}}} \sqz{+} \vectprod{\prl{\vect{x}_k \sqz{-} \ctrdmid{K\!}{\vectbf{x}} \!}\!}{\ctrdsep{K\!}{\vectbf{u}}}}{\norm{\ctrdsep{K\!}{\vectbf{x}}}} \nonumber \\
& \hspace{15mm}- \sepmag{k,K\!}{\vectbf{x}} \frac{\vectprod{\ctrdsep{K\!}{\vectbf{x}}\!}{\ctrdsep{K\!}{\vectbf{u}}}}{\norm{\ctrdsep{K\!}{\vectbf{x}}}^2}. 
\end{align}
Note that  $\fH\prl{\vectbf{x}, \vectbf{u}, I} $ is well defined for any singleton cluster $I \in \cluster{\tree}$ and  is equal to the identity map, i.e. $\fH\prl{\vectbf{x}, \vectbf{u}, I}  = \vectbf{u}$, since $\childCL{I,\tree} = \emptyset$; and also observe that $\fH\prl{\vectbf{x}, \vectbf{u}, I}  = \vectbf{u}$ for any $I \in \cluster{\tree}$ if the complementary clusters $\childCL{I,\tree}$ are well-separated, i.e. $\sepmag{k,K}{\vectbf{x}}\geq r_k + \beta$ for all $k \in K$ and $K \in \childCL{I,\tree}$.
The latter is important to avoid  the  ``finite escape time" phenomenon\footnote{ A trajectory of a dynamical system is said to have a finite escape time  if it  escapes to infinity at a finite time \cite{khalil_NonlinearSystems_2001}.} (\refprop{prop.HierInvNavExistenceUniqueness}).  

\smallskip

Finally, \reftab{tab.HierInvNav}.5) guarantees that if a partial configuration is neither  in the domain of the attracting field nor are its children clusters, $\childCL{I,\tree}$, properly separated, i.e. $\vectbf{x} \not \in  \setA\prl{I} \cup \setH\prl{I}$, then the complementary clusters are driven apart using another repulsive field, $\fS: \stratum{\tree} \times \prl{\R^d}^{\indexset} \times \cluster{\tree} \rightarrow \prl{\R^d}^{\indexset}$, until asymptotically establishing a certain safety margin $\beta > \alpha$: 
{
\begin{align} \label{eq.SeparationField}
\hspace{-1mm} \fS\prl{\vectbf{x}, \vectbf{u}, I}_{j} 
\sqz{\ldf} -\ctrd{\vectbf{x} \sqz{-} \vectbf{y}|I} +  2\fbeta{I\!}{\vectbf{x}} \!\frac{\card{\complementLCL{K}{\tree}}}{\card{I}}  \frac{\ctrdsep{K\!}{\vectbf{x}}}{\norm{\ctrdsep{K\!}{\vectbf{x}}}}, \!\! \! 
\end{align}
}%
for all $j \sqz{\in} K$ and $K \sqz{\in} \childCL{I, \tree}$; otherwise, $\fS\prl{\vectbf{x}, \vectbf{u}, I}_{j} \!\ldf \vect{u}_j $, where the magnitude, $\fbeta{I}{\vectbf{x}}$, of repulsion between complementary clusters $\childCL{I,\tree}$ is given by
\begin{align}\label{eq.Beta}
\fbeta{I}{\vectbf{x}} \ldf \max _{\substack{k \in K \\ K \in \childCL{I,\tree}}} \!\! \max \prl{\big.\!-\prl{\big.\sepmag{k,K}{\vectbf{x}} \sqz{-} r_k \sqz{-} \beta}\!, 0} \! . \!
\end{align}
For completeness, we set $\fS\prl{\vectbf{x}, \vectbf{u}, I} = \fA\prl{\vectbf{x}, \vectbf{u}, I}$ for any singleton cluster $I \in \cluster{\tree}$.

\smallskip

We summarize the properties of this construction as follows:\footnote{This construction indeed solves \refprob{prob.HierarchyInvariant} since a flow is a retraction of its basin into the attractor \cite{bhatia_szego_1967}.}
\begin{theorem}\label{thm.HierInvNav}
The recursion of \reftab{tab.HierInvNav} results in a well-defined function $\f :\stratum{\tree} \rightarrow \prl{\R^d}^{\indexset}$ that can be computed in $\bigO{\card{\indexset}^2}$ time for any $\vectbf{x} \in\stratum{\tree}$. 
For all  $\tree \in \bintreetopspace_{\indexset}$, the stratum, $\stratum{\tree}$ is positive invariant and any $\vectbf{y} \in \stratum{\tree}$ is an asymptotically stable equilibrium point of a continuous piecewise smooth  flow arising from $\f$ whose basin of attraction includes all of $\stratum{\tree}$ with the exception of an empty interior set.
\end{theorem}
\begin{proof}
These results are proven in \refapp{app.HierInvNav} according to the following plan. 
\refprop{prop.HierInvNavProperty} establishes that the recursion in \reftab{tab.HierInvNav} indeed results in a function computable in quadratic time.  
The invariance, stability, and continuous flow generating properties of $\f$ are shown using an equivalent system model within the sequential composition framework \cite{Burridge_Rizzi_Koditschek_1999}, as follows.
\reftab{tab.LocalPolicy} defines a new recursion shown  in \refprop{prop.LocalPolicyProperty} to result in a family of continuous and piecewise smooth vector fields.
\refprop{prop.DomainInclusion} asserts  that the family of domains associated with these fields \refeqn{eq.LocalPolicyDomain} defines a (finite) open cover of  $\stratum{\tree}$ relative to which a selection function (\reftab{tab.PolicySelection}) induces a partition of that stratum. 
\refprop{prop.SystemEquivalence} demonstrates that the composition of the covering vector field family with the output of  this partitioning function yields a new function that coincides exactly with the  original control field defined in  \reftab{tab.HierInvNav}. 
Finally, \refprop{prop.HierInvNavExistenceUniqueness}, \refprop{prop.HierInvNavInvariance} and \refprop{prop.HierInvNavStability} demonstrate, respectively, the flow, positive invariance and stability properties of $\f$,  which are inherited from the flow, invariance and stability properties (\refprop{prop.LocalPolicyExistenceUniqueness}, \refprop{prop.DomainInvariance} and \refprop{prop.PreparesRelation}, respectively) of substratum policies executed over a strictly decreasing finite prepares graph (\refprop{prop.PreparesGraph})  via their nondegenerately, real-time  executed (\refprop{prop.NonZeroExecutiontime})  sequential composition.   \qed
\end{proof}

%%%%%%%%%%%%%%%%%%%%%%%%%%%%%%%%%%%%%%%%%%%%%%%%%%%
%%%%%%%%%%%%%%%%%%%%%%%%%%%%%%%%%%%%%%%%%%%%%%%%%%%
\subsection{Navigation in the Space of Binary Trees}
\label{sec.NNINav}
%%%%%%%%%%%%%%%%%%%%%%%%%%%%%%%%%%%%%%%%%%%%%%%%%%
%%%%%%%%%%%%%%%%%%%%%%%%%%%%%%%%%%%%%%%%%%%%%%%%%%

In principle,  navigation in the adjacency graph of trees (\refprob{prob.Transition}) is a trivial matter since the number of trees over a finite set of leaves is finite. 
However, in practice, the cardinality of trees grows super exponentially  \cite{billera_holmes_vogtmann_aap2001},
{
\begin{align}\label{eq.NumTree}
\card{\bintreetopspace_{\indexset}} &= (2\card{\indexset} - 3)!! = (2\card{\indexset}-3)(2\card{\indexset}-5) \ldots 3, 
\end{align}  
}%
for $\card{\indexset} \geq 2$. 
Hence standard graph search algorithms, like the A* or Dijkstra's algorithm \cite{Cormen_2009}, become rapidly impracticable. 
In particular, computing the shortest path (geodesic) in the NNI-graph, a regular subgraph of the adjacency graph (\refthm{thm.Portal}), is NP-complete \cite{dasguptaEtAl_SDM1997}.
%; and even embedding $\bintreetopspace_{\indexset}$ into a continuous metric space  (through the sort of dendrogram structure that we might imagine appending to our trees through appropriate assignment of distances in configuration space \cite{carlsson_memoli_jmlr2010}) results in geodesic algorithms whose ``cheapest" known instances are $\bigO{\card{\indexset}^4}$ \cite{owen_provan_tcbb2011}. 

Alternatively, we have recently developed in \cite{ArslanEtAl_NNITechReport2013}  an efficient recursive procedure for navigating in the NNI graph $\NNIgraph_{\indexset} = \prl{\bintreetopspace_{\indexset}, \NNIedgeset}$ towards any given binary tree  $\tree \in \bintreetopspace_{\indexset}$, taking the form of an abstract discrete dynamical system as follows:
\begin{subequations} \label{eq:SystemModel}%
\begin{align}
\treeA^{k+1} &= \NNI\prl{\big. \treeA^k, G^k},  \label{eq:NNITransition}\\
 G^k &= \vectbf{u}_{\tree}(\treeA^k), \label{eq:NNIControl} 
\end{align}
\end{subequations}
where $\NNI\prl{\treeA^k, G^k} $ denotes the NNI move\footnote{Here, note that the NNI move at the empty cluster  corresponds to the  identity map in $\bintreetopspace_{\indexset}$, i.e. $\treeA = \NNI\prl{\treeA, \emptyset}$ for all $\treeA \in \bintreetopspace_{\indexset}$.
Therefore, the notion of identity map in $\bintreetopspace_{\indexset}$ slightly extends  the NNI graph by adding  self-loops at every vertex, which  is necessary for a discrete-time  dynamical system in $\bintreetopspace_{\indexset}$ to have  fixed points.
} on $\treeA^k$ at  cluster   $G^k \in \cluster{\tree}$, illustrated in \reffig{fig.NNIMove}, and $\vectbf{u}_{\tree}$ is our NNI control policy returning an NNI move as summarised in \reftab{tab.NNIControl}.
Abusing notation, we shall denote the closed-loop dynamical system as
\begin{align} \label{eq:CLSystemModel}
\treeA^{k+1} = g_{\tree}\prl{\treeA^k} \ldf \prl{\NNI \circ \vectbf{u}}\prl{\treeA^k}.% = \NNI\prl{\treeA^k, \vectbf{u}\prl{\treeA^k}}.
\end{align}

\begin{table}[h]
\caption{The NNI Control Law}
\label{tab.NNIControl}
\centering 
\vspace{-2mm}
\begin{tabular}{|p{0.45\textwidth}@{\hspace{2mm}}|}
\hline
\vspace{-1mm}

To navigate from an arbitrary hierarchy $\treeA \in \bintreetopspace_{\indexset}$ towards any selected desired hierarchy $\treeB \in \bintreetopspace_{\indexset}$ in the NNI-graph,  the NNI control policy $\vectbf{u}_{\treeB}$ returns an NNI move on $\treeA$ at a cluster $G \in \cluster{\treeA}$, as follows: 

\begin{enumerate}
\item If $\treeA = \treeB$, then just return the identity move,  $G = \emptyset$. 

\item Otherwise, 

\begin{enumerate}

\item Select a common cluster $K \in \cluster{\treeA} \cap \cluster{\treeB}$ with $\childCL{K,\treeA} \neq \childCL{K,\treeB}$, and let $\crl{K_L, K_R} = \childCL{K,\treeB}$.
\item Find a nonsingleton cluster $I \in \cluster{\treeA}$  with children  $\crl{I_L, I_R} = \childCL{I,\treeA}$ satisfying  $I_L \subseteq K_L$ and $I_R \subseteq K_R$. 
\item Return a proper NNI navigation move  on $\treeA$ at  grandchild $G  \in \childCL{I, \treeA}$  selected as follows: 

\begin{enumerate}
\item If $ \complementLCL{I}{\treeA} \subset K_L$ , 
then return $G  = I_R$.
\item Else if $ \complementLCL{I}{\treeA} \subset K_R$ , then return $G  = I_L$.
\item Otherwise , return an arbitrary NNI move at a child of $I$ in $\treeA$; for example, $G = I_L$.
\end{enumerate}
\end{enumerate}
\vspace{-3mm}
\end{enumerate}  

\\
\hline

\end{tabular}

\end{table}

The NNI control law endows the NNI-graph with a  directed edge structure whose paths all lead to  $\tree$, and whose longest path (from the furthest possible initial hierarchy, $\treeA \in \bintreetopspace_{\indexset}$) is tightly bounded by $\frac{1}{2}\prl{\card{\indexset} -1}\prl{\card{\indexset} -2}$ for $\card{\indexset} \geq 2$. 
Given such a goal we show in \cite{ArslanEtAl_NNITechReport2013}
that the cost of computing an appropriate NNI move from any other $\treeA\in \bintreetopspace_{\indexset}$ toward an adjacent tree at a lower value of a ``discrete
Lyapunov function" relative to that destination is $\bigO{\card{\indexset}}$.  
We summarize such important properties of our NNI navigation algorithm as:
\begin{theorem}[\hspace{-0.06mm}\cite{ArslanEtAl_NNITechReport2013} \hspace{-1.5mm}]\label{thm.Transition}
 The NNI control law $\vect{u}_{\tree}$ (\reftab{tab.NNIControl}) recursively defines a closed loop  discrete dynamical system \refeqn{eq:CLSystemModel}  in the NNI-graph, taking the form of a discrete transition rule, $g_{\tree}$, with global attractor at $\tree$ and longest trajectory of length $\bigO{\card{\indexset}^2\!}$ endowed with a discrete Lyapunov function relative to which computing a descent direction from any $\treeA \in \bintreetopspace_{\indexset}$ requires a computation of time $\bigO{\card{\indexset}}$.  
\end{theorem}

%%%%%%%%%%%%%%%%%%%%%%%%%%%%%%%%%%%%%%%%%%%%%%%%%%%%%%%%%
%%%%%%%%%%%%%%%%%%%%%%%%%%%%%%%%%%%%%%%%%%%%%%%%%%%%%%%%%
\subsection{Portal Transformations}
\label{sec.PortalTransformation}
%%%%%%%%%%%%%%%%%%%%%%%%%%%%%%%%%%%%%%%%%%%%%%%%%%%%%%%%%
%%%%%%%%%%%%%%%%%%%%%%%%%%%%%%%%%%%%%%%%%%%%%%%%%%%%%%%%%

We now turn attention to construction of the crucial portal map  that effects the geometric realization of the NNI-graph as required for \refprob{prob.Portal}; and herein we extend our recent construction of the realization function, $\portalconf{}$, in \cite{arslan_guralnik_kod_WAFR2014} for point particle configurations to thickened disk configurations. %     iterative 2-means  clustering, $\HCkmeans$.

Throughout this section, the trees $\treeA,\treeB \in \bintreetopspace_{\indexset}$ are NNI-adjacent (as defined in \refsec{sec.TreeGraph}) and fixed, 
and we therefore take the liberty of suppressing all mention of these trees wherever convenient, for the sake of simplifying the presentation of our equations.

Since the trees $\treeA, \treeB$ are NNI-adjacent, we may apply Lemma 1 from \cite{ArslanEtAl_NNITechReport2013} to find common disjoint clusters $A,B,C$ such that $\crl{A \cup B} = \cluster{\treeA}\setminus\cluster{\treeB}$ and $\crl{B \cup C} = \cluster{\treeB} \setminus \cluster{\treeA}$. 
Note that the triplet $\crl{A,B,C}$ of the pair $\;\prl{\treeA,\treeB}$ is unique. 
We call $\crl{A,B,C}$ the \emph{NNI-triplet} of the pair $\prl{\treeA,\treeB}$.
Since $\treeA$ and $\treeB$ are fixed throughout this section, so will be $A,B,C$ and $P:=A\cup B\cup C$.

In the construction of the portal map, $\portalconf{}$ \refeqn{eq.portalconf}, we restrict our attention to the portal configurations with a certain  symmetry property, defined as:
\begin{definition} [\!\!\cite{arslan_guralnik_kod_WAFR2014}]
We call  $\vectbf{\state} \in \prl{\R^d}^{J}$ a \emph{symmetric} configuration associated with $\prl{\treeA,\treeB}$ if centroids of partial configurations $\vectbf{\state}|A$, $\vectbf{\state}|B$ and $\vectbf{\state}|C$ form an equilateral triangle, as illustrated in \reffig{fig.PortalConf}. 
The set of all symmetric configurations with respect to $\prl{\treeA,\treeB}$ is denoted $\mathtt{Sym}\prl{\treeA,\treeB}$.
\end{definition}

\begin{figure}[h]
\centering
\includegraphics[width=0.3\textwidth]{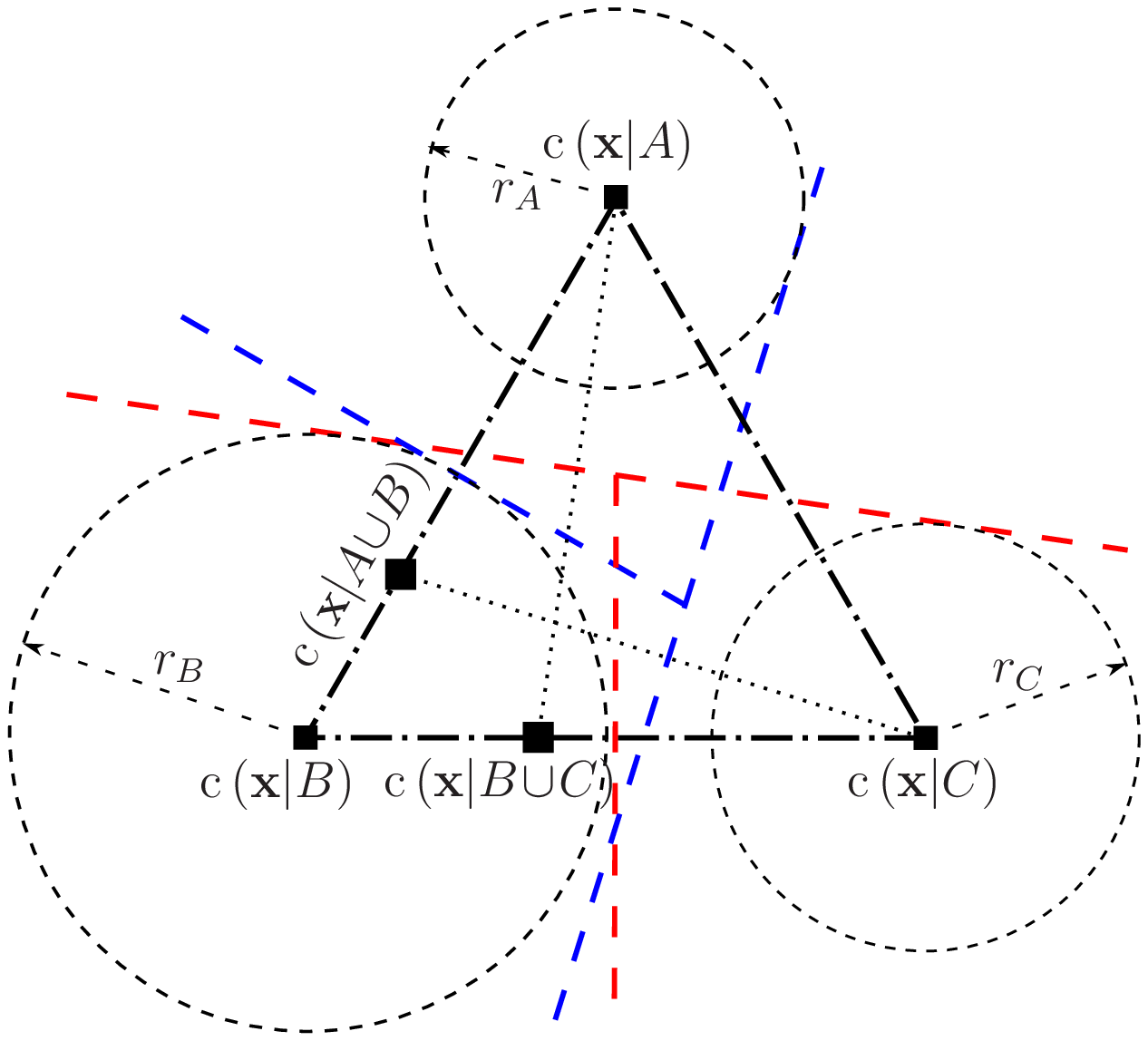} 
\vspace{-1mm}
\caption{An illustration of a symmetric  configuration $\vectbf{x} \in \mathtt{Sym}\prl{\treeA,\treeB}$,  where the consensus ball $B_{Q}\prl{\vectbf{x}}$ of partial configuration of cluster $Q \in \crl{A,B,C}$ has a positive radius. 
%Here, the colored dashed lines depict the separating hyperplanes of clusters.  
% Partial configurations, $\vectbf{\state}|A$, $\vectbf{\state}|B$, and $\vectbf{\state}|C$, whose centroids forming an equilateral triangle satisfy hierarchical constraints defining children $\crl{A,B}$ and  $\crl{ A \cup B, C}$ of clusters $I= A \cup B$ and $P= A \cup B \cup C$, respectively, of a hierarchy $\treeA \in \bintreetopspace_{\indexset}$,  with some positive centroidal radius. 
%This also holds for NNI-adjacent hierarchy $\treeB$ result from the NNI move on $\treeA$ at $A$.
}
\label{fig.PortalConf}
\end{figure}

An important observation about the symmetric configurations is:
\begin{lemma}[\!\!\cite{arslan_guralnik_kod_WAFR2014}]
Let $\vectbf{x} \in \stratum{\treeA}$ be a symmetric configuration in $\mathtt{Sym}\prl{\treeA,\treeB}$. 
If each partial configuration $\vectbf{x} |Q$ of cluster $Q \in \crl{A,B,C}$ is  contained in the associated ``consensus'' ball $B_{Q}\prl{\vectbf{\state}}$ --- an open ball\footnote{In a metric space $\prl{X, d}$, the open ball $B\prl{\vect{x},r}$ centered at $\vect{x}$ with radius $r \in \R_{\geq0}$ is the set of points in $X$ whose distance to $\vect{x}$ is less than $r$, i.e $B\prl{\vect{x},r} = \crl{\vect{y} \in X \; | \; d\prl{\vect{x},\vect{y}} < r}$.
} centered at $\ctrd{\vectbf{x}|Q}$ with radius  
\begin{align} \label{eq.RadiusStandardConf}
\hspace{-1mm}
\radiusDD{Q}\prl{\vectbf{\state}} & \sqz{\ldf}  \hspace{-7mm} \min_{\substack{\treeC  \; \in  \; \prl{\treeA,\treeB} \\ D \in \crl{ Q, \parentCL{Q,\treeC} } \setminus\crl{ P}}} \hspace{-7mm} - \; \vectprod{\prl{\big. \ctrd{\vectbf{\state}|Q} \sqz{-} \ctrdmid{D,\treeC}{\vectbf{\state}}\!}\!}{  \frac{\ctrdsep{D,\treeC}{\vectbf{\state}}}{\norm{\ctrdsep{D,\treeC}{\vectbf{\state}}}} }, \!\!   
\end{align}
then $\vectbf{x}$ also supports $\treeB$, i.e. $\vectbf{x} \in \stratum{\treeB}$, and so $\vectbf{x}$ is a portal configuration, $\vectbf{x} \in \portal\prl{\treeA, \treeB}$.
\end{lemma} 

\noindent Note that for any symmetric configuration $\vectbf{x} \in \mathtt{Sym}\prl{\treeA, \treeB}$ the consensus ball of each partial configuration of cluster $Q \in \crl{A,B,C}$ always has a nonempty interior, i.e.  $\radiusDD{Q}\prl{\vectbf{\state}}> 0$ \cite{arslan_guralnik_kod_WAFR2014} --- see \reffig{fig.PortalConf}.

In the following, we first describe how we relate any given triangle to an equilateral triangle using   Napoleon transformations, and then define our portal map. 

\medskip

%%%%%%%%%%%%%%%%%%%%%%%%%%%%%%%%%%%%%%%%%%%%%%%%
%%%%%%%%%%%%%%%%%%%%%%%%%%%%%%%%%%%%%%%%%%%%%%%%
\subsubsection{Napoleon Triangles} 
\label{sec.Napoleon}
%%%%%%%%%%%%%%%%%%%%%%%%%%%%%%%%%%%%%%%%%%%%%%%%
%%%%%%%%%%%%%%%%%%%%%%%%%%%%%%%%%%%%%%%%%%%%%%%%

We recall a theorem of geometry describing how to create an equilateral triangle from an arbitrary triangle: 
construct, either all outer or all inner, equilateral triangles at the sides of a triangle in the plane containing the triangle, and so centroids of the constructed equilateral triangles form another equilateral triangle in the same plane, known as the ``\emph{Napoleon triangle}" \cite{coxeter1996geometry} --- see \reffig{fig.Napoleon}. 
%in a plane,  centroids of, either all outer or all inner, equilateral triangles constructed at the sides of a triangle form an equilateral triangle, known as the ``\emph{Napoleon triangle}" \cite{coxeter1996geometry}. 
We will refer to this construction as the Napoleon transformation, and we find it convenient to define the \emph{double outer Napoleon triangle} as the equilateral triangle resulting from two concatenated outer Napoleon transformations of a triangle. 
Let $\NT : \R^{3d} \rightarrow \R^{3d }$ denote the double outer Napolean transformation, see \cite{ArslanEtAl_Techreport2013} for an explicit form of $\NT$.     

\begin{figure}[h]
\centering
\includegraphics[width=0.23\textwidth]{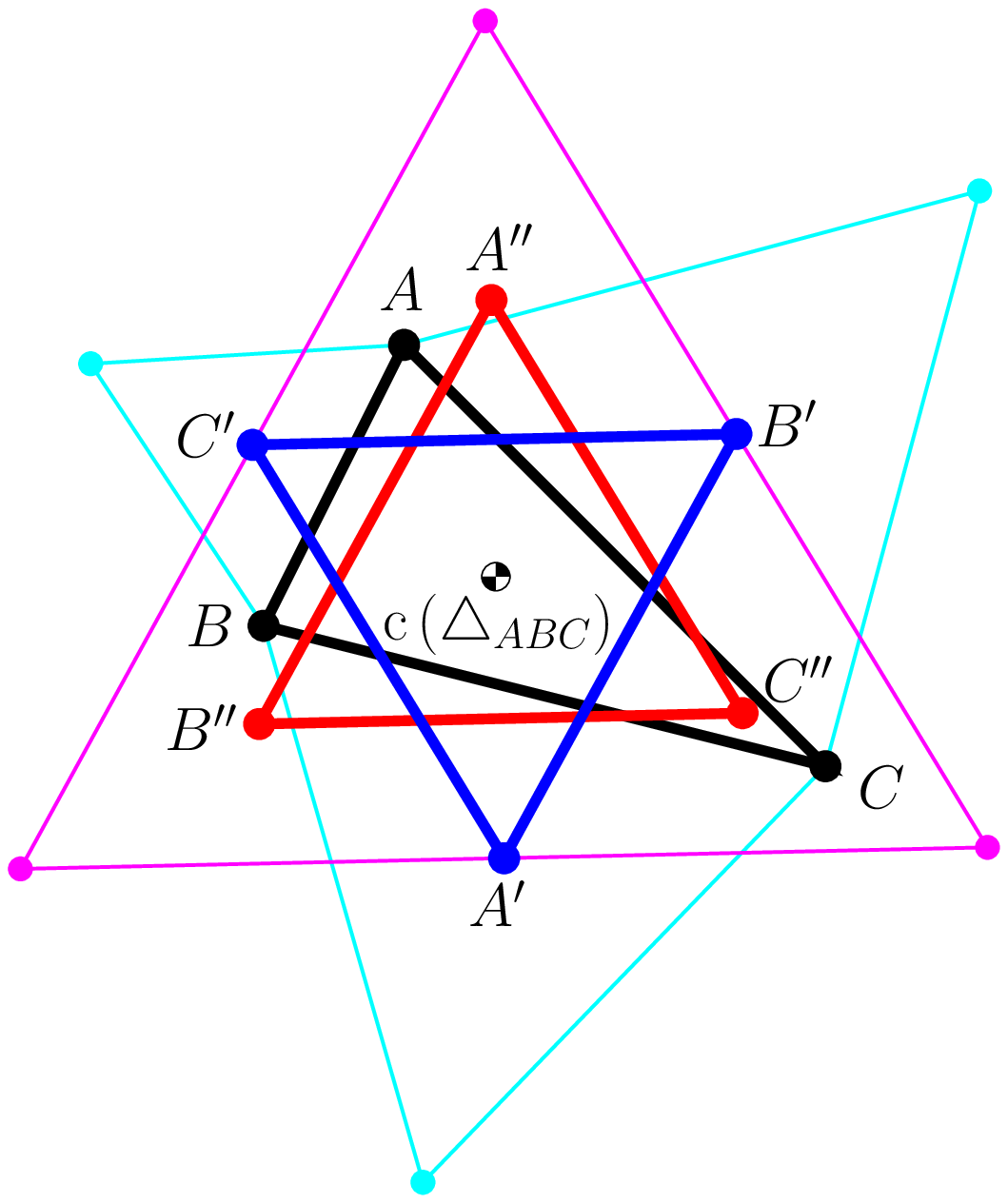} 
\vspace{-1mm}
\caption{Outer Napoleon Triangles  $\bigtriangleup_{A'B'C'}$ and  $\bigtriangleup_{A''B''C''}$ of $\bigtriangleup_{ABC}$ and  $\bigtriangleup_{A'B'C'}$, respectively, and  $\bigtriangleup_{A''B''C''}$  is referred to as the double outer triangle of $\bigtriangleup_{ABC}$. Note that centroids of all triangles coincides, i.e. $\ctrd{\bigtriangleup_{ABC}} = \ctrd{\bigtriangleup_{A'B'C'}} = \ctrd{\bigtriangleup_{A''B''C''}}$.
}
\label{fig.Napoleon}
\end{figure}

The NNI-triplet $\crl{A, B, C}$ defines  an associated triangle with distinct vertices for each configuration,  $ \bigtriangleup_{A, B, C} : \stratum{\treeA} \rightarrow \confspace{\R^d,\brl{3}, \vectbf{0}}$, 
\begin{align}
\bigtriangleup_{A, B, C}\prl{\vectbf{\state}} \ldf \tr{\threevecT{ \ctrd{\vectbf{\state} | A}}{\ctrd{\vectbf{\state}|B}}{\ctrd{\vectbf{\state}|C}}}.
\end{align}
The double outer Napolean tranformation of $\bigtriangleup_{A, B, C}\prl{\vectbf{\state}}$ returns  symmetric  target locations for  $\ctrd{\vectbf{\state}|A}$,  $\ctrd{\vectbf{\state}|B}$ and $\ctrd{\vectbf{\state}|C}$, and the corresponding displacement of $\ctrd{\vectbf{x}|P}$, denoted 
$\NToff{A,B,C} : \cRdJr \rightarrow \R^d$, is given by the formula\footnote{Here, $\mat{I}_{d}$ is the $d \times d$ identity matrix, and $\mat{1}_k$ is the $\R^k$ column vector of all ones. Also, $\otimes$ and $\cdot$ denote the Kronecker product and the standard array product, respectively.}
\begin{align} % This is your eqn (20)
\NToff{ A,B,C }\prl{\vectbf{\state}} \ldf \ctrd{\vectbf{\state} | P}
-    \Gamma \cdot \NT \circ
\bigtriangleup_{A,B,C}\prl{\vectbf{\state}},  \label{eq.napoleonoffset}
\end{align} 
where $\Gamma := \frac{1}{ \card{P} } \threevecT{\card{A}}{\card{B}}{\card{C}}  \otimes \mat{I}_d \in \R^{d \times 3d}$, %{  $\otimes$ and $\cdot$ denote the Kronecker product and the standard array product, respectively.} 
and the vertices of the associated equilateral triangle with compensated offset of $\ctrd{\vectbf{x}|P}$  are\addtocounter{footnote}{-1}\footnotemark
\begin{align} \label{eq.SymCentroids}
\tr{\threevecT{c_{A}}{c_{B}}{c_C} \hspace{-0.5mm}} \!\ldf  \NT \sqz{\circ} \bigtriangleup_{A,B,C} \prl{\vectbf{\state}} \sqz{+} \mat{1}_3 \sqz{\otimes} \NToff{A,B,C} \prl{\vectbf{\state}}\!. \!\!
\end{align}
%
%where  $\mat{1}_k$ is the $\R^k$ row vector of all ones.

\smallskip

%%%%%%%%%%%%%%%%%%%%%%%%%%%%%%%%%%%%%%%%%%%%%%%%
%%%%%%%%%%%%%%%%%%%%%%%%%%%%%%%%%%%%%%%%%%%%%%%%
\subsubsection{Portal Maps} 
\label{sec.Portal}
%%%%%%%%%%%%%%%%%%%%%%%%%%%%%%%%%%%%%%%%%%%%%%%%
%%%%%%%%%%%%%%%%%%%%%%%%%%%%%%%%%%%%%%%%%%%%%%%%

We now define a portal map, $\portalconf{}: \stratum{\treeA} \rightarrow \portal\prl{\treeA,\treeB}$, to be  
\begin{align}\label{eq.portalconf}
\hspace{-1mm}\portalconf{}\prl{\vectbf{\state}} \sqz{\sqz{\ldf}} 
\left \{
\begin{array}{@{}l@{}@{}l@{}}
\vectbf{\state} & \text{, if } \vectbf{\state} \sqz{\in} \portal\prl{\treeA,\treeB}\!,\\
 \prl{\portalmerge{} \sqz{\circ} \portalscale{} \sqz{\circ} \portalcenter{}} \prl{\vectbf{\state}}  & \text{, otherwise,}
\end{array}
\right. \!\! \!
\end{align} 
where  $\portalcenter{} : \stratum{\treeA} \rightarrow \mathtt{Sym}\prl{\treeA, \treeB}$
rigidly translates  the partial configurations, $\vectbf{x}|A$, $\vectbf{x}|B$ and $\vectbf{x}|C$, to the new centroid locations, $c_A$, $c_B$ and $c_C$   \refeqn{eq.SymCentroids}, respectively, yielding a symmetric configuration, 
\begin{align}
\hspace{-1mm}\portalcenter{}\prl{\vectbf{x}} \sqz{\ldf} \left \{
\begin{array}{@{}l@{}@{}l@{}}
\vect{\state}_i & \text{, if } i \sqz{\not \in}  P,\\
\vect{\state}_i \sqz{-} \ctrd{\vectbf{\state}|Q} \sqz{+} \vect{c}_{Q}  & \text{, if } i \hspace{-0.6mm}\in \hspace{-0.6mm} Q, Q \hspace{-0.6mm} \in \hspace{-0.6mm} \crl{A,B,C}\!,
\end{array}
\right. \!\!\!\!\! \label{eq.portalcenter}
\end{align}
It is important to observe that $\portalcenter{}$ keeps the barycenter of $\vectbf{\state}|P$ fixed, and so separating hyperplanes of the rest of clusters ascending and disjoint with  $P$  are kept unchanged.

\begin{figure*}[t] 
\centering
\begin{tabular}{@{}c@{}@{}c@{}@{}c@{}}

\begin{tabular}{@{}c@{}}
\includegraphics[width=0.3\textwidth]{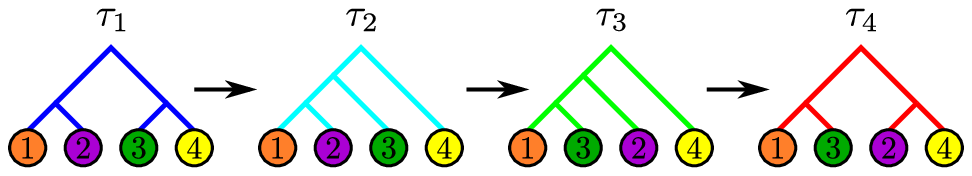} \\[-1mm]
\scalebox{0.7}{(a)}  \\
\includegraphics[width=0.34\textwidth]{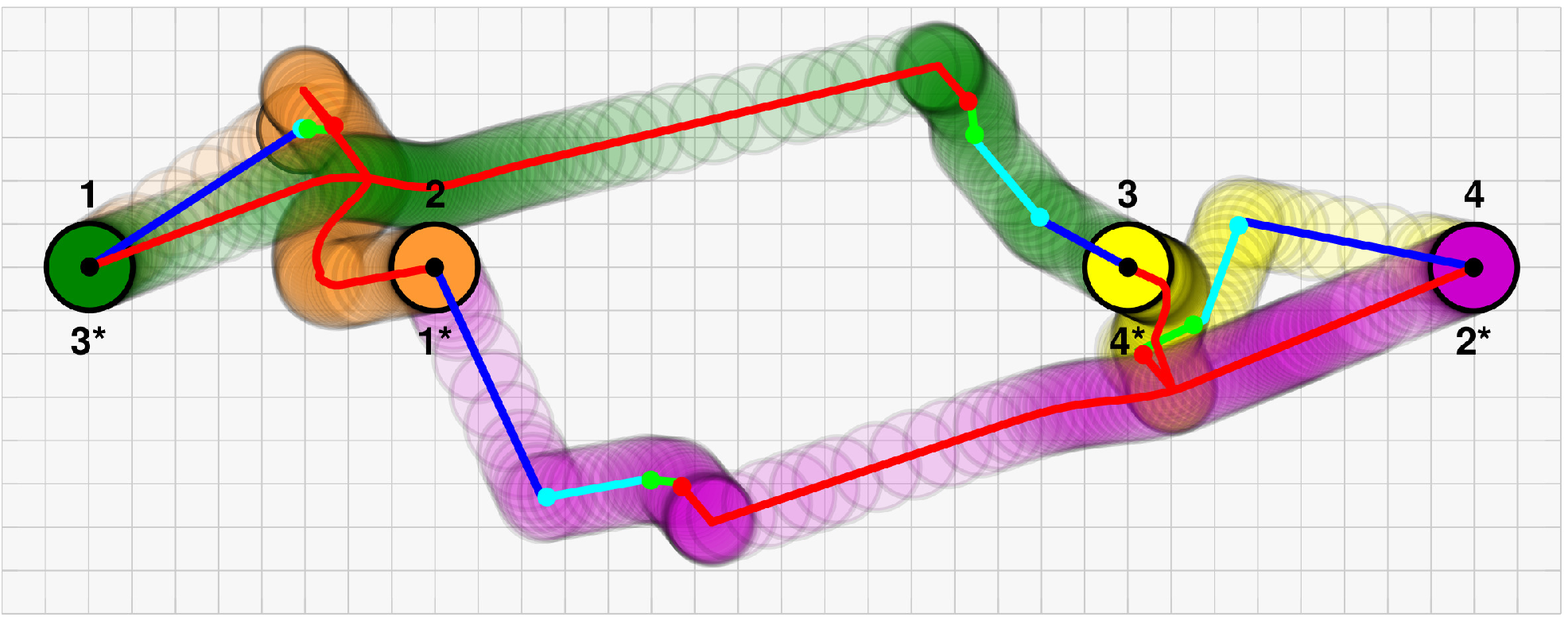} \\[-1mm]
\scalebox{0.7}{(b)}
\end{tabular}
&
\begin{tabular}{@{\hspace{1mm}}c@{}}
\includegraphics[width=0.25\textwidth]{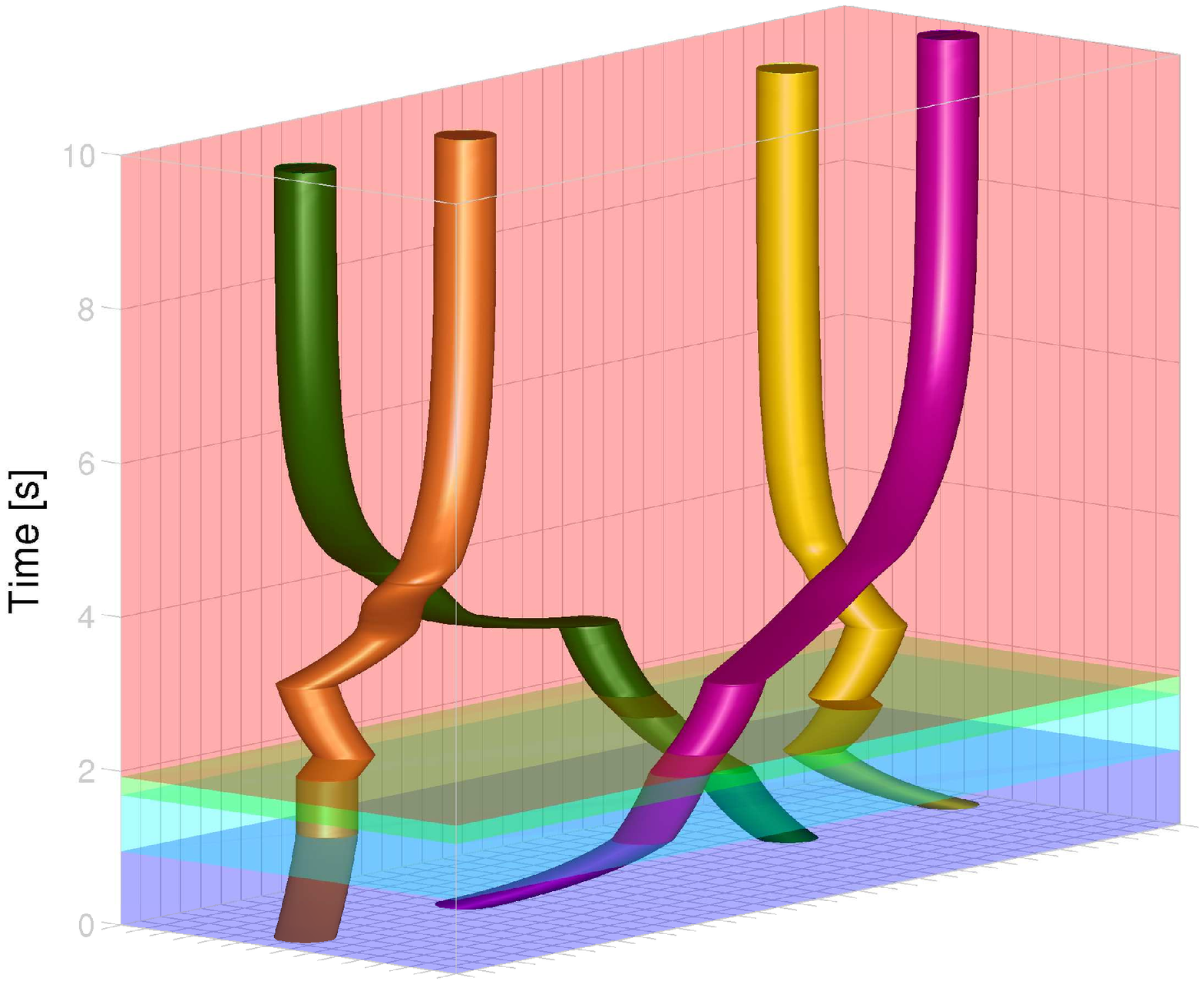}\\[-1mm]
\scalebox{0.7}{(c)} 
\end{tabular}
&
\begin{tabular}{@{}c@{}}
\includegraphics[width=0.4\textwidth]{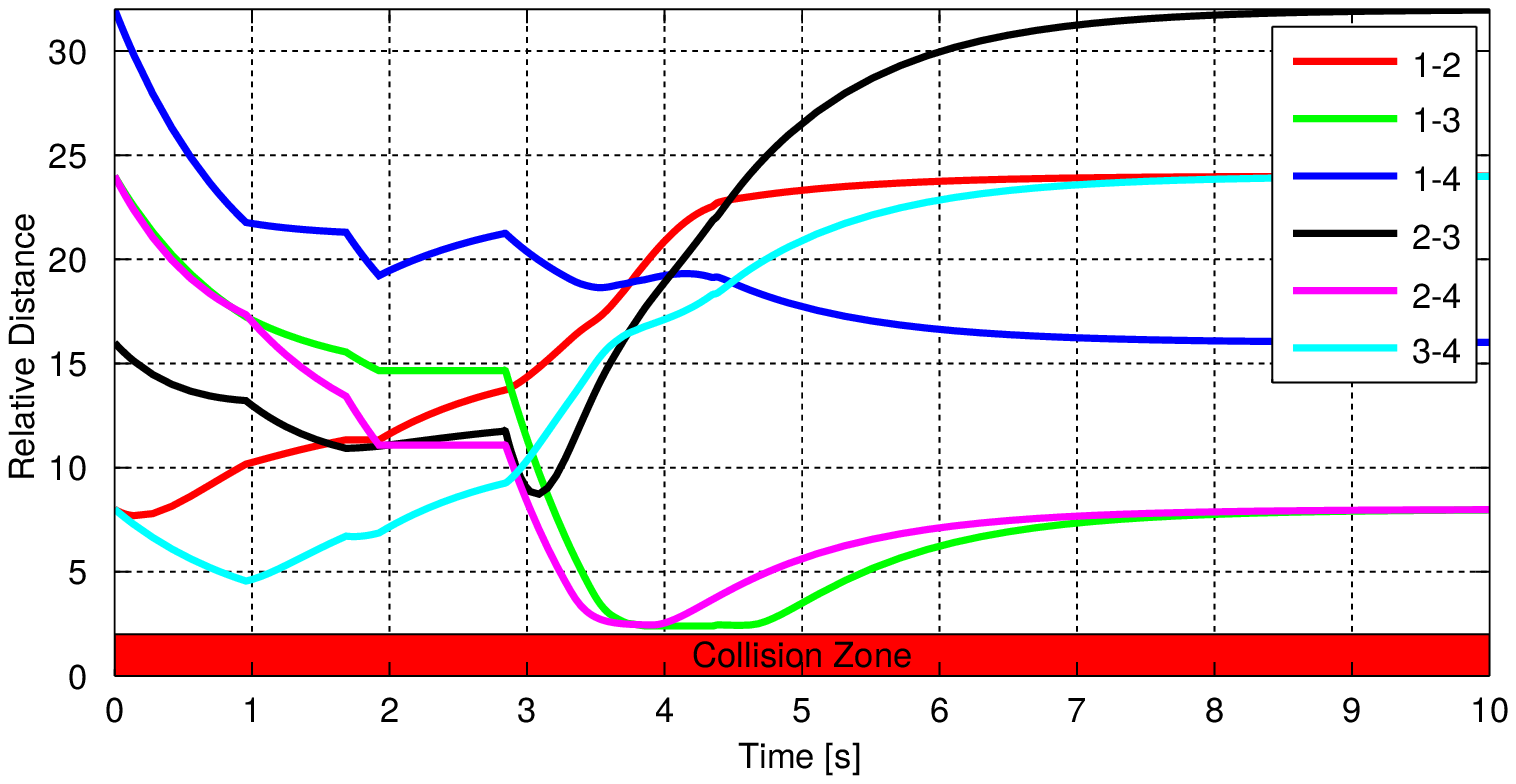} \\
\scalebox{0.7}{(d)}
\end{tabular}
\end{tabular}
\vspace{-3mm}
\caption{An illustrative navigation trajectory of the hybrid dynamics generated by the HNC algorithm for 4 disks in a planar ambient space. 
Disks are placed on the horizontal axis for both the initial and desired configurations in different orders, from left to right $\prl{1, 2 , 3 , 4}$ and  $\prl{3^*,1^*,4^*,2^*}$ at the start and goal, respectively. 
(a) The sequence of trees  associated with deployed local controllers during the execution of the hybrid navigation controller. 
(b) Centroidal trajectory of each disk colored according the active local controller, where $\vectbf{x}_c \in \stratum{\tree_1}\cap \stratum{\tree_2}$, $\vectbf{x}_g \in \stratum{\tree_2}  \cap \stratum{\tree_3} $ and $\vectbf{x}_r \in \stratum{\tree_3} \cap \stratum{\tree_4}$ shown by cyan, green and red dots, respectively, are portal configurations. (c) Space-time curve of disks (d) Pairwise  distances between disks. }
\label{fig.SampleTrajectory1}
\end{figure*}

%\smallskip

After obtaining a symmetric configuration in $\mathtt{Sym}\prl{\treeA, \treeB}$, $\portalscale{}:\mathtt{Sym}\prl{\treeA,\treeB} \rightarrow \mathtt{Sym}\prl{\treeA,\treeB}$ rigidly translates each partial configuration, $\vectbf{\state}|A$, $\vectbf{\state}|B$ and $\vectbf{\state}|C$, to scale and fit into the corresponding consensus ball so that the new configuration simultaneously support both  subtrees of $\treeA$ and $\treeB$ rooted at $P$,
\begin{align}\label{eq.portalscale}
\hspace{-1mm}\portalscale{}\prl{\vectbf{\state}} &\sqz{\ldf} \!
\left \{
\begin{array}{@{}l@{}l}
\vect{x}_i, & \text{, if } i \not \in P 
\\
\vect{x}_i \sqz{+} \scaleconst \! \cdot \prl{\big.\ctrd{\vectbf{x}|Q} \sqz{-} \ctrd{\vectbf{x}|P}\!} & \text{, if } 
\begin{array}{@{}c@{}}
 i \sqz{\in} Q, \\ 
 \! Q \sqz{\in} \crl{A,B,C}\!,
 \end{array}
\end{array}
\right . \!\!\!\!\!\!\!\!
\end{align}
where $\scaleconst \in [0,\infty)$ is a scale parameter defined as
{\small
\begin{align}
\scaleconst \ldf \max_{Q \in \crl{A,B,C}} \max \prl{\frac{\radiusCL{}\prl{\vectbf{x}|Q} + \alpha}{\radiusDD{Q}\prl{\vectbf{\state}}}, 1} - 1.
\end{align}
}%
Here, $\alpha > 0$ is a safety margin as used in \refeqn{eq.Alpha}, and  $\radiusCL{}\prl{\vectbf{\state}|Q}$ \refeqn{eq.ConfRadius} denotes the centroidal radius of  partial configuration $\vectbf{\state}|Q$ and $\radiusDD{Q}\prl{\vectbf{\state}}$ \refeqn{eq.RadiusStandardConf} is the radius of its consensus ball. 
Note that $\portalscale{}$ preserves the symmetry of the configurations, i.e. centroids  $\ctrd{\vectbf{x}|A}$, $\ctrd{\vectbf{x}|B}$ and $\ctrd{\vectbf{x}|C}$ still form an equilateral triangle after the mapping, and lefts the barycenter of $\vectbf{\state}|P$ unchanged.

%\smallskip

Finally, $\portalmerge{}:\mathtt{Sym}\prl{\treeA,\treeB} \rightarrow \mathtt{Sym}\prl{\treeA,\treeB}$ iteratively translates and merges partial configurations of common complementary clusters of $\treeA$ and $\treeB$, in a bottom-up fashion starting at $P$,  to simultaneously support both hierarchies $\treeA$ and $\treeB$, 
\begin{align}\label{eq.portalmerge}
\portalmerge{}\prl{\vectbf{x}} \ldf \portalmerge{P}\prl{\vectbf{x}},
\end{align}
where for any $I \in \crl{P} \cup \ancestorCL{P,\treeA}$
\begin{align}\label{eq.portalmergerecursion}
\portalmerge{I}\prl{\vectbf{x}} \sqz{\ldf}  \left \{
\begin{array}{@{}l@{}l}
\vectbf{x} & \text{, if }  I = \indexset, \\
(\portalmerge{\parentCL{I,\tree}} \circ \portalseparate{I})\prl{\vectbf{x}} & \text{, otherwise.}
\end{array} 
\right.
\end{align}
Here,  $\portalseparate{I}$ separates complementary clusters $I$ and $ \complementLCL{I}{\treeA}$ such that the clearance between every agent in $I \cup \complementLCL{I}{\treeA}$ and  the associated separating hyperplane is at least $\alpha$ units (i.e. if  $\hat{\vectbf{x}} = \portalseparate{I}\prl{\vectbf{x}}$ for some $\vectbf{x} \in \prl{\R^d}^{\indexset}$ with $\ctrdsep{I,\treeA}{\vectbf{x}} \neq 0$, then $\sepmag{k,K,\treeA}{\hat{\vectbf{x}}} \geq r_k + \alpha$ for any $k \in K$, $K \in \crl{I, \complementLCL{I}{\treeA}}$): for any $j \in \indexset$  
{
\begin{align}\label{eq.portalseparation}
\hspace{-2mm}\portalseparate{I}\prl{\vectbf{x}}_{\!j} \sqz{\ldf} \! \left \{
\begin{array}{@{}l@{}l}
\vectbf{x}_j & \text{, if } j \sqz{\not \in} \parentCL{I,\treeA}\!,\\
\vectbf{x}_j \sqz{+} 2 \lambda \frac{|\complementLCL{K}{\treeA}|}{\card{\parentCL{K,\treeA}}}\!\frac{\ctrdsep{K,\treeA}{\vectbf{x}}}{\norm{\ctrdsep{K,\treeA}{\vectbf{x}}}}  & \text{, if } 
\begin{array}{@{}c@{}}
j \in K, \\ 
\! K \sqz{\in} \!\crl{\!I, \complementLCL{I}{\treeA}\!}\!,
\end{array}
\end{array}
\right. \!\!\! \!\! \!\! \!
\end{align}
}%
where  the required amount of centroidal separation, $\lambda \in [0, \infty)$, is given by
{\small
\begin{align}
 \lambda \ldf \!\!  \max_{\substack{k \in K \\ K \in \crl{I, \complementLCL{I}{\treeA}}}} \!\! \max \prl{\big.\! -\prl{\sepmag{k,K,\treeA}{\vectbf{x}} \sqz{-} r_k \sqz{-} \alpha}, 0}.
\end{align}
}%
Note that since $\ctrd{\vectbf{x}|P} = \ctrd{\hat{\vectbf{x}}|P}$ for any $\vectbf{x} \in \stratum{\treeA}$ and $\hat{\vectbf{x}} = \prl{\portalscale{} \circ \portalcenter{}}\prl{\vectbf{x}}$, we always have  $\ctrdsep{I,\treeA}{\hat{\vectbf{x}}} \neq 0$ for any $I \in \crl{P} \cup \ancestorCL{P, \treeA}$, which is required for $\portalseparate{I}$ to be well defined. 
Further, using \refeqn{eq.portalseparation}, one can verify that $\ctrd{\vectbf{x}|\parentCL{I,\treeA}} = \ctrd{\hat{\vectbf{x}}|\parentCL{I,\treeA}}  = \ctrd{\tilde{\vectbf{x}}|\parentCL{I,\treeA}} $ for $\tilde{\vectbf{x}} = \portalseparate{I}\prl{\hat{\vectbf{x}}}$, and so $\ctrdsep{A, \treeA}{\tilde{\vectbf{x}}} \neq 0$ for any $A \in \ancestorCL{I,\treeA}$, which guarantees that recursive calls of $\portalseparate{I}$ in the computation of $\portalconf{}$  are always well-defined.

We find  it useful to summarize some critical properties of the portal map for the strata of $\HCkmeans$ as:

\begin{theorem} \label{thm.Portal}
The NNI-graph $\NNIgraph_{\indexset}=\prl{\bintreetopspace_{\indexset}, \NNIedgeset}$ is a subgraph of the \emph{$\HCkmeans$} adjacency graph $\adjgraph{\indexset} = \prl{\bintreetopspace_{\indexset}, \adjedgeset}$, i.e.
for any pair $\prl{\treeA, \treeB}$ of  NNI-adjacent trees in $\bintreetopspace_{\indexset}$, $\portal\prl{\treeA,\treeB} \neq \emptyset$.
Further, given an edge, $\prl{\treeA, \treeB} \in \NNIedgeset \subset \adjedgeset$, a geometric realization via the map $\portalconf{\prl{\treeA,\treeB}}: \stratum{\treeA} \rightarrow \portal\prl{\treeA, \treeB}$  
\refeqn{eq.portalconf} can be computed in quadratic, $\bigO{\card{\indexset}^2}$, time  with the number of leaves, $\card{\indexset}$. 
\end{theorem}
\begin{proof}
See \refapp{app.Portal}. \qed 
\end{proof}

%%%%%%%%%%%%%%%%%%%%%%%%%%%%%%%%%%%%%%%%%%%%%
%%%%%%%%%%%%%%%%%%%%%%%%%%%%%%%%%%%%%%%%%%%%%
\section{Numerical Simulations}
\label{sec.Simulation}
%%%%%%%%%%%%%%%%%%%%%%%%%%%%%%%%%%%%%%%%%%%%%
%%%%%%%%%%%%%%%%%%%%%%%%%%%%%%%%%%%%%%%%%%%%%

For the sake of clarity, we first illustrate the behavior of the  hybrid system defined in \refsec{sec.HierNavSphere} for the case  of  four disks moving in a two-dimensional ambient space.\footnote{For all simulations we consider unit disks moving in an ambient plane, i.e. $r_j = 1$ for all $j \in \indexset$, and we set  $\alpha = 0.2$ and $\beta = 1$; and all simulations are obtained through numerical integration of the hybrid dynamics generated by the HNC algorithm (\reftab{tab.HNCAlgorithm})  using the \texttt{ode45} function of MATLAB.}

In order to visualize in this simple setting the most complicated instance  of  collision-free navigation and observe maximal number of transitions between local controllers, we pick the initial, $\vectbf{x} \in \stratum{\tree_1}$, and desired configurations, $\vectbf{x}^* \in \stratum{\tree_4}$, where disks are placed on the horizontal axis and left-to-right ordering of their labels are $\prl{1,2,3,4}$ and $\prl{3^*,1^*,4^*,2^*}$, respectively, and their corresponding clustering trees are $\tree_1 \in \bintreetopspace_{\brl{4}}$ and $\tree_4 \in \bintreetopspace_{\brl{4}}$,  see \reffig{fig.SampleTrajectory1}.

The resultant trajectory of each disk following the hybrid navigation planner in \refsec{sec.HierNavSphere}, the relative distance between each pair of disks and the sequence of trees associated with visited hierarchical strata are shown in \reffig{fig.SampleTrajectory1}. 
Here, the disks start following the local controller associated with $\tree_1$
until they enter in finite time the domain of the following local controller associated with $\tree_2$ at $\vectbf{x}_c \in \stratum{\tree_1} \cap \stratum{\tree_2}$ --- shown by cyan dots in \reffig{fig.SampleTrajectory1}.
After a finite time navigating in $\stratum{\tree_2}$ and $\stratum{\tree_3}$, respectively, the group enters the domain of the goal controller $f_{\tree_4, \vectbf{x}^*}$ (\reftab{tab.HierInvNav}) at $\vectbf{x}_r \in \stratum{\tree_3} \cap \stratum{\tree_4}$ --- shown by red dots in \reffig{fig.SampleTrajectory1}, and $f_{\tree_4, \vectbf{x}^*}$  asymptotically steers the disks to the desired configuration $\vectbf{x}^* \in \stratum{\tree_4}$.       
Finally, note that the total number of binary trees over four leaves is 15; however, our hybrid navigation planner reactively deploys only 4 of them.

We now consider a similar, but slightly more complicated  setting: a group of six disks  in a plane where agents are initially  placed evenly on the horizontal axes and switch their positions at the destination as shown in \reffig{fig:SampleNNITrajectory6-8-16P}(a), which is also  used in \cite{Tanner_Boddu_TRO2012} as an example of  complicated multi-agent arrangements.
%Here, the order of particle labels from left to right are, respectively, $1,2,3,4,5,6$ and $4^*,5^*,1^*,6^*,2^*,4^*$ for the initial and goal configurations.
While steering the disks  towards the goal, the hybrid navigation planner automatically deploys only 6 local controllers out of the family of 945 local controllers.
The time evolution of the disk is illustrated in \reffig{fig:SampleNNITrajectory6-8-16P}(a).

\begin{figure*}[t]
\centering
\begin{tabular}{@{}c@{\hspace{8mm}}cc@{}}
\raisebox{12mm}{\includegraphics[width=0.33\textwidth]{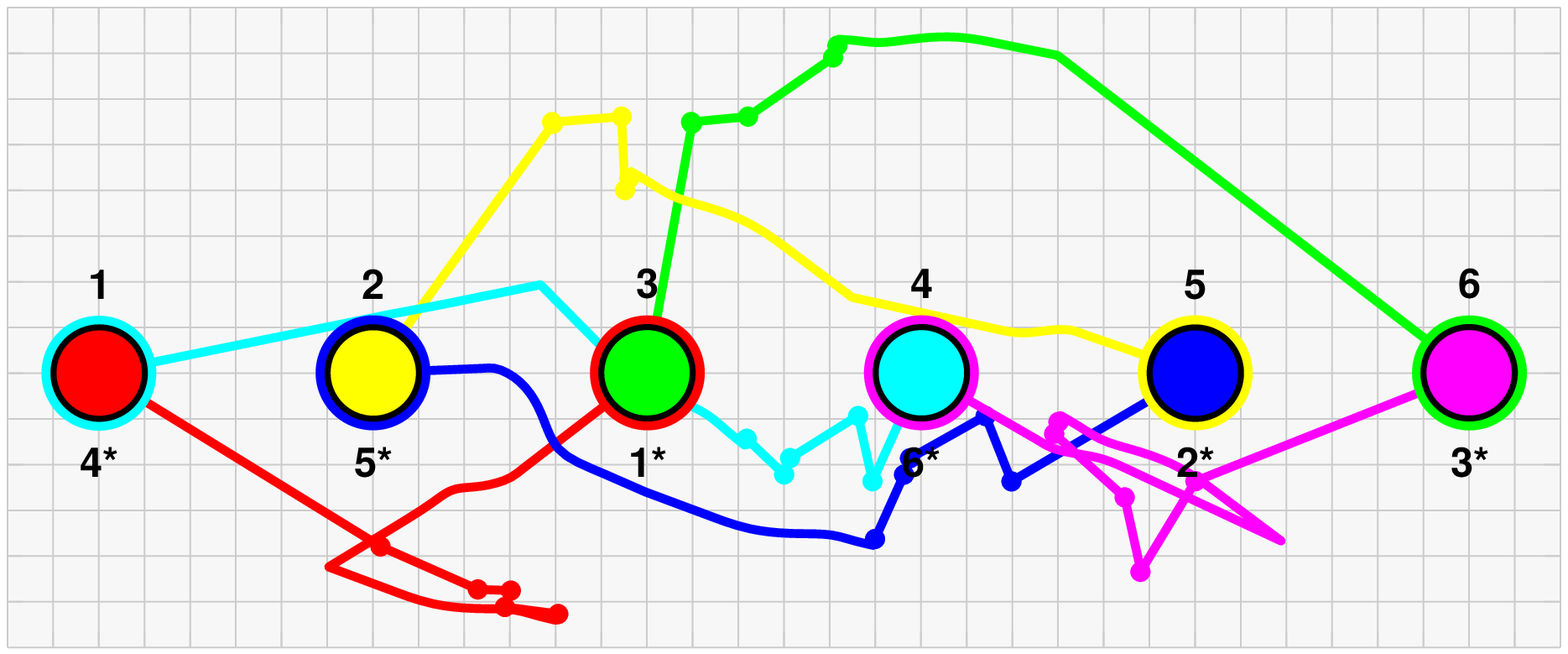}}
&
\raisebox{3mm}{\includegraphics[width=0.215 \textwidth]{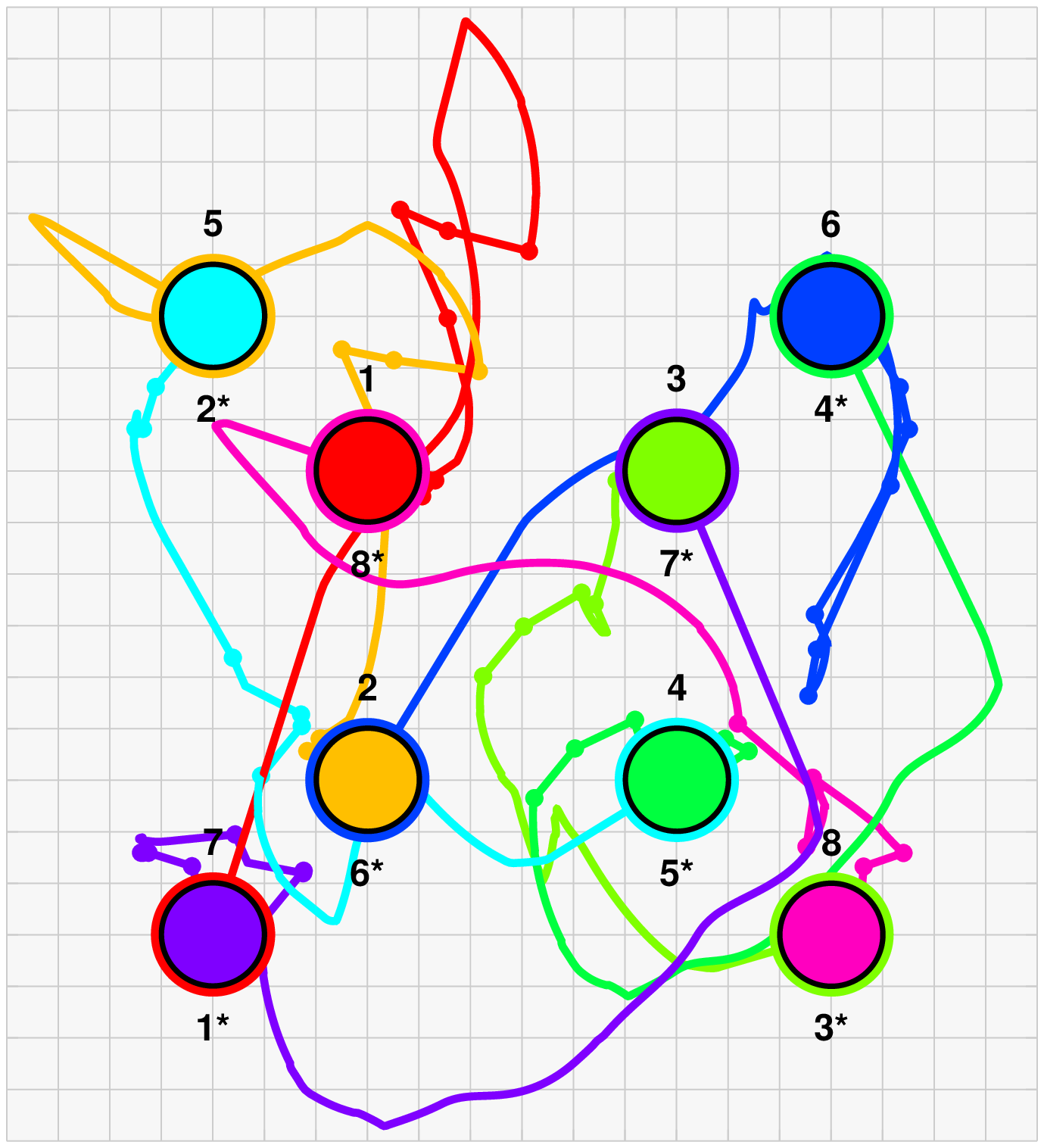}}
&
\includegraphics[width=0.23 \textwidth]{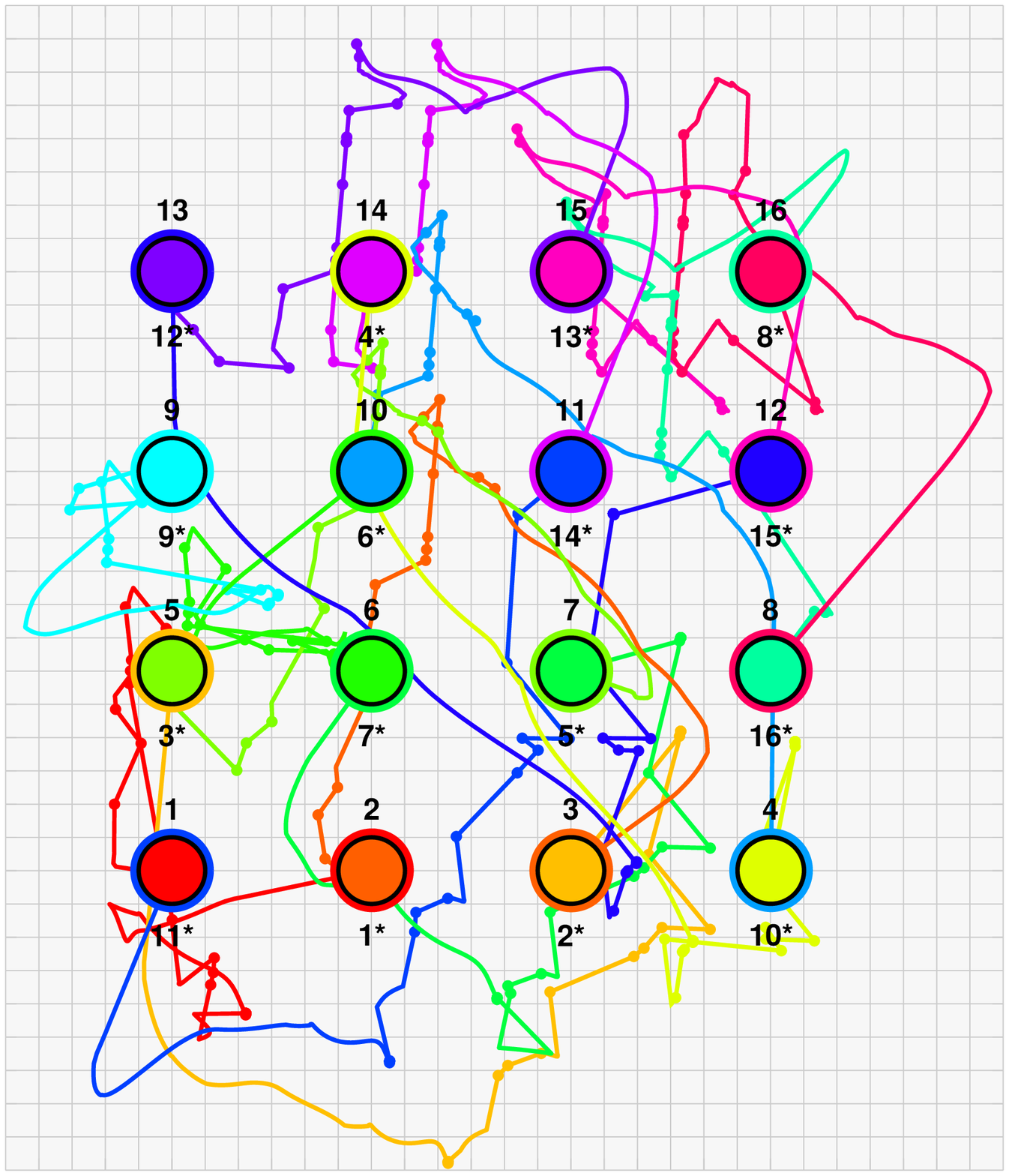}
\\[-2mm]
\raisebox{5mm}{\includegraphics[width=0.3\textwidth]{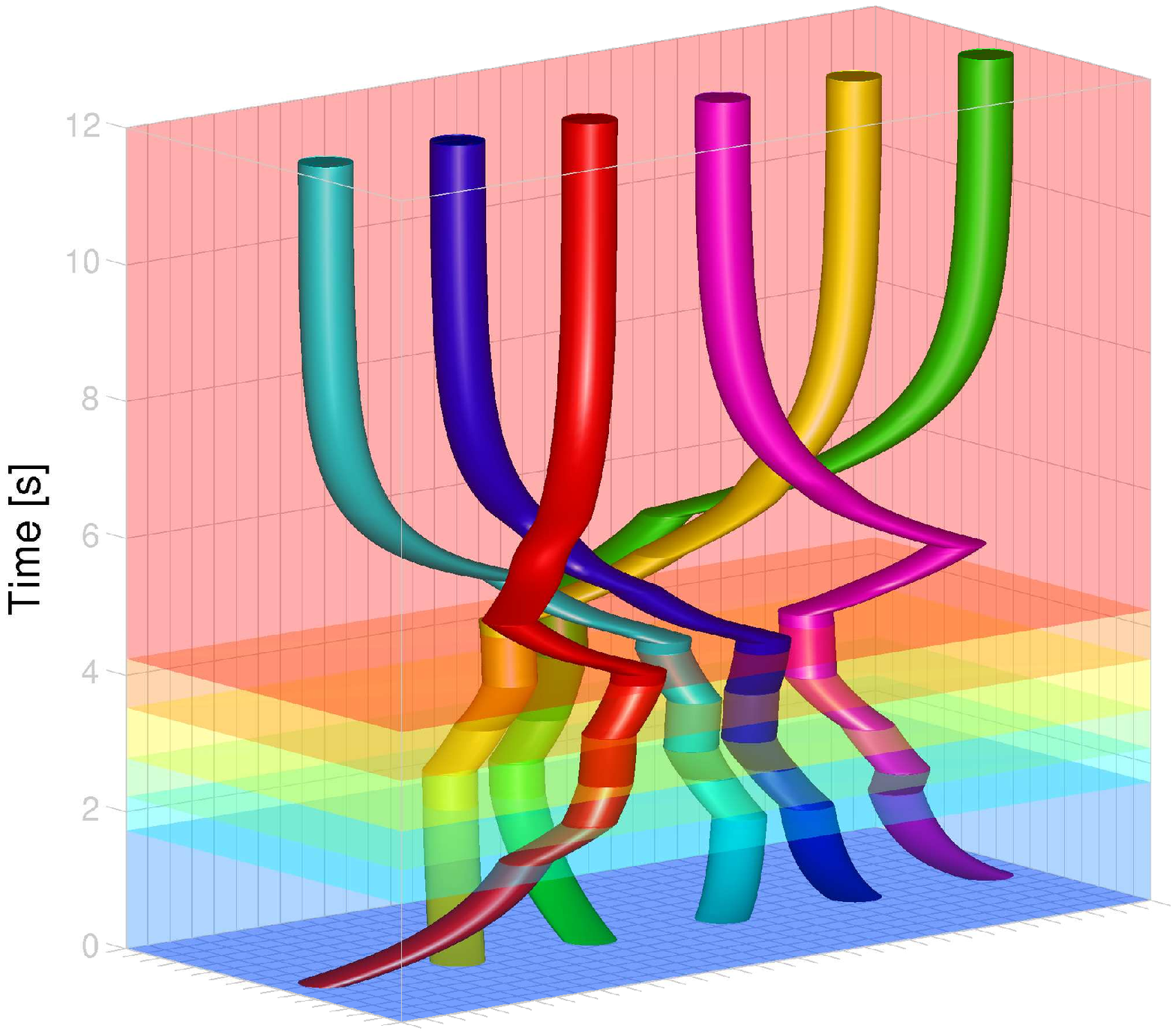}} 
&
\includegraphics[width=0.24 \textwidth]{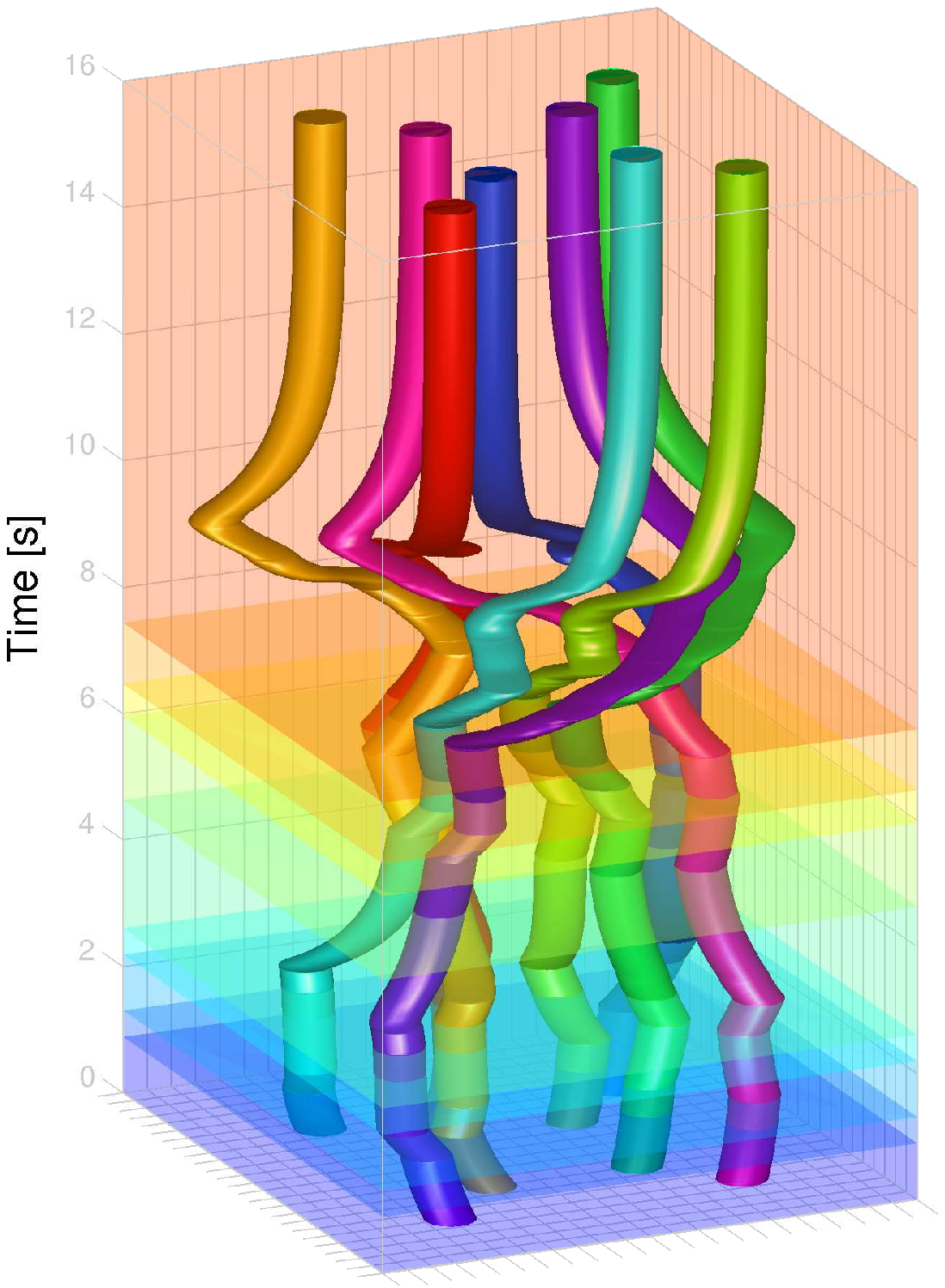} 
&
\raisebox{0mm}{\includegraphics[width=0.30\textwidth]{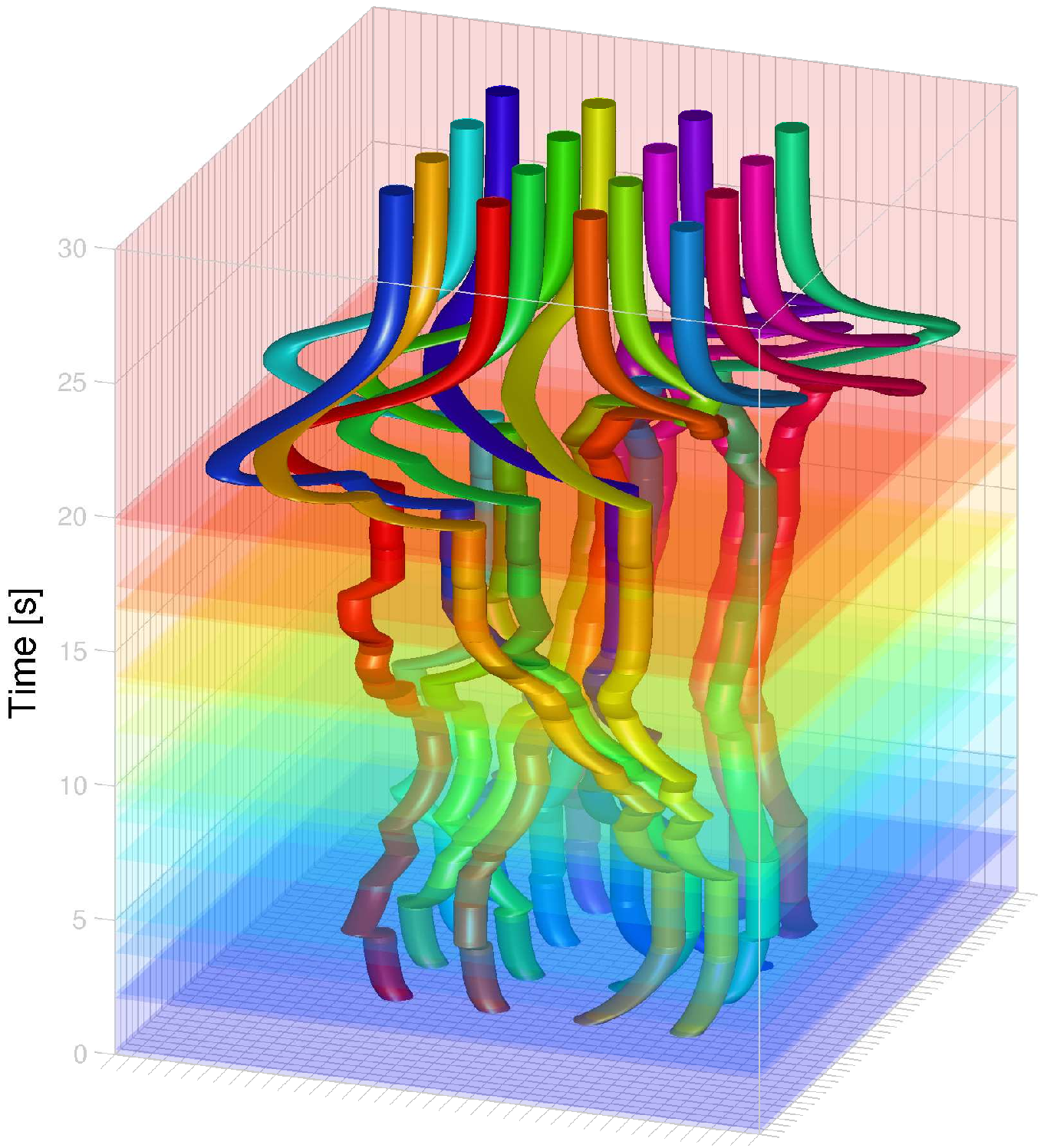}} \\[-1.5mm]
\scalebox{0.7}{(a)} & \scalebox{0.7}{(b)} & \scalebox{0.7}{(c)}
\end{tabular}
\vspace{-1mm}
\caption{Example trajectories of the hybrid vector field planner for  (a) 6,  (b) 8 and (c) 16 disks in a planar ambient space.  
(top) Trajectory  and (bottom) state-time curve of each disk.
Each colored time interval demonstrates the  execution duration of  an activated local controller.
Dots correspond to the portal configurations where transitions between local controllers occur at.
}
\label{fig:SampleNNITrajectory6-8-16P}
\vspace{-1mm}
\end{figure*}

Finally, to demonstrate the efficiency of the deployment policy of our hybrid planner, we separately consider groups of 8 and 16 disks in an ambient plane, illustrated in  \reffig{fig:SampleNNITrajectory6-8-16P}. 
The eight disks are initially located at the corner of two squares whose centroids coincide  and the perimeter of one is twice of the perimeter of the other. 
At the destination, disks switch their locations as illustrated in \reffig{fig:SampleNNITrajectory6-8-16P}(b).
For sixteen disk case, disks are initially placed at the vertices of a 4 by 4 grid, and  their task is to switch their location as illustrated in  \reffig{fig:SampleNNITrajectory6-8-16P}(c).
Although there are a large number of local controllers for the case of groups of 8 and 16 disks ($\card{\bintreetopspace_{\brl{8}}} > 10^5$ and $\card{\bintreetopspace_{\brl{16}}} > 6 \times 10^{15}$), our hybrid navigation planner only deploys 9 and 19 local controllers, respectively.

The number of potentially available local controllers for a group of $n$ disks \refeqn{eq.NumTree} grows super exponentially with $n$. 
On the other hand, if agents have perfect sensing and actuation modelled as in the present paper, the hybrid navigation planner automatically deploys at most
$\frac{1}{2}\prl{n-1}\prl{n-2}$ local controllers \cite{ArslanEtAl_NNITechReport2013}, illustrating the computational efficiency of our construction.

\section{Conclusion}
\label{sec.Conclusion}
%%%%%%%%%%%%%%%%%%%%%%%%%%%%%%%%%%%%%%%%%%%%%
%%%%%%%%%%%%%%%%%%%%%%%%%%%%%%%%%%%%%%%%%%%%%

In this paper, we introduce a novel application of clustering to the problem of coordinated robot navigation.
The notion of hierarchical clustering offers a natural abstraction for ensemble task encoding and control in terms of precise yet flexible organizational specifications at different resolutions.
Based on this new abstraction, we propose a provably correct generic hierarchical navigation framework for collision-free motion design towards any given destination via a sequence of hierarchy preserving controllers.
For the 2-means divisive hierarchical clustering \cite{savaresi_boley_ICDM2001}, based on
a topological characterization of the underlying space, we present a centralized online (completely reactive) and computationally efficient instance of our hierarchical navigation framework for disk-shaped robots, which generalizes to an arbitrary number of disks and ambient space dimension. 

Specifically, matching the component problem statements of \refsec{sec.HierNav} to their subsequent resolution:  
we address \refprob{prob.HierarchyInvariant} in  \refthm{thm.HierInvNav} (guaranteeing that the construction of \reftab{tab.HierInvNav} results in a hierarchy invariant vector field planner); 
we address \refprob{prob.Transition}  in \refthm{thm.Transition} (guaranteeing that the construction of \reftab{tab.NNIControl} results in a reactive strategy that finds, given any non-goal tree, an edge in the graph of all hierarchies  leading to a new tree that is closer to the desired goal hierarchy); 
and we address  \refprob{prob.Portal} in \refthm{thm.Portal} (providing a geometric realization in the configuration space of the combinatorial edge toward the physical goal). 
The efficacy  of this overarching strategy is guaranteed by \refthm{thm.HierNavAlg} (proving the correctness of these problems  steps and their resolutions as presented in \reftab{tab.HNCAlgorithm}).

Work now in progress targets  more practical settings in the field of robotics including navigating around obstacles in  compact spaces and a distributed implementation of our navigation framework.
%These methods suggest a promising, broad domain of new hierarchical formation (motion) planning and hierarchical perceptual servoing problems.
We are also exploring a number of application settings for hierarchical formation specification and control including problems of perception, perceptual servoing, anomaly detection and automated exploration and various problems of multi-agent coordination. 

In the longer term, especially when the scalability and efficiency of hierarchical protocols in sensor networks for information routing and aggregation is of concern \cite{akkaya_younis_AHN2005},  these methods suggest a promising unifying framework to simultaneously  handle control, communication and information aggregation (fusion) in multi-agent systems.

% if have a single appendix:
%\appendix[Proof of the Zonklar Equations]
% or
%\appendix  % for no appendix heading
% do not use \section anymore after \appendix, only \section*
% is possibly needed

% use appendices with more than one appendix
% then use \section to start each appendix
% you must declare a \section before using any
% \subsection or using \label (\appendices by itself
% starts a section numbered zero.)
%

\appendices

%%%%%%%%%%%%%%%%%%%%%%%%%%%%%%%%%%%%%%%%
%%%%%%%%%%%%%%%%%%%%%%%%%%%%%%%%%%%%%%%%
\section{Properties of The Hierarchy Invariant Vector Field}
\label{app.HierInvNav}
%%%%%%%%%%%%%%%%%%%%%%%%%%%%%%%%%%%%%%%%
%%%%%%%%%%%%%%%%%%%%%%%%%%%%%%%%%%%%%%%%

Although the recursive definition of the hierarchy preserving navigation policy, $\f$, in \reftab{tab.HierInvNav} expresses an efficient encoding of intra-cluster and inter-cluster interactions and dependencies of individuals, which we suspect will prove to have value for  distributed settings, it yields a discontinuous vector field complicating the qualitative (existence, uniqueness, invariance and stability) analysis, as anticipated from the proof structure of \refthm{thm.HierInvNav} in \reftab{tab.ProofStructure}.
We find it convenient to proceed instead by developing an alternative, equivalent representation of this vector field. 
Namely, we introduce a family of continuous and piecewise smooth covering vector fields whose application over a partition (derived from their covering domains) of the stratum yields a continuous piecewise smooth flow
(identical to  that generated by the original construction) which is considerably easier to analyze because it admits an interpretation as a sequential composition \cite{Burridge_Rizzi_Koditschek_1999} over the covering family.

We find it useful to first observe that the original construction yields a well defined and effectively computable function. 
\begin{proposition}\label{prop.HierInvNavProperty}
The recursion in \reftab{tab.HierInvNav} results in a well defined function, $\f: \stratum{\tree} \rightarrow \prl{\R^d}^{\indexset}$, that can be computed for each configuration $\vectbf{x} \in \stratum{\tree}$ in $\bigO{\card{\indexset}^2}$ time.
\end{proposition}
\begin{proof}
See \refapp{app.HierInvNavProperty}. \qed
\end{proof}

\begin{table}[t]
\centering
\caption{The Proof Structure of \refthm{thm.HierInvNav} : Logical Dependencies}
\label{tab.ProofStructure}
\vspace{-2mm}
\begin{tabular}{|p{0.475\textwidth}@{\hspace{1mm}}|}
\hline

\vspace{-2mm}
\begin{itemize}[leftmargin=2mm]
\item \refprop{prop.HierInvNavProperty} (Quadratic Time Function) [\ref{app.HierInvNav}, p.\pageref{prop.HierInvNavProperty} $\Leftarrow$ \ref{app.HierInvNavProperty}, p.\pageref{app.HierInvNavProperty}]

\item \refprop{prop.LocalPolicyProperty} (Continuous \!\&\! Piecewise Smooth) 
[\ref{app.EquivalentSystemModel}, p.\pageref{prop.LocalPolicyProperty} $\Leftarrow$ \ref{app.LocalPolicyProperty},p.\pageref{app.LocalPolicyProperty}]

\begin{itemize}[leftmargin =2mm]
\item \reflem{lem.ClusterPartition} (Child Partition Block) [\ref{app.SubstratumObservation}, p.\pageref{lem.ClusterPartition}]
\end{itemize}

\item \refprop{prop.DomainInclusion}  (Domain Covering Induced Partition) \![\ref{app.EquivalentSystemModel}, \!p.\pageref{prop.DomainInclusion} \!$\Leftarrow$\! \ref{app.DomainInclusion}, \!p.\pageref{app.DomainInclusion}]

\item \refprop{prop.SystemEquivalence} (Equivalent Vector Field) [\ref{app.EquivalentSystemModel}, p.\pageref{prop.SystemEquivalence} $\Leftarrow$ \ref{app.SystemEquivalence}, p.\pageref{app.SystemEquivalence}]

\item \refprop{prop.HierInvNavInvariance} (Stratum Positive Invariance) [\ref{app.HierInvNavQualitative}, p.\pageref{prop.HierInvNavInvariance}]

\begin{itemize}[leftmargin=2mm]
\item Recalls \refprop{prop.DomainInclusion},  \refprop{prop.SystemEquivalence}

\item \refprop{prop.DomainInvariance} (Substratum Positive Invariance) [\ref{app.LocalPolicyQualitative}, p.\pageref{prop.DomainInvariance} $\Leftarrow$ \ref{app.DomainInvariance}, p.\pageref{app.DomainInvariance}]

\begin{itemize}[leftmargin=2mm]
\item \reflem{lem.DomainInvarianceBase1} (Invariance - Base Case 1) [\ref{app.DomainInvariance}, p.\pageref{lem.DomainInvarianceBase1} $\Leftarrow$ \ref{app.DomainInvarianceBase1}, p.\pageref{app.DomainInvarianceBase1}]

\item \reflem{lem.DomainInvarianceBase2} (Invariance - Base Case 2) [\ref{app.DomainInvariance}, p.\pageref{lem.DomainInvarianceBase2} $\Leftarrow$ \ref{app.DomainInvarianceBase2}, p.\pageref{app.DomainInvarianceBase2}]

\item \reflem{lem.DomainInvarianceRecursion} (Invariance - Recursion) [\ref{app.DomainInvariance}, p.\pageref{lem.DomainInvarianceRecursion} $\Leftarrow$ \ref{app.DomainInvarianceRecursion}, p.\pageref{app.DomainInvarianceRecursion}]
\end{itemize}
\end{itemize}

\item \refprop{prop.HierInvNavExistenceUniqueness} (Stratum Existence \& Uniqueness) [\ref{app.HierInvNavQualitative}, p.\pageref{prop.HierInvNavExistenceUniqueness}]
 
\begin{itemize}[leftmargin=2mm]
\item Recalls \refprop{prop.DomainInclusion},  \refprop{prop.SystemEquivalence}

\item \refprop{prop.LocalPolicyExistenceUniqueness} (Substratum \!Existence  \!Uniqueness) \![\ref{app.LocalPolicyQualitative}, \!p.\pageref{prop.LocalPolicyExistenceUniqueness} \!$\Leftarrow$\! \ref{app.LocalPolicyExistenceUniqueness}, \!p.\pageref{app.LocalPolicyExistenceUniqueness}]

\begin{itemize}[leftmargin=2mm]
\item Recalls \refprop{prop.LocalPolicyProperty}, \refprop{prop.DomainInvariance}.
\item \reflem{lem.CentroidDynamics} (Relative Centroidal Dynamics) [\ref{app.SubstratumObservation}, p.\pageref{lem.CentroidDynamics} $\Leftarrow$ \ref{app.CentroidDynamics},p.\pageref{app.CentroidDynamics}]
\item \reflem{lem.ConfRadiusBound} (Configuration Bound Radius) [\ref{app.SubstratumObservation}, p.\ref{lem.ConfRadiusBound} $\Leftarrow$ \ref{app.ConfRadiusBound}, p.\pageref{app.ConfRadiusBound}]
\begin{itemize}[leftmargin=*]
\item Recalls \reflem{lem.CentroidDynamics}.
\end{itemize}
\end{itemize}

\item \refprop{prop.HierInvNavStability} (Stratum Stability) [\ref{app.HierInvNavQualitative}, p.\pageref{prop.HierInvNavStability}]

\begin{itemize}[leftmargin=2mm]
\item Recalls \refprop{prop.SystemEquivalence}, \refprop{prop.DomainInvariance}.

\item \refprop{prop.PolicyPriority} (Substratum Policy Selection) [\ref{app.LocalPolicyQualitative}, p.\pageref{prop.PolicyPriority} $\Leftarrow$ \ref{app.PolicyPriority}, p.\pageref{app.PolicyPriority}]
\begin{itemize}[leftmargin=2mm]
\item Recalls \refprop{prop.DomainInclusion}.

\item \reflem{lem.FinerCoarserResolution} (Partition Refinement) [\ref{app.SubstratumObservation}, p.\pageref{lem.FinerCoarserResolution}]
\end{itemize}

\item  \refprop{prop.PreparesRelation} (Finite Time Prepares Relation) [\ref{app.LocalPolicyQualitative}, p.\pageref{prop.PreparesRelation} $\Leftarrow$ \ref{app.PreparesRelation}, p.\pageref{app.PreparesRelation}]

\begin{itemize}[leftmargin=2mm]
\item \reflem{lem.PreparesRelationSingleton} (Case \ref{def.subprepgraph1} in \refdef{def.SubPreparesGraph}) [\ref{app.PreparesRelation}, p.\pageref{lem.PreparesRelationSingleton} $\Leftarrow$ \ref{app.PreparesRelationSingleton}, p.\ref{app.PreparesRelationSingleton}] 

\item \reflem{lem.PreparesRelationNegative} (Case \ref{def.subprepgraph2} in \refdef{def.SubPreparesGraph}) [\ref{app.PreparesRelation}, p.\pageref{lem.PreparesRelationNegative} $\Leftarrow$ \ref{app.PreparesRelationNegative}, p.\ref{app.PreparesRelationNegative}]

\item \reflem{lem.PreparesRelationPositive} (Case \ref{def.subprepgraph3} in \refdef{def.SubPreparesGraph}) [\ref{app.PreparesRelation}, p.\pageref{lem.PreparesRelationPositive} $\Leftarrow$ \ref{app.PreparesRelationPositive}, p.\ref{app.PreparesRelationPositive}]
\end{itemize}
\end{itemize}
\end{itemize}

\item \refprop{prop.PreparesGraph} (Substratum Prepares Graph) [\ref{app.OnlineSequentialComposition}, p.\pageref{prop.PreparesRelation} $\Leftarrow$ \ref{app.PreparesGraph}, p.\pageref{app.PreparesGraph}]

\begin{itemize}[leftmargin=2mm]
\item Recalls \reflem{lem.ClusterPartition}.
\end{itemize}

\item \refprop{prop.NonZeroExecutiontime} (Nondegenerate Execution) [\ref{app.LocalPolicyQualitative}, p.\pageref{prop.NonZeroExecutiontime}]

\begin{itemize}[leftmargin =2mm]
\item Recalls \refprop{prop.PolicyPriority}, \refprop{prop.DomainInvariance}.
\item \reflem{lem.LocalPolicyDomain} (Closed Substratum Domain) [\ref{app.LocalPolicyQualitative}, p.\pageref{lem.LocalPolicyDomain}]
\end{itemize}

\vspace{-3.5mm}

\end{itemize}

\\

\hline
\end{tabular}
\vspace{-2mm}
\end{table}

%%%%%%%%%%%%%%%%%%%%%%%%%%%%%%%%%%%%%%%%
%%%%%%%%%%%%%%%%%%%%%%%%%%%%%%%%%%%%%%%%
\subsection{An Equivalent System Model}
\label{app.EquivalentSystemModel}
%\subsubsection{A Closer Look at the Controller --- The Devil in Details}
%%%%%%%%%%%%%%%%%%%%%%%%%%%%%%%%%%%%%%%%
%%%%%%%%%%%%%%%%%%%%%%%%%%%%%%%%%%%%%%%%

Key for understanding the hierarchy preserving navigation policy, $\f$, in  \reftab{tab.HierInvNav} is the observation that for any configuration $\vectbf{x} \in \stratum{\tree}$  the list of visited clusters of $\tree$  satisfying base conditions during the recursive computation of $\f$ defines a partition  $\purt$ of $\indexset$ compatible with $\tree$, i.e. $\purt \subset \cluster{\tree}$.\footnote{Note that the recursions in \reftab{tab.HierInvNav} and \reftab{tab.PolicySelection} have the same base and recursion conditions, and the recursion in \reftab{tab.PolicySelection} returns the list of clusters satisfying  base conditions, which defines a partition of $\indexset$ (\refprop{prop.LocalPolicyProperty}). 
Hence, using the relation between these recursions in \refprop{prop.SystemEquivalence}, one can conclude this  observation.}

\begin{table}[t]
\caption{Local Control Policies in a Hierarchical Stratum}
\label{tab.LocalPolicy}

\vspace{-2mm}
\centering 
\begin{tabular}{|p{0.45\textwidth}@{\hspace{2mm}}|}
\hline

\vspace{-1mm}
Let $\purt$ be a partition of $\indexset$ with $\purt \subset \cluster{\tree}$, and  $\vectbf{b} = \prl{\vect{b}_{I}}_{I \in \purt} \in \crl{-1, +1} ^{\purt}$.
For any desired $\vectbf{y} \in \stratum{\tree}$, supporting $\tree \in \bintreetopspace_{\indexset}$, and  any initial  $\vectbf{x} \in \domain\prl{\purt, \vectbf{b}}$ \refeqn{eq.LocalPolicyDomain}, the local control policy, $\h{\purt, \vectbf{b}} : \domain\prl{\purt, \vectbf{b}} \rightarrow \prl{\R^d}^{\indexset}$, 
\begin{align}
\h{\purt, \vectbf{b}}\prl{\vectbf{x}} \ldf \hhat{\purt, \vectbf{b}}\prl{\vectbf{x}, \vectbf{0}, \indexset}, \nonumber
\end{align}
is recursively computed using the post-order traversal of $\tree$ starting at the root cluster $\indexset$ with the zero control input $\vectbf{0} \in \prl{\R^d}^{\indexset}$ as follows: for any $\vectbf{u} \in \prl{\R^d}^{\indexset}$ and $I \in \visitedcluster$ \refeqn{eq.VisitedClusters},
\vspace{-2mm}

\begin{tabular}{@{\hspace{-1mm}}c@{\hspace{3mm}}c@{\hspace{2mm}}c}
\begin{tabular}{@{}p{0.001\textwidth}@{}}
\\
$\rotatebox{90}{\hspace{-6mm}Base Cases}\left \{ \begin{array}{c}
\\[0.5mm]
\\[0.5mm]
\\[0.5mm]
\\[0.5mm]
\\[0.5mm]
\\[0.5mm]
\end{array} \right.$ 
\\
$\rotatebox{90}{\hspace{-4.5mm}Recursion}\left \{ \begin{array}{c}
\\[0.6mm]
\\[0.6mm]
\\[0.6mm]
\\[0.6mm]
\\[0.6mm]
\\[0.6mm]
\end{array} \right.$
\end{tabular}
&
\begin{tabular}{@{}p{0.25\textwidth}@{}}

\begin{enumerate}[itemsep=0.5mm]
 \item \textbf{function} $\hat{\vectbf{u}} = \hhat{\purt, \vectbf{b}}\prl{\vectbf{x}, \vectbf{u}, I}$  
 \item \quad \textbf{if} $I \in \purt$,
 \item \quad \quad \textbf{if} $\vect{b}_I = +1$
 \item \quad \quad \quad $\hat{\vectbf{u}} \leftarrow \fA\prl{\vectbf{x}, \vectbf{u}, I}$ \refeqn{eq.AttractiveField},
 \item \quad \quad \textbf{else} 
 \item \quad \quad \quad  $\hat{\vectbf{u}} \leftarrow  \fS\prl{\vectbf{x}, \vectbf{u}, I} $   \refeqn{eq.SeparationField},
 \item \quad \quad \textbf{end}
 \item \quad \textbf{else} 
 \item \quad  \quad $\crl{I_L, I_R} \leftarrow \childCL{I,\tree}$,
 \item \quad \quad $\hat{\vectbf{u}}_L \leftarrow \hhat{\purt, \vectbf{b}}\prl{\vectbf{x}, \vectbf{u}, I_L}$,
 \item \quad \quad $\hat{\vectbf{u}}_R \leftarrow \hhat{\purt, \vectbf{b}}\prl{\vectbf{x}, \hat{\vectbf{u}}_L, I_R}$,
 \item \quad \quad $\hat{\vectbf{u}} \leftarrow  \fH\prl{\vectbf{x}, \hat{\vectbf{u}}_R, I} $ \refeqn{eq.SplitPreservingField},
 \item \quad \textbf{end}
 \item \textbf{return} $\hat{\vectbf{u}}$
 \vspace{-3mm}
\end{enumerate}

\end{tabular}

&
\begin{tabular}{@{}l@{}}
  \\[0.5mm]
  \\[0.5mm]
  \\[1.5mm]
\% Attracting Field \\[0.5mm]
\\[0.6mm]
\% Split Separation Field \\[0.5mm]
\\[0.6mm]
\\[0.6mm]
\\[1.0mm]
\% Recursion for Left Child\\[0.8mm]
\% Recursion for Right Child\\[0.8mm]
\% Split Preserving Field\\[0.7mm]
\\[0.6mm]
\vspace{-3mm}
\end{tabular}

\end{tabular}

\\
\hline
\end{tabular}
\end{table}

Now observe, depending on which base condition holds (\reftab{tab.HierInvNav}.2) or \reftab{tab.HierInvNav}.4)), every block $I$ of partition $\purt$, associated with any fixed configuration $\vectbf{x} \in \setA\prl{I} \cup \,\prl{\stratum{\tree} \setminus \setH\prl{I}}$, can be associated with a binary scalar $\hat{\vect{b}}_I \prl{\vectbf{x}} \in \crl{-1, +1}$ such that\footnote{ Observe from \reftab{tab.HierInvNav} that any configuration $\vectbf{x} \in \stratum{\tree}$ satisfies a base condition (\reftab{tab.HierInvNav}.2) or \reftab{tab.HierInvNav}.4)) at cluster $I \in \cluster{\tree}$ if $\vectbf{x} \in \setA\prl{I} \cup \prl{\stratum{\tree} \setminus \setH\prl{I}}$.
Also note that we have $\setA\prl{I} \cup \prl{\stratum{\tree} \setminus \setH\prl{I}} = \setA\prl{I} \cup \prl{\big.\stratum{\tree} \setminus \prl{\setA\prl{I} \cup \setH\prl{I}}}$, and  $\setA\prl{I}$ and  $\stratum{\tree} \setminus \prl{\setA\prl{I} \cup \setH\prl{I}}$ are disjoint.
}
\begin{align}\label{eq.BaseCondition}
\hat{\vect{b}}_I \prl{\vectbf{x}}  = \left \{ 
\begin{array}{ll}
 + 1 & \text{, if } \vectbf{x} \in \setA\prl{I}, \\
 - 1 & \text{, if } \vectbf{x} \not \in \setA\prl{I} \cup \setH\prl{I},
\end{array}
\right .
\end{align}
where $\setA\prl{I}$ and $\setH\prl{I}$ are defined as in \refeqn{eq.SetA} and \refeqn{eq.SetR}, respectively.
%
% $\vect{b}_I = +1$ if $\vectbf{x} \in \setA_{\tree, \vectbf{y}}\prl{I}$ (line 2)), and $\vect{b}_I = -1$ if $\vectbf{x}\not \in \setA_{\tree, \vectbf{y}}\prl{I} \cup \setR_{\tree}\prl{I}$ (line 4)).
We will use this configuration space labeling scheme to recast  the hierarchy preserving control policy $\f$ as an online sequential composition of a family of continuous and piecewise smooth local controllers indexed by partitions of $\indexset$ compatible with $\tree$ and associated binary vectors as follows.

\smallskip

A partition $\purt$ of $\indexset$ is said to be compatible with $ \tree \in \bintreetopspace_{\indexset}$ if and only if  $\purt \subset \cluster{\tree}$, and denote by $\purtset_{\indexset}\prl{\tree}$ the set of partitions of $\indexset$ compatible with $\tree$.   
Accordingly, define $\subpolicyset{\indexset}{\tree}$ to be the set of substratum policy indices,
\begin{align} \label{eq.SubPolicySet}
\subpolicyset{\indexset}{\tree} \ldf \crl{ \prl{\purt, \vectbf{b}} \Big| \, \purt \in \purtset_{\indexset}\prl{\tree}, \vectbf{b} \in \crl{-1, +1}^{\purt}}.
\end{align}
For any partition $\purt \in \purtset_{\indexset}\prl{\tree}$ of $\indexset$  and $\vectbf{b} \ldf \prl{\vect{b}_{I}}_{I \in \purt} \in \crl{-1, +1} ^ {\purt}$, the domain $\domain\prl{\purt, \vectbf{b}}$ of a local control policy $\h{\purt, \vectbf{b}}$, presented in \reftab{tab.LocalPolicy}, is defined to be
\begin{align}\label{eq.LocalPolicyDomain}
\domain\prl{\purt, \vectbf{b}} \ldf \bigcap _{I \in \purt} \left ( \Big. \right. \!\! \setB\prl{I, \vect{b}_I} \, \scalebox{1.3}{$\cap$} \!\! \bigcap_{K \in \mathrm{Anc}\; \prl{I,\tree}} \!\!\!\! \setH\prl{K} \!\!\left . \Big. \right )\!,  
\end{align}
where the set of configurations satisfying the base condition  associated with cluster $I$ of $\purt$ and  binary scalar $\vect{b}_I$ is given by
\begin{align}\label{eq.SetB}
\setB\prl{I, \vect{b}_I} \ldf \left \{ 
\begin{array}{@{}c@{}@{}l}
\setA\prl{I} & \text{, if } \vect{b}_I = + 1, \\
\stratum{\tree} & \text{, if } \vect{b}_I = - 1, 
\end{array}
\right .
\end{align}
and all ancestors $K \in \ancestorCL{I,\tree}$ of $I$ in $\tree$ satisfy the recursion condition of having properly separated children clusters described by $\setH\prl{K}$ \refeqn{eq.SetR}. 
Accordingly, let  $\visitedcluster$ denote the set of clusters of $\tree$ visited during the recursive computation of $\h{\purt, \vectbf{b}}$  in \reftab{tab.LocalPolicy},
\begin{align}\label{eq.VisitedClusters}
\visitedcluster \ldf \crl{K \in \cluster{\tree} \big| K \supseteq I, I \in \purt }.
\end{align}
Note that $\indexset \in \visitedcluster$ since $\purt$ is a partition of the root cluster $\indexset$ and any block $I \in \purt$ satisfies $I \subseteq \indexset$.

Observe that each local control policy $\h{\purt, \vectbf{b}}$ is a recursive composition of continuous functions of $\vectbf{x}$, so it is continuous:
\begin{proposition}\label{prop.LocalPolicyProperty}
The recursion in \reftab{tab.LocalPolicy} results in a continuous and piecewise smooth function\footnote{Note that if $f: U \rightarrow \R^m$ is continuous and piecewise smooth on an open set $U \subset \R^n$, then it is locally Lipschitz on $U$ \cite{kuntz_scholtes_JMAA1994}.}, $\h{\purt, \vectbf{b}} : \stratum{\tree} \rightarrow \prl{\R^d}^{\indexset}$.
\end{proposition}
\begin{proof}
See \refapp{app.LocalPolicyProperty}. \qed
\end{proof}

\smallskip

To conclude our introduction of the family of   covering fields in \reftab{tab.LocalPolicy}, we now observe that the vector field $\f$ in \reftab{tab.HierInvNav} is an online concatenation of continuous local controllers, $\h{\purt, \vectbf{b}}$, of \reftab{tab.LocalPolicy} using a policy selection method described in \reftab{tab.PolicySelection}, summarized as: 

\begin{table}[t]
\caption{Policy Selection Algorithm}
\label{tab.PolicySelection}

\vspace{-2mm}
\centering 
\begin{tabular}{|p{0.45\textwidth}@{\hspace{2mm}}|}
\hline

\vspace{-1mm}
For any initial  $\vectbf{x} \in \stratum{\tree}$ and desired $\vectbf{y} \in \stratum{\tree}$, supporting $\tree \in \bintreetopspace_{\indexset}$, the policy selection algorithm, $\p : \stratum{\tree} \rightarrow \subpolicyset{\indexset}{\tree}$,
\begin{align}
\p\prl{\vectbf{x}} \ldf \phat\prl{\vectbf{x}, \indexset}, \nonumber
\end{align}
recursively generates a local policy index in $\subpolicyset{\indexset}{\tree}$ \refeqn{eq.SubPolicySet} using the post-order traversal of $\tree$ starting at the root cluster $\indexset$ as follows: for any  $I \in \cluster{\tree}$,
\vspace{-2mm}

\begin{tabular}{@{\hspace{0mm}}c@{\hspace{3mm}}c@{\hspace{2mm}}}
\begin{tabular}{@{}p{0.001\textwidth}@{}}
\\
$\rotatebox{90}{\hspace{-6mm}Base Cases}\left \{ \begin{array}{c}
\\[0.6mm]
\\[0.6mm]
\\[0.6mm]
\\[0.6mm]
\\[0.6mm]
\\[0.6mm]
\end{array} \right.$ 
\\
$\rotatebox{90}{\hspace{-4.5mm}Recursion}\left \{ \begin{array}{c}
\\[0.5mm]
\\[0.5mm]
\\[0.5mm]
\\[0.5mm]
\\[0.5mm]
\\[0.5mm]
\\[0.5mm]
\end{array} \right.$
\\
\end{tabular}
&
\begin{tabular}{@{}p{0.42\textwidth}@{}}

\begin{enumerate}[itemsep=0.5mm]
 \item \textbf{function} $(\hat{\mathcal{I}}, \hat{\vectbf{b}}) = \phat\prl{\vectbf{x}, I}$  
 \item \quad \textbf{if} $\vectbf{x} \in \setA\prl{I}$ \refeqn{eq.SetA},
 \item \quad \quad $\hat{\mathcal{I}} \leftarrow  \crl{I}$,
 \item \quad \quad $\hat{\vectbf{b}} \leftarrow   +1$,
 \item \quad \textbf{else if} $\vectbf{x} \not \in \setH\prl{I}$ \refeqn{eq.SetR},
 \item \quad \quad $\hat{\mathcal{I}} \leftarrow  \crl{I}$,
 \item \quad \quad $\hat{\vectbf{b}} \leftarrow   -1$,
 \item \quad \textbf{else} 
 \item \quad  \quad $\crl{I_L, I_R} \leftarrow \childCL{I,\tree}$,
 \item \quad \quad $(\hat{\mathcal{I}}_L, \hat{\vectbf{b}}_L) \leftarrow \phat\prl{\vectbf{x}, I_L}$,
 \item \quad \quad $(\hat{\mathcal{I}}_R, \hat{\vectbf{b}}_R) \leftarrow \phat\prl{\vectbf{x}, I_R}$,
 \item \quad \quad $\hat{\mathcal{I}} \leftarrow \hat{\mathcal{I}}_L \cup \hat{\mathcal{I}}_R$,
 \item \quad \quad $\hat{\vectbf{b}} \leftarrow \hat{\vectbf{b}}_L \| \hat{\vectbf{b}}_R$, \footnotemark
 \item \quad \textbf{end}
 \item \textbf{return} $(\hat{\mathcal{I}}, \hat{\vectbf{b}})$
 \vspace{-3mm}
\end{enumerate}

\end{tabular}

\end{tabular}

\\
\hline
\end{tabular}
\end{table}

\footnotetext{Here, $\vectbf{p} \| \vectbf{q}$ denotes the concatenation of vectors $\vectbf{p}$ and $\vectbf{q}$. 
That is to say, let $X, Y$ be two sets and $A,B$ be two finite sets of coordinate indices, then for any  $\vectbf{p} \in X^A$ and $\vectbf{q} \in Y^B$ we say $\vectbf{r} \in X^A \times Y^B$ is the concatenation of $\vectbf{p}$ and $\vectbf{q}$, denoted by $\vectbf{r} = \vectbf{p} \| \vectbf{q}$, if and only if $\vect{r}_a = \vect{p}_a$ and $\vect{r}_b = \vect{q}_b$ for all $a \in A$ and $b \in B$.  }

\begin{proposition}\label{prop.DomainInclusion}
For any given configuration $\vectbf{x} \in \stratum{\tree}$ the policy selection algorithm $\p$ in \reftab{tab.PolicySelection} always returns a valid policy index, $\prl{\mathcal{\indexset, \vectbf{b}}} = \p\prl{\vectbf{x}}$, in $\subpolicyset{\indexset}{\tree}$ \refeqn{eq.SubPolicySet} such that the domain $\domain\prl{\purt, \vectbf{b}}$ \refeqn{eq.LocalPolicyDomain} of the associated local control policy  $\h{\purt, \vectbf{b}}$ (\reftab{tab.LocalPolicy})  contains $\vectbf{x}$, i.e.
\begin{align}
\vectbf{x} \in \prl{\domain \circ \p}\prl{\vectbf{x}}.
\end{align}
\end{proposition}
\begin{proof}
See \refapp{app.DomainInclusion}.\qed
\end{proof}

\begin{proposition}\label{prop.SystemEquivalence}
For any given configuration  $\vectbf{x} \in \stratum{\tree}$, the vector field $f_{\tree, \vectbf{x}}$ (\reftab{tab.HierInvNav}) and the local control policy $\h{\purt, \vectbf{b}}\prl{\vectbf{x}}$ (\reftab{tab.LocalPolicy}) selected as $\prl{\mathcal{\indexset, \vectbf{b}}} = \p\prl{\vectbf{x}}$ (\reftab{tab.PolicySelection})   generate the same control (velocity) inputs, i.e.  
\begin{align}
\f\prl{\vectbf{x}} = \h{\p\prl{\vectbf{x}}}\prl{\vectbf{x}}.
\end{align}
\end{proposition}
\begin{proof}
See \refapp{app.SystemEquivalence}. \qed
\end{proof}
Since the hierarchy invariant field $\f$  is defined for entire $\stratum{\tree}$, it is useful to remark that the domains, $\domain\prl{\purt, \vectbf{b}}$ \refeqn{eq.LocalPolicyDomain}, of local control policies, $\h{\purt, \vectbf{b}}$, define a cover of $\stratum{\tree}$ indexed by partitions of $\indexset$ compatible with $\tree$ and associated binary vectors.

%%%%%%%%%%%%%%%%%%%%%%%%%%%%%%%%%%%%%%%%%%
%%%%%%%%%%%%%%%%%%%%%%%%%%%%%%%%%%%%%%%%%%
\subsection{Online Sequential Composition of Substratum Policies}
\label{app.OnlineSequentialComposition}
%%%%%%%%%%%%%%%%%%%%%%%%%%%%%%%%%%%%%%%%%%
%%%%%%%%%%%%%%%%%%%%%%%%%%%%%%%%%%%%%%%%%%

We  now briefly describe the logic behind online  sequential composition \cite{Burridge_Rizzi_Koditschek_1999}  of substratum policies.

%Although a given configuration $\vectbf{x} \in \stratum{\tree}$ might be contained in the domain of more than one local controller, the policy selection method in \reftab{tab.PolicySelection} selects only one of such available controllers, i.e. it is a well defined function from $\stratum{\tree}$ to the set of local policy indices.  
To characterize our policy selection strategy, we first define a priority  measure\footnote{In the general past literature, such a priority assignment of local controllers is done using backchaining of the prepares graph in an offline manner\cite{Burridge_Rizzi_Koditschek_1999}.} for each local controller $\h{\purt, \vectbf{b}}$ associated with a partition $\purt \in \purtset_{\indexset}\prl{\tree}$ of $\indexset$ and a binary vector $\vectbf{b} \in \crl{-1,+1}^{\purt}$ to be
\begin{align}\label{eq.Priority}
\priority\prl{\purt, \vectbf{b}} \ldf \sum_{I \in \purt} \vect{b}_{I} \card{I}^2.
\end{align}
Note that the maximum and minimum of the priority measure is attained at the coarsest  partition $\crl{\indexset}$ of $\indexset$, and $\vect{b}_{\indexset} = +1$ and $\vect{b}_{\indexset} = -1$, respectively,
\begin{subequations} \label{eq.PriorityRange}
\begin{align}
\priority\prl{\crl{\indexset}, +1} &= \card{\indexset}^2, \\
\priority\prl{\crl{\indexset}, -1} &= -\card{\indexset}^2.
\end{align} 
\end{subequations}
Accordingly, we shall refer to the local control policy with index $\prl{\crl{\indexset}, +1}$ as the goal policy since it  has the highest priority and asymptotically steers all configurations in its domain $\domain\prl{\crl{\indexset}, +1}$\refeqn{eq.LocalPolicyDomain} to $\vectbf{y}$ following the negated gradient of $\V\prl{\vectbf{x}} = \frac{1}{2}\norm{\vectbf{x} - \vectbf{y}}$, i.e. for any $\vectbf{x} \in \domain\prl{\crl{\indexset}, +1}$
\begin{align}
\h{\crl{\indexset}, + 1} \prl{\vectbf{x}} = - \nabla \V\prl{\vectbf{x}} = - (\vectbf{x} - \vectbf{y}).
\end{align} 
Note that since the root cluster $\indexset$ has no ancestor, i.e. $\ancestorCL{\indexset,\tree} = \emptyset$, by definition \refeqn{eq.LocalPolicyDomain}, $\domain\prl{\crl{\indexset}, +1} = \setA\prl{\indexset}$,
and $\setA\prl{\indexset}$ \refeqn{eq.SetA} contains the goal configuration $\vectbf{y}$.

We now introduce an abstract connection between local policies for high-level planning:
\begin{definition}\label{def.PreparesGraph}
Let $\prl{\purt, \vectbf{b}},(\purt', \vectbf{b}') \in \subpolicyset{\indexset}{\tree}$ be  two distinct substratum policy indices. 
Then  $\h{\purt, \vectbf{b}}$ is said to \emph{prepare} $\h{\purt', \vectbf{b}'}$ if and only if all trajectories of $\h{\purt, \vectbf{b}}$ starting in its domain $\domain\prl{\purt, \vectbf{b}}$, possibly excluding a set of measure zero, reach $\domain(\purt', \vectbf{b}')$ in finite time.\footnote{ Here, we slightly relax the original definition of the prepares relation  in  \cite{Burridge_Rizzi_Koditschek_1999}  by not requiring the knowledge of goal sets, globally asymptotically stable states, of local control policies in advance.} 

Accordingly,  define the \emph{prepares graph} $\subpreparesgraph = \prl{\subpolicyset{\indexset}{\tree}, \subpreparesedgeset}$ to have vertex set $\subpolicyset{\indexset}{\tree}$\refeqn{eq.SubPolicySet}  with a policy index $\prl{\purt, \vectbf{b}} \sqz{\in} \subpolicyset{\indexset}{\tree}$ connected to another policy index $(\purt', \vectbf{b}')$ by a directed edge in $\subpreparesedgeset$ if and only if $\h{\purt, \vectbf{b}}$ prepares $\h{\purt', \vectbf{b}'}$. 
\end{definition}

\noindent Although, the prepares graph $\subpreparesgraph$ is the most critical component of the sequential composition framework \cite{Burridge_Rizzi_Koditschek_1999} defining a discrete abstraction of continuous control policies, the exponentially growing cardinality of substratum policies, discussed in \refapp{app.SubstratumCardinality}, and the lack of an explicit characterization of globally asymptotically stable configurations of substratum policies make it usually difficult to compute the complete prepares graph. 

Alternatively, we introduce a computationally efficient and recursively constructed graph of substratum policies that is nicely compatible with our needs, yielding a subgraph of the prepares graph, where every policy index is connected to the goal policy index $\prl{\crl{\indexset}, +1}$ through a directed path, as follows.   

\begin{definition} \label{def.SubPreparesGraph}
Let $\subpreparesgraphS = (\subpolicyset{\indexset}{\tree}, \subpreparesedgesetS)$ be a graph with vertex list $\subpolicyset{\indexset}{\tree}$, and a policy index $\prl{\purt, \vectbf{b}} \in \subpolicyset{\indexset}{\tree}$ that is connected to another policy index $(\purt', \vectbf{b}') \in \subpolicyset{\indexset}{\tree}$ by a directed edge in $\subpreparesedgesetS$ if and only if at least one of the following properties holds:\footnote{One may think of these conditions as restructuring operations of policy indices by merging/splitting of partition blocks and/or alternating binary index values, like NNI moves of trees in \refsec{sec.TreeGraph}.}
\begin{enumerate}[label=(\roman*)]
\item \label{def.subprepgraph1} (Complement) There exists a singleton  cluster $I \in \purt$ such that $\vect{b}_I = -1$, and  $\purt' = \purt$ and $\vectbf{b}' \in \crl{-1, +1}^{\purt'}$ with  $\vect{b}'_I = +1$ and $\vect{b}'_D = \vect{b}_D$ for all $D \in \purt \setminus \crl{I}$.

\item \label{def.subprepgraph2} (Split) There exists a nonsingleton cluster $I \in \purt$ such that $\vect{b}_I = -1$, and $\purt' = \purt \setminus \crl{I} \cup \childCL{I,\tree}$  and 
$\vectbf{b}' \in \crl{-1, +1}^{\purt'}$ with  $\vect{b}'_K = -1$ for all $K \in \childCL{I,\tree}$ and $\vect{b}'_D = \vect{b}_D$ for all $D \in \purt \setminus \childCL{I,\tree}$.

\item \label{def.subprepgraph3} (Merge) There exists a nonsingleton cluster $I \sqz{\in} \cluster{\tree}$ such that $\childCL{I,\tree} \sqz{\subset} \purt$  and $\vect{b}_K \sqz{=} +1$ for all $K \sqz{\in} \childCL{I,\tree}$, and $\purt' = \purt \setminus \childCL{I,\tree} \cup \crl{I}$ and $\vectbf{b}' \in \crl{-1, +1}^{\purt'}$ with $\vect{b}'_I = +1$ and $\vect{b}'_D = \vect{b}_D$ for all $D \in \purt \setminus \childCL{I,\tree}$.
\end{enumerate}
\end{definition}

\noindent Note that, since $\purt$ is compatible with $\tree$, i.e. $\purt \subset \cluster{\tree}$, if $\card{\purt}> 1$, then there exists  a cluster $I \in \cluster{\tree}$ such that $\childCL{I,\tree} \subset \purt$ (\reflem{lem.ClusterPartition}).
Hence, for any policy index $\prl{\purt, \vectbf{b}} \neq \prl{\crl{\indexset}, +1}$ there always exists a policy index $(\purt', \vectbf{b}') \neq \prl{\purt, \vectbf{b}}$ satisfying one of these conditions,  \ref{def.subprepgraph1}-\ref{def.subprepgraph3} above. 
Thus, the out-degree of a policy index $\prl{\purt, \vectbf{b}} \neq \prl{\crl{\indexset}, +1}$ in  $\subpreparesgraphS$ is at least one, whereas the goal policy index $\prl{\crl{\indexset}, +1}$ in  $\subpreparesgraphS$ has an out-degree of zero.    
We summarize some important properties of  $\subpreparesgraphS$ as follows:

\begin{proposition}\label{prop.PreparesGraph}
The graph $\subpreparesgraphS = (\subpolicyset{\indexset}{\tree}, \subpreparesedgesetS)$, as defined in \refdef{def.SubPreparesGraph}, is an acyclic subgraph of the prepares graph $\subpreparesgraph = (\subpolicyset{\indexset}{\tree}, \subpreparesedgeset)$ (\refdef{def.PreparesGraph}) such that all policy indices in $\subpolicyset{\indexset}{\tree}$ are connected to the goal policy index $\prl{\crl{\indexset}, +1}$ through directed paths in $\subpreparesedgesetS$, of length at most $\bigO{\card{\indexset}^2}$ hops, along which $\priority$ \refeqn{eq.Priority} is strictly increasing, i.e. for any $\left ( \big. \right. \!\prl{\purt, \vectbf{b}}\!, (\purt', \vectbf{b}') \!\! \left. \big. \right) \in \subpreparesedgesetS$
\begin{align}
\priority(\purt', \vectbf{b}') > \priority\prl{\purt, \vectbf{b}}.
\end{align}
\end{proposition}
\begin{proof}
See \refapp{app.PreparesGraph}. \qed
\end{proof}

Although a given local policy can prepare more than one potential successor (i.e. higher priority), our policy selection method chooses the one with the strictly highest priority:
\begin{proposition}\label{prop.PolicyPriority}
For any given $\vectbf{x} \in \stratum{\tree}$ the policy selection method, $\p$, in \reftab{tab.PolicySelection} always returns the index of a local controller with the maximum priority among all local controllers whose domain contains $\vectbf{x}$, 
\begin{align}
 \p\prl{\vectbf{x}} = 
\argmax_{
\substack{
(\purt', \vectbf{b}') \in \subpolicyset{\indexset}{\tree} \\
\vectbf{x} \in \domain\prl{\purt', \vectbf{b}'}
}
} 
\priority (\purt', \vectbf{b}').
\end{align}
 and all the other available local controllers have strictly lower priorities.
\end{proposition}
\begin{proof}
See \refapp{app.PolicyPriority}. \qed
\end{proof}

%%%%%%%%%%%%%%%%%%%%%%%%%%%%%%%%%%%%%%%%%%%%%%%%
%%%%%%%%%%%%%%%%%%%%%%%%%%%%%%%%%%%%%%%%%%%%%%%%
\subsection{Qualitative Properties of Substratum Policies}
\label{app.LocalPolicyQualitative}
%%%%%%%%%%%%%%%%%%%%%%%%%%%%%%%%%%%%%%%%%%%%%%%%
%%%%%%%%%%%%%%%%%%%%%%%%%%%%%%%%%%%%%%%%%%%%%%%%

We now list important qualitative (existence, uniqueness, invariance and stability) properties of the substratum control policies of \reftab{tab.LocalPolicy}.
Let $\purt$ be a partition of $\indexset$ compatible with $\tree$, i.e. $\purt \subset \cluster{\tree}$, and $\vectbf{b}$ is a binary vector in $\crl{-1,1}^{\purt}$.

\begin{proposition} \label{prop.DomainInvariance}
The domain,  $\domain\prl{\purt, \vectbf{b}}$ \refeqn{eq.LocalPolicyDomain}, of a substratum policy, $\h{\purt, \vectbf{b}}$ (\reftab{tab.LocalPolicy}), is positive invariant.
\end{proposition}
\begin{proof}
See \refapp{app.DomainInvariance}. \qed
\end{proof}

\begin{proposition} \label{prop.LocalPolicyExistenceUniqueness}
(Substratum Existence and Uniqueness) The vector field $\h{\purt, \vectbf{b}}$ (\reftab{tab.LocalPolicy}) is locally Lipschitz in $\stratum{\tree}$; and
for any initial $\vectbf{x} \in \domain\prl{\purt, \vectbf{b}} \subset \stratum{\tree}$ there always exists a compact (bounded and closed) subset $W$ of $ \domain\prl{\purt, \vectbf{b}}$ \refeqn{eq.LocalPolicyDomain} such that all trajectories of $\h{\purt, \vectbf{b}}$ starting at $\vectbf{x}$ remain in $W$ for all future time. 

Therefore, there is a unique continuous and piecewise smooth flow of $\h{\purt, \vectbf{b}}$  in $\domain\prl{\purt, \vectbf{b}}$ that is defined for all future time.
\end{proposition}
\begin{proof}
See \refapp{app.LocalPolicyExistenceUniqueness}. \qed
\end{proof}

\begin{proposition}\label{prop.PreparesRelation}
(Finite Time Prepares Relation) Each local control policy, $\h{\purt, \vectbf{b}}$, with the exception of the goal controller $\h{\crl{\indexset}, +1}$,  steers (almost) all configurations in its domain, $\domain\prl{\purt, \vectbf{b}}$, to the domain, $ \domain\prl{\purt', \vectbf{b}'}$, of another local controller, $\h{\purt', \vectbf{b}'}$, at a higher $\,\priority$ \refeqn{eq.Priority} in finite time.  
\end{proposition}
\begin{proof}
See \refapp{app.PreparesRelation}. \qed
\end{proof}

\begin{proposition} \label{prop.NonZeroExecutiontime}
(Nonzero Execution Time) Let $\vectbf{x}^t$ be a trajectory of the local control policy $\h{\purt, \vectbf{b}}$ starting at $\vectbf{x}^0 \in \domain\prl{\purt, \vectbf{b}}$ such that $\p\prl{\vectbf{x}^0} = \prl{\purt, \vectbf{b}}$.

Then the local controller is guaranteed to steers the group for a nonzero time until reaching the domain of a local controller at a higher $\priority$ \refeqn{eq.Priority}, i.e.
\begin{align}
\inf_{t} \crl{t \geq 0 \big | \p\prl{\vectbf{x}^t} \neq \prl{\purt, \vectbf{b}}} > 0. 
\end{align}
\end{proposition}
\begin{proof}

Recall that for any configuration $\vectbf{x} \in \stratum{\tree}$ the policy selection method in \reftab{tab.PolicySelection} always yields the index of the local controller with the highest priority among all local controllers whose domains contain $\vectbf{x}$ (\refprop{prop.PolicyPriority}).
Hence, since the initial configuration $\vectbf{x}^0$ is not included in the domain of any other local controller with a higher priority than $\priority(\purt, \vectbf{b})$ and domains of local controllers are closed relative to $\stratum{\tree}$ (\reflem{lem.LocalPolicyDomain}), there always exists an open set around $\vectbf{x}^0$ which does not intersect with the domain of any local controller at a higher priority than $\priority(\purt, \vectbf{b})$.
Thus, since its domain is positively invariant (\refprop{prop.DomainInvariance}), $\h{\purt, \vectbf{b}}$ is guaranteed to steer the configuration in the intersection of this open set and $\domain\prl{\purt, \vectbf{b}}$ for a nonzero time. \qed
\end{proof}

%%%%%%%%%%%%%%%%%%%%%%%%%%%%%%%%%%%%%%%%%%%%%%%%
%%%%%%%%%%%%%%%%%%%%%%%%%%%%%%%%%%%%%%%%%%%%%%%%
\subsection{Qualitative Properties of Stratum Policies}
\label{app.HierInvNavQualitative}
%%%%%%%%%%%%%%%%%%%%%%%%%%%%%%%%%%%%%%%%%%%%%%%%
%%%%%%%%%%%%%%%%%%%%%%%%%%%%%%%%%%%%%%%%%%%%%%%%

We now proceed with some important qualitative (existence, uniqueness, invariance and stability) properties of  the hierarchy preserving navigation policy of \reftab{tab.HierInvNav}.

\begin{proposition} \label{prop.HierInvNavInvariance}
The stratum $\stratum{\tree}$ is positive invariant under the hierarchy-invariant control policy,  $\f$ (\reftab{tab.HierInvNav}).
\end{proposition}
\begin{proof}
Recall that the domains, $\domain$ \refeqn{eq.LocalPolicyDomain}, of local control policies in \reftab{tab.LocalPolicy} define a cover of $\stratum{\tree}$ (\refprop{prop.DomainInclusion}) each of whose elements is positively invariant under the flow of the associated local policy (\refprop{prop.DomainInvariance}).
Thus, the result follows since the hierarchy preserving vector field $\f$ is equivalent to  online sequential composition of local control policies of \reftab{tab.LocalPolicy} based on the policy selection algorithm in \reftab{tab.PolicySelection} (\refprop{prop.SystemEquivalence}). \qed
\end{proof}

\begin{proposition} \label{prop.HierInvNavExistenceUniqueness}
(Stratum Existence and Uniqueness) The hierarchy invariance control policy,  $\f$ (\reftab{tab.HierInvNav}), has a unique, continuous and  piecewise smooth flow, $\flow^t$, in $\stratum{\tree}$, defined for all $t \geq 0$. 
\end{proposition}
\begin{proof}
Recall from \refprop{prop.SystemEquivalence} that $\f$ is equivalent to online sequential composition of a family of substratum policies which have unique, continuous and piecewise smooth flows, defined for all $t \geq 0$, in their positive invariant domains (\refprop{prop.LocalPolicyExistenceUniqueness}).
Since their domains  define a finite closed cover of $\stratum{\tree}$ (\refprop{prop.DomainInclusion}), the unique, continuous and piecewise flow of  $\f$ is constructed by piecing together trajectories of these substratum policies. \qed
\end{proof}

\begin{proposition} \label{prop.HierInvNavStability}
Any $\vectbf{y} \in \stratum{\tree}$ is an asymptotically stable equilibrium point of the hierarchy-invariant control policy, $\f$ (\reftab{tab.HierInvNav}), whose basin of attraction includes $\stratum{\tree}$, except a set of measure zero.  
\end{proposition}
\begin{proof}
Using the equivalence (\refprop{prop.SystemEquivalence}) of the hierarchy preserving field $\f$ and the sequential composition of substratum control policies of \reftab{tab.LocalPolicy} based on the policy selection method in \reftab{tab.PolicySelection}, the result can be obtained as follows.

Since $\priority$ \refeqn{eq.Priority} is an integer-valued function with bounded  range \refeqn{eq.PriorityRange}, using \refprop{prop.PolicyPriority} and \refprop{prop.PreparesRelation}, one can conclude that the disks starting at almost any configuration in $\stratum{\tree}$ reach the domain $\domain\prl{\crl{\indexset}, +1} $ of the goal policy $\h{\crl{\indexset, +1}}$ in finite time after visiting at most $\bigO{\card{\indexset}^2}$ of other local control policies.
Note that $\vectbf{y} \in \domain\prl{\crl{\indexset}, +1}$.
Then, the goal policy $\h{\crl{\indexset}, +1} $
\begin{align}
\h{\crl{\indexset}, +1} \prl{\vectbf{x}}= - \nabla \scalebox{1}{$\frac{1}{2}$}\norm{\vectbf{x} - \vectbf{y}}_2^2 =  - \prl{\vectbf{x} - \vectbf{y}},
\end{align} 
asymptotically steers all configuration in  $\domain\prl{\crl{\indexset}, +1} $ to $\vectbf{y}$ while  keeping its domain of attraction $\setA\prl{\indexset}$  positively invariant (\refprop{prop.DomainInvariance}), which completes the proof \qed
\end{proof}

%%%%%%%%%%%%%%%%%%%%%%%%%%%%%%%%%%%%%%%%%%%%%%%%%%%%%%%%
%%%%%%%%%%%%%%%%%%%%%%%%%%%%%%%%%%%%%%%%%%%%%%%%%%%%%%%%
\subsection{On the Cardinality of Substratum Policies}
\label{app.SubstratumCardinality}
%%%%%%%%%%%%%%%%%%%%%%%%%%%%%%%%%%%%%%%%%%%%%%%%%%%%%%%%
%%%%%%%%%%%%%%%%%%%%%%%%%%%%%%%%%%%%%%%%%%%%%%%%%%%%%%%%

To gain an appreciation for the computational efficiency of hierarchy preserving vector field in \reftab{tab.HierInvNav},  we find it useful to have a brief discussion without proofs on the cardinality of the family of local control policies of \reftab{tab.LocalPolicy}.   
The number of partitions $\mathcal{P}_{\indexset}\prl{\tree}$  of $\indexset$
\footnote{The number of partitions of a set with $n$ elements is given by the Bell number, $B_n$, recursively defined as: for any $n \in \N$ \cite{rota_AMM1964}
\begin{align}
B_{n+1} = \sum_{k = 0}^{n} {n \choose k} B_k,
\end{align} 
where $B_0 = 1$. 
The Bell number, $B_n$, grows super exponentially with the set size, $n$; 
however,  in our case we require partitions of $\indexset$ to be compatible with $\tree$ and this restricts the growth of number of such partitions of $\indexset$ to at most exponential with $\card{\indexset}$, depending on the structure of $\tree$.} 
compatible with a cluster hierarchy $\tree \in \bintreetopspace_{\indexset}$ is recursively given by 
\footnote{Let $\crl{\indexset_L, \indexset_R} = \childCL{\indexset, \tree}$ be the root split of $\tree$, and $\tree_L$ and $\tree_R$ are the associated subtrees of $\tree$ rooted at $\indexset_L$ and $\indexset_R$, respectively. 
Then, any partition of $\indexset$  compatible with $\tree$, except the trivial partition $\crl{\indexset}$, can be written as the union of a partition of $\indexset_L$ compatible with $\tree_L$ and a partition of $\indexset_R$ compatible with $\tree_R$. 
Hence, one can conclude the recursion in \refeqn{eq.CompatiblePartition}.}
\begin{align}\label{eq.CompatiblePartition}
\card{\mathcal{P}_{\indexset}\prl{\tree}} = 1 + \card{\mathcal{P}_{\indexset}\prl{\tree_L}} \card{\mathcal{P}_{\indexset}\prl{\tree_R}},
\end{align}
where $\tree_L, \tree_R$ denote the left and right subtrees of $\tree$, respectively. 
For any caterpillar tree\footnote{A caterpillar tree is a rooted tree in which at most one of the children of every interior cluster is nonsingleton.} $\treeA \in \bintreetopspace_{\indexset}$, $\card{\mathcal{P}_{\indexset}\prl{\treeA}} = \card{\indexset}$ since one of two subtrees of $\treeA$ is always one-leaf tree. 
On the other hand, for a balanced tree $\treeC \in \bintreetopspace_{\indexset}$ the cardinality of partitions of $\indexset$ compatible with $\treeC$ grows exponentially,\footnotemark 
\begin{align}
\sqrt{2}^{\card{\indexset}} \leq \card{\mathcal{P}_{\indexset}\prl{\treeC}}  \leq  \frac{4}{5}\sqrt{\frac{5}{2}}^{\card{\indexset}},
\end{align}  
for $\card{J}  = 2^{k}$, $k  \in \N_{+} = \crl{ 1, 2, 3,  \ldots}$; for example, $\card{\mathcal{P}_{\brl{2}}\prl{\treeC}} = 2$, $\card{\mathcal{P}_{\brl{4}}\prl{\treeC}} = 5$, $\card{\mathcal{P}_{\brl{8}}\prl{\treeC}} = 26$ and $\card{\mathcal{P}_{\brl{16}}\prl{\treeC}} = 677$.
In addition to a partition $\purt$ of $\indexset$ compatible with $\tree$, every local control policy $\h{\purt, \vectbf{b}}$ is indexed by a binary variable of size $\card{\purt}$ with a possible choice of $2^{\card{\purt}}$ values.   
Therefore, the number of local control policies $\h{\purt, \vectbf{b}}$  grows exponentially with the group size, $\card{\indexset}$.
% but the hierarchy preserving swarming in \reftab{tab.HierInvNav} uses an efficient recursive  policy selection method of \reftab{tab.PolicySelection}.   

\footnotetext{Let $F_n$ denote the number of partitions of $\brl{n} = \crl{1,2,\ldots, n}$ compatible with a balanced rooted binary tree with $n$ leaves, where $n = 2^k$  for some $k \in \N_+$, and by \refeqn{eq.CompatiblePartition}  it satisfies
\begin{align}
F_{2n} = 1 + F_n^2, 
\end{align}
subject to the base condition $F_2 = 2$.
Define $G_n$ and $H_n$, for $n = 2^k$  and $k \in \N_+$, to be, respectively,
\begin{align}
G_{2n} = G_n^2 \quad \text{and} \quad H_{2n} = \frac{5}{4} H_n^2
\end{align}
where $G_2 = H_2 = 2$.
Note that $G_n = \sqrt{2}^n$ and $H_n = \frac{4}{5}\sqrt{\frac{5}{2}}^n$ for $n = 2^k$ and $k \in \N_+$.
Now observe that for any $n = 2^k$ and $k \in \N_+$
\begin{align}
G_n \leq F_n \leq H_n,
\end{align}
and so
\begin{align}
\sqrt{2}^n \leq F_n \leq \frac{4}{5}\sqrt{\frac{5}{2}}^n.
\end{align}
}

%%%%%%%%%%%%%%%%%%%%%%%%%%%%%%%%%%%%%%%%%%%%%%%%
%%%%%%%%%%%%%%%%%%%%%%%%%%%%%%%%%%%%%%%%%%%%%%%%
\subsection{A Set of Useful Observations on Substratum Policies}
\label{app.SubstratumObservation}
%%%%%%%%%%%%%%%%%%%%%%%%%%%%%%%%%%%%%%%%%%%%%%%%
%%%%%%%%%%%%%%%%%%%%%%%%%%%%%%%%%%%%%%%%%%%%%%%%

Here we introduce a set of useful lemmas that constitute building blocks for proving some qualitative properties of substratum policies presented in \refapp{app.LocalPolicyQualitative}.   
Let $\purt$ be a partition of $\indexset$ compatible with $\tree$, i.e.  $\purt \subset \cluster{\tree}$, and $\vectbf{b} $ is a binary vector in $\crl{-1, +1}^{\purt}$. 

\begin{lemma}\label{lem.LocalPolicyDomain} 
The domain, $\domain\prl{\purt,\vectbf{b}}$ \refeqn{eq.LocalPolicyDomain}, of each substratum policy, $\h{\purt, \vectbf{b}}$, is closed relative to $\stratum{\tree}$.
\end{lemma}
\begin{proof}
Using the continuity of functions\footnote{A function $f: X \rightarrow Y$ between two topological spaces, $X$ and $Y$, is  continuous if the inverse image of every open subset of $Y$ of $f$ is an open subset of $X$ \cite{munkres_Topology_2000}.} in the predicates used to define them, one can conclude that for any $I \in \cluster{\tree}$ sets  $\setA\prl{I}$ \refeqn{eq.SetA} and $\setH\prl{I}$ \refeqn{eq.SetR}  are closed relative to $\stratum{\tree}$. Hence, since the intersection of arbitrary many closed sets are closed \cite{munkres_Topology_2000}, the domain $\domain\prl{\purt, \vectbf{b}}$ \refeqn{eq.LocalPolicyDomain} of each local controller $\h{\purt, \vectbf{b}}$ is closed relative to $\stratum{\tree}$.
\qed
\end{proof}

A critical observation used for bounding the centroidal configuration radius (\reflem{lem.ConfRadiusBound}) and the range of a trajectory of a substratum policy  (\refprop{prop.LocalPolicyExistenceUniqueness}) is:
\begin{lemma}\label{lem.CentroidDynamics}
(Relative Centroidal Dynamics)
Let $\vectbf{x} \in \stratum{\tree}$ and
\begin{align}
\vectbf{u} = \h{\purt, \vectbf{b}} \prl{\vectbf{x}}. 
\end{align}
Then, the centroidal dynamics of any cluster $I \in \visitedcluster$ \refeqn{eq.VisitedClusters} visited during recursive computation of $\h{\purt, \vectbf{b}}$ (\reftab{tab.LocalPolicy}), except the root $ \indexset$, satisfies\footnote{Here, for any $I \in \cluster{\tree}$ we use $\ctrdsym:\prl{\R^d}^I \rightarrow \R^d$ \refeqn{eq.centroid}.}
{
\begin{align}\label{eq.CentroidDynamicsGeneral}
\ctrd{\vectbf{u}|I} & = -\ctrd{\vectbf{x} \sqz{-} \vectbf{y}|I} + 2 \falpha{P}{\vectbf{x}, \vectbf{v}_P} \! \frac{\card{\complementLCL{I}{\tree}}}{\card{P}} \frac{\ctrdsep{I\!}{\vectbf{x}}}{\norm{\ctrdsep{I\!}{\vectbf{x}}}} \nonumber \\
& \hspace{30mm} + \ctrd{\vectbf{u}|P} + \ctrd{\vectbf{x} \sqz{-} \vectbf{y}|P},
\end{align}
}%
for some $\vectbf{v}_P \in \prl{\R^d}^{\indexset}$ associated with parent cluster $P = \parentCL{I,\tree}$; whereas we have for the root cluster $\indexset$ 
\begin{align}\label{eq.CentroidDynamicsRoot}
\ctrd{\vectbf{u}|\indexset} = -\ctrd{\vectbf{x} - \vectbf{y}|\indexset}.
\end{align}
\end{lemma}
\begin{proof}
See \refapp{app.CentroidDynamics}. \qed
\end{proof}

%If the clusters are well separated, we will observe navigation under the attractive field which motivates the following result.

\begin{lemma} \label{lem.ConfRadiusBound}
(Upper Bound on Configuration Radius)
Let  $\vectbf{x}^t$ denote a trajectory of $\h{\purt, \vectbf{b}}$ (\reftab{tab.LocalPolicy}) starting at any initial $\vectbf{x}^0 \in \domain\prl{\purt, \vectbf{b}}$ \refeqn{eq.LocalPolicyDomain} for $t \geq 0$.

Then, the centroidal configuration radius, $\radiusCL{}\prl{\vectbf{x}^t|\indexset}$ \refeqn{eq.ConfRadius},  is  bounded above for all $t \geq 0$ by a certain finite value, $R\prl{\vectbf{x}^{0}, \vectbf{y}}$, depending on $\vectbf{x}^{0}$ and $\vectbf{y}$, i.e.
\begin{align}
\radiusCL{}\prl{\vectbf{x}^t|\indexset} \leq R\prl{\vectbf{x}^0, \vectbf{y}} < \infty, \quad \forall t \geq 0.
\end{align} 
\end{lemma}
\begin{proof}
See \refapp{app.ConfRadiusBound}. \qed
\end{proof}

\begin{lemma}\label{lem.ClusterPartition}
%Let $\purt$ be a partition of $\indexset$ compatible with $\tree$, i.e. $\purt \subset \cluster{\tree}$. 
If $\purt$ is not the trivial partition, i.e. $\card{\purt}> 1$, then there always exists a cluster $I \in \cluster{\tree}$ such that $\childCL{I,\tree} \subset \purt$.
\end{lemma}
\begin{proof}
Define the depth of cluster $I \in \cluster{\tree}$ in $\tree$ to be the number of its ancestors, $\card{\ancestorCL{I,\tree}}$.

Let $K \in \purt$ be a cluster in $\purt$ with the maximal depth, i.e.
\begin{align}\label{eq.ClusterPartition}
\card{\ancestorCL{K,\tree}} = \argmax_{D \in \purt} \card{\ancestorCL{D,\tree}}.
\end{align}
Then, we now show that $\complementLCL{K}{\tree}$ is also in $\purt$, and so $I = \parentCL{K,\tree}$  satisfies the lemma. 

Proof by a contradiction. 
Suppose  that $\complementLCL{K}{\tree}$ is not in $\purt$. 
Since $\purt$  is a partition of $\indexset$ compatible with $\tree$, then some descendant $D \in \descendantCL{\complementLCL{K}{\tree},\tree}$ is in $\purt$.
Note that $\card{\ancestorCL{K,\tree}} = \card{\ancestorCL{\complementLCL{K}{\tree}}} < \card{\ancestorCL{D,\tree}}$, which contradicts \refeqn{eq.ClusterPartition}. \qed
\end{proof}

\begin{lemma} \label{lem.FinerCoarserResolution}
Let $\purt$ and $\purt'$ be two distinct partitions of $\indexset$ compatible with $\tree$, i.e. $\purt \neq \purt' \subset \cluster{\tree}$. 
Then, at least one of the followings always holds
\begin{enumerate}[label=(\roman*)]
\item \label{lem.FinerResolution} ($\purt$ Partially Refines $\purt'$) There exists a cluster $K' \in \purt'$ with a nontrivial partition $\mathcal{K}'$ such that $\mathcal{K}' \subset \purt$.
\item \label{lem.CoarserResolution}($\purt'$ Partially Refines $\purt$) There exists a cluster $K \in \purt$ with a nontrivial partition $\mathcal{K}$ such that $\mathcal{K} \subset \purt'$.
\end{enumerate}
\end{lemma}
\begin{proof}
For any $j \in \indexset$, let $\purt\prl{j}$ denote the unique element of $\purt$ containing $j$.

Since $\purt \neq \purt'$, let $K' \in \purt' \setminus \purt$ be an unshared cluster.
Since both $\purt$ and $\purt'$ are  partitions of $\indexset$ compatible with $\tree$,  we have  either $\purt(k') \subsetneq K'$ or $\purt(k') \supsetneq K'$ for all $k' \in K'$. 
If $\purt(k') \subsetneq K'$ for all $k' \in K'$, then $\mathcal{K}' = \crl{\purt(k') \big| k' \in K'}$ defines a partition of $K'$ and we obtain \reflem{lem.FinerCoarserResolution}.\ref{lem.FinerResolution}. 
Otherwise, by symmetry, we have \reflem{lem.FinerCoarserResolution}.\ref{lem.CoarserResolution}.
Thus, the lemma follows. \qed
\end{proof}

%%%%%%%%%%%%%%%%%%%%%%%%%%%%%%%%%%%%%%%%%%
%%%%%%%%%%%%%%%%%%%%%%%%%%%%%%%%%%%%%%%%%%
\section{Proofs}
\label{app.Proofs}
%%%%%%%%%%%%%%%%%%%%%%%%%%%%%%%%%%%%%%%%%%
%%%%%%%%%%%%%%%%%%%%%%%%%%%%%%%%%%%%%%%%%%

%%%%%%%%%%%%%%%%%%%%%%%%%%%%%%%%%%%%%%%%%%
%%%%%%%%%%%%%%%%%%%%%%%%%%%%%%%%%%%%%%%%%%
\subsection{Proof of \refthm{thm.Portal}}
\label{app.Portal}
%%%%%%%%%%%%%%%%%%%%%%%%%%%%%%%%%%%%%%%%%%
%%%%%%%%%%%%%%%%%%%%%%%%%%%%%%%%%%%%%%%%%%

\begin{proof}
To prove the first part of the result, we shall consider $\portalconf{\treeA,\treeB}$ as a mapping from $\stratum{\treeA}$ to $\prl{\R^d}^{\indexset}$ and verify that $\portalconf{\treeA, \treeB}\prl{\stratum{\treeA}} \subseteq \portal\prl{\treeA,\treeB}$.

By definition, the restriction of $\portalconf{\treeA,\treeB}$ to  $\portal\prl{\treeA,\treeB}$ is the identity map on $\portal\prl{\treeA,\treeB}$.
Hence, we only need to show that $\portalconf{\treeA, \treeB}\prl{\stratum{\treeA} \setminus \portal\prl{\treeA,\treeB}} \subseteq \portal\prl{\treeA,\treeB}$.

Let $\vectbf{\stateB} = \portalcenter{}\prl{\vectbf{\stateA}} $ and $\vectbf{\stateC} = \portalscale{}  \prl{\vectbf{\stateB}}$  be   intermediate configurations during the portal transformation of a configuration $\vectbf{\stateA} \in \stratum{\treeA} \setminus \portal\prl{\treeA,\treeB}$ into $\vectbf{\stateD}= \portalmerge{}\prl{\vectbf{\stateC}}=\portalconf{}\prl{\vectbf{\stateA}}$.

First, recall that rigid transformations and scaling of partial configurations preserve their clustering structure \cite{arslanEtAl_Allerton2012}. 
Hence, the common subtrees  of $\treeA$ and $\treeB$ rooted at $A$, $B$ and $C$ are preserved after each transformation by $\portalcenter{}$ \refeqn{eq.portalcenter}, $\portalscale{}$ \refeqn{eq.portalscale} and $\portalmerge{}$ \refeqn{eq.portalmerge}. 

Second, each partial configuration of the symmetric configuration $\vectbf{\stateB} \in \mathtt{Sym}\prl{\treeA,\treeB}$ associated with  $\prl{\treeA,\treeB}$ is properly translated by $\portalscale{}$ \refeqn{eq.portalscale} so that each of them lies in the corresponding consensus ball, i.e. $\radiusCL{}\prl{\vectbf{\stateC}|Q} < \radiusDD{Q}\prl{\vectbf{\stateC}}$ for all $Q \in \prl{A,B,C}$. 
Hence, the partial configuration $\vectbf{\stateC}|P$ supports both of the subtrees of $\treeA$ and $\treeB$ rooted at $P$. 

Finally, if $P = \indexset$, then the result simply follows since $\vectbf{\stateC} = \vectbf{\stateD} \in \portal\prl{\treeA, \treeB}$. 
Otherwise, for every $I \in \crl{P} \cup \ancestorCL{P,\treeA}$, $\portalmerge{}$ \refeqn{eq.portalmerge} iteratively separates the common complementary clusters $I $ and $\complementLCL{I}{\treeA}$ of $\treeA$ and $\treeB$, in a bottom up fashion starting at cluster $P$, to support the subtrees of $\treeA$ and $\treeB$ rooted at $\parentCL{I,\treeA}$.
Note that in the base case $\vectbf{\stateC}$ supports both of the subtrees of $\treeA$ and $\treeB$ rooted at $P$ and $\complementLCL{P}{\treeA}$; and at the termination at cluster $\indexset$, $\vectbf{\stateD}$ supports both trees $\treeA$ and $\treeB$, i.e. $\vectbf{w} \in \portal\prl{\treeA,\treeB}$. 

We now proceed with the computational properties of $\portalconf{\treeA,\treeB}$.
As stated in the proof of \refprop{prop.HierInvNavProperty}, the inclusion test of a configuration for being in a hierarchical stratum can be computed in $\bigO{\card{\indexset}^2}$ time, from which one conclude that the inclusion test for being in $\portal\prl{\treeA, \treeB}$ can also be computed in  $\bigO{\card{\indexset}^2}$ time.
If the given configuration is not a portal configuration, then the computation of $\portalconf{\treeA,\treeB}$\refeqn{eq.portalconf} requires cluster centroids of $\treeA$, which can be computed in linear, $\bigO{\card{\indexset}}$, time as described in the proof of \refprop{prop.HierInvNavProperty}.  
Given cluster centroids, one can compute $\portalcenter{}$ \refeqn{eq.portalcenter} and $\portalscale{}$ \refeqn{eq.portalscale} in linear, $\bigO{\card{\indexset}}$, time since the Napoleon transformation $\NT$ of an arbitrary triangle can be computed in constant, $\bigO{1}$, time \cite{ArslanEtAl_Techreport2013}.        
Finally, given the cluster centroids, each iteration of $\portalmerge{}$ \refeqn{eq.portalmerge} can be computed in linear $\bigO{\card{\indexset}}$ time; and so all iterations of $\portalmerge{}$ can be performed in $\bigO{\card{\indexset}^2}$ time since it may require at most $\card{\indexset}$ iterations.    
Thus, the result follows. \qed
\end{proof}

%%%%%%%%%%%%%%%%%%%%%%%%%%%%%%%%%%%%%%%%%%
%%%%%%%%%%%%%%%%%%%%%%%%%%%%%%%%%%%%%%%%%%
\subsection{Proof of \refprop{prop.HierInvNavProperty}}
\label{app.HierInvNavProperty}
%%%%%%%%%%%%%%%%%%%%%%%%%%%%%%%%%%%%%%%%%%
%%%%%%%%%%%%%%%%%%%%%%%%%%%%%%%%%%%%%%%%%%

\begin{proof}
Recall from \refeqn{eq.SetA} that for any singleton cluster $I \in \cluster{\tree}$ we have $\setA\prl{I} = \stratum{\tree}$.
Hence,  for any given $\vectbf{x} \in \stratum{\tree}$ the base condition $\vectbf{x} \in \setA\prl{I}$  (\reftab{tab.HierInvNav}.2) always holds at any singleton cluster $I \in \cluster{\tree}$.
Moreover, the cardinality of any cluster passed as an argument in a recursive step of the evaluation must decrease relative to the calling cluster size (\reftab{tab.HierInvNav}.7). 
Therefore, the recursion in \reftab{tab.HierInvNav} terminates, in the worst case, after visiting all clusters of $\tree$ only once.   
%Accordingly, observe that the line 4 and below in \reftab{tab.HierInvNav} only computed for nonsingleton clusters which always have non-empty children clusters. 

Since all vector fields ($\fA$ \refeqn{eq.AttractiveField}, $\fH$ \refeqn{eq.SplitPreservingField} and $\fS$ \refeqn{eq.SeparationField}) used in \reftab{tab.HierInvNav} are well defined over the entirety of their domain $\stratum{\tree}$ with codomain  $\prl{\R^d}^{\indexset}$,  the  recursion in \reftab{tab.HierInvNav} results in a true function, $\f: \stratum{\tree} \rightarrow \prl{\R^d}^{\indexset}$, with well defined evaluation for each configuration $\vectbf{x} \in \stratum{\tree}$.

%We now show that $\f$ can be computed in $\bigO{\card{\indexset}^2}$ time for any $\vectbf{x} \in \stratum{\tree}$, which shall complete the proof.

We now assess the computational complexity of the recursion in \reftab{tab.HierInvNav}.
Centroids of clusters of $\tree$ can be computed  all at once in $\bigO{\card{\indexset}}$ time using the post-order traversal of $\tree$ and the following recursive relation of cluster centroids: for any disjoint $A, B \subset \indexset$,
\begin{align}
\ctrd{\vectbf{x}|A \cup B} = \scalebox{1}{$\frac{\card{A}}{\card{A} + \card{B}} $}\ctrd{\vectbf{x}|A} + \scalebox{1}{$\frac{\card{B}}{\card{A}+\card{B}}$} \ctrd{\vectbf{x}|B}.
\end{align}  

Given cluster centroids, $\sepmag{i, I, \tree}{\vectbf{x}}$ \refeqn{eq.SepMag} can be computed in constant, $\bigO{1}$, time for any $i \in I$ and $ I \in \cluster{\tree}$.  
Hence, since  $\card{\cluster{\tree}} = 2 \card{\indexset} - 1$ for any $\tree \in \bintreetopspace_{\indexset}$ and $I = \crl{k \, \big| \, k \in K, K \in \childCL{I,\tree}}$ for any nonsingleton cluster $I \in \cluster{\tree}$, we conclude:
\begin{itemize}
\item The inclusion test for being in $\stratum{\tree}$ \refeqn{eq.ClosedStrata} can be computed in $\bigO{\card{\indexset}^2}$. 
\item Given $\vectbf{x} \in \stratum{\tree}$, the inclusion test for being in $\setH\prl{I}$ \refeqn{eq.SetR} for any cluster $I \in \cluster{\tree}$  can be computed in $\bigO{\card{\indexset}}$ time; and  the recursion in \reftab{tab.HierInvNav} requires  such inclusion tests at most only once for all clusters of $\tree$  which can be computed in  $\bigO{\card{\indexset}^2}$ time. 
\item The vector fields $\fA$ \refeqn{eq.AttractiveField}, $\fH$ \refeqn{eq.SplitPreservingField} and $\fS$ \refeqn{eq.SeparationField} at any cluster $I \in \cluster{\tree}$ can be computed in $\bigO{\card{\indexset}}$ time; and, once again, the recursion in \reftab{tab.HierInvNav} requires such computation at most at every cluster of $\tree$ all of which can be  performed in $\bigO{\card{\indexset}^2}$ time.
\end{itemize}

Finally, to conclude that $\f$ is computable in $\bigO{\card{\indexset}^2}$ time, we show that  the inclusion test for  being in  $\setA\prl{I}$ \refeqn{eq.SetA} for all clusters $I \in \cluster{\tree}$ can be efficiently computed in $\bigO{\card{\indexset}^2}$ time as follows.
Given cluster centroids,  $\lie_{\overrightarrow{\vectbf{y}}}\vectprod{\prl{\vect{x}_k \sqz{-} \ctrdmid{K\!}{\vectbf{x}}\!}\!}{\ctrdsep{K\!}{\vectbf{x}}}$ \refeqn{eq.SetACond2} can be computed in constant, $\bigO{1}$, time for any $k \in K$, $K \in \cluster{\tree}$; and, likewise, $\lie_{\overrightarrow{\vectbf{y}}} \frac{1}{2} \norm{\vect{x}_i - \vect{x}_j}^2$ \refeqn{eq.SetACond1} can be computed in constant $\bigO{1}$ time for any given pair $i \neq j \in \indexset$.
Further, using \refeqn{eq.SetA} and hierarchical relations of clusters, observe the following recursive relation of $\setA\prl{I}$: for any nonsingleton $I \in \cluster{\tree}$ and $\crl{I_L, I_R} = \childCL{I,\tree}$,
\begin{align}\label{eq.SetARecursion}
\setA\prl{I} = \setA\prl{I_L} \cap \setA\prl{I_R} \cap\setAhat\prl{I_L,I_R}, 
\end{align}
subject to the base condition  $\setA\prl{I} = \stratum{\tree}$ for any singleton cluster $I \in \cluster{\tree}$, where
{\small
\begin{align}\label{eq.SetAHat} 
\hspace{-2mm}\setAhat\prl{I_L, I_R} \!\sqz{\ldf}\! \left\{\Big. \right. \!\! \vectbf{x} \sqz{\in} \stratum{\tree} \! \Big | \,\!  & \lie_{\overrightarrow{\vectbf{y}}}\!  \scalebox{1}{$\frac{1}{2}$}\!\norm{\vect{x}_i \sqz{-} \vect{x}_j}^2 \sqz{\geq} \prl{r_i \sqz{+} r_j\!}^{2}\!\!\!, \, \forall i \sqz{\in}\! I_L, j \sqz{\in}\! I_R , \nonumber \\
& \hspace{-21mm} \lie_{\overrightarrow{\vectbf{y}}} \vectprod{\prl{\vect{x}_k \sqz{-} \ctrdmid{K\!}{\vectbf{x}}\!}\!}{\!\ctrdsep{K\!}{\vectbf{x}}} \sqz{\geq} 0, \forall k \sqz{\in} K\!, K  \sqz{\in} \crl{I_L, I_R} \!\!  \!\left . \Big. \right \}  \!.  \!  \! \! 
\end{align}
}%
Note that, given $\vectbf{x} \in \stratum{\tree}$, the inclusion test for being in $\setAhat\prl{I_L, I_R}$  for the children $\crl{I_L, I_R} = \childCL{I,\tree}$ of any nonsingleton cluster $I \in \cluster{\tree}$ can be computed in $\bigO{\card{I_L}\card{I_R} + \card{I_L} + \card{I_R}}$ time.

Hence, given $\vectbf{x} \in \stratum{\tree}$, the inclusion test for being in $\setA$ \refeqn{eq.SetA}
for any cluster $I \in \cluster{\tree}$ and all its descendants in $\descendantCL{I,\tree}$ can be  computed at once  in $\bigO{\card{I}^2}$ time using the post-order traversal of the subtree of $\tree$ rooted at $I$ and the recursive formulation \refeqn{eq.SetARecursion} of $\setA\prl{I}$. 
This can be verified as follows.
First, observe that the cluster set $\cluster{\tree}$ of $\tree$ can be recursively defined as:
\begin{subequations}\label{eq.ClusterSetBottomUpRecursion}
\begin{align}
&\text{$\bullet$ (Base Step) $\crl{j} \in \cluster{\tree}$ for all $j \in \indexset$.}\\
&\text{$\bullet$ (Recursion) If $I, \complementLCL{I}{\tree} \sqz{\in} \cluster{\tree}\sqz{\setminus} \crl{\indexset}$, then $\parentCL{I,\tree} \sqz{\in} \cluster{\tree}$.} \!\!\!  
\end{align} 
\end{subequations}
Accordingly, we provide a proof by structural induction \cite{rosen_DiscreteMathematics_2011}. 
For any $I \in \cluster{\tree}$:
\begin{itemize}[leftmargin=*]
\item (Base Case) If  $I$ is singleton, then the result simply holds since  any singleton cluster $I \in \cluster{\tree}$ has no descendant in $\tree$ and satisfies $\setA\prl{I} = \stratum{\tree}$.
\item (Induction) Otherwise ($\card{I} \geq 2$), let $\crl{I_L, I_R} = \childCL{I,\tree}$.   
(Induction hypothesis) Suppose that the inclusion test for being in $\setA$ for any child $K \in \childCL{I,\tree}$ and all its descendant in $\descendantCL{K,\tree}$ is computable in $\bigO{\card{K}^2}$.
Then, by the recursion in \refeqn{eq.SetARecursion},  the inclusion test for being in $\setA$ for cluster $I$ and all its descendants in $\descendantCL{I,\tree}$ only requires the extra test for being in $\setAhat\prl{I_L, I_R}$ for the children $\crl{I_L, I_R} = \childCL{I,\tree}$ in addition to the inclusion test for every child $K \in \childCL{I,\tree}$ and its descendants in $\descendantCL{K,\tree}$. 
Hence, the total computation cost for  cluster $I$ and its descendants in $\tree$ is $\bigO{\card{I_L}^2} + \bigO{\card{I_R}^2} + \bigO{\Big.\card{I_L}\card{I_R} + \card{I_R} + \card{I_R}} = \bigO{\card{I}^2}$.
\end{itemize}

\noindent Therefore, since $ \cluster{\tree} = \crl{\indexset} \cup \descendantCL{\indexset, \tree}$, given $\vectbf{x} \in \stratum{\tree}$ the inclusion test for being in  $\setA\prl{I}$ for all clusters $I \in \cluster{\tree}$ can be  computed at once  in $\bigO{\card{\indexset}^2}$ time, and this completes the proof. \qed
\end{proof}

%%%%%%%%%%%%%%%%%%%%%%%%%%%%%%%%%%%%%%%%%%%%%%%%
%%%%%%%%%%%%%%%%%%%%%%%%%%%%%%%%%%%%%%%%%%%%%%%%
\subsection{Proof of \refprop{prop.LocalPolicyProperty}}
\label{app.LocalPolicyProperty}
%%%%%%%%%%%%%%%%%%%%%%%%%%%%%%%%%%%%%%%%%%%%%%%%
%%%%%%%%%%%%%%%%%%%%%%%%%%%%%%%%%%%%%%%%%%%%%%%%

\begin{proof}
To demonstrate how the recursion in \reftab{tab.LocalPolicy} recursively composes continuous vector fields,  we shall recast $\fA$ \refeqn{eq.AttractiveField}, $\fH$ \refeqn{eq.SplitPreservingField}, $\fS$ \refeqn{eq.SeparationField} and the recursion $\hhat{\purt, \vectbf{b}}$ (\reftab{tab.LocalPolicy}) as follows: for any cluster $I \in \visitedcluster$ \refeqn{eq.VisitedClusters} visited during recursive computation of $\h{\purt, \vectbf{b}}$,
\begin{align}
&
\begin{array}{c@{\hspace{1mm}}c@{\hspace{0mm}}c@{\hspace{1mm}}c}
\fA^{I} :& \stratum{\tree} \sqz{\times} \prl{\R^d}^{\indexset}  &\rightarrow & \stratum{\tree} \sqz{\times} \prl{\R^d}^{\indexset}
\\
&\prl{\vectbf{x}, \vectbf{u}}  &\mapsto & \prl{\vectbf{x}, \fA\prl{\vectbf{x}, \vectbf{u}, I}}
\end{array} 
\\
&
\begin{array}{c@{\hspace{1mm}}c@{\hspace{0mm}}c@{\hspace{1mm}}c}
\fH^{I} :& \stratum{\tree} \sqz{\times} \prl{\R^d}^{\indexset}  &\rightarrow & \stratum{\tree} \sqz{\times} \prl{\R^d}^{\indexset}
\\
&\prl{\vectbf{x}, \vectbf{u}}  &\mapsto & \prl{\vectbf{x}, \fH\prl{\vectbf{x}, \vectbf{u}, I}}
\end{array} 
\\
&
\begin{array}{c@{\hspace{1mm}}c@{\hspace{0mm}}c@{\hspace{1mm}}c}
\fS^{I} :& \stratum{\tree} \sqz{\times} \prl{\R^d}^{\indexset}  &\rightarrow & \stratum{\tree} \sqz{\times} \prl{\R^d}^{\indexset}
\\
&\prl{\vectbf{x}, \vectbf{u}}  &\mapsto & \prl{\vectbf{x}, \fS\prl{\vectbf{x}, \vectbf{u}, I}}
\end{array} 
\\
&\hspace{-2.5mm}
\begin{array}{c@{\hspace{1mm}}c@{\hspace{0mm}}c@{\hspace{1mm}}c}
\hhat{\purt, \vectbf{b}}^{I} :& \stratum{\tree} \sqz{\times} \prl{\R^d}^{\indexset}  &\rightarrow & \stratum{\tree} \sqz{\times} \prl{\R^d}^{\indexset}
\\
&\prl{\vectbf{x}, \vectbf{u}}  &\mapsto & \prl{\vectbf{x},\hhat{\purt, \vectbf{b}}\prl{\vectbf{x}, \vectbf{u}, I}}
\end{array} 
\end{align}
Note that, by definition, $\fA^{I}\prl{\vectbf{x}, \vectbf{u}}$ is smooth  in both $\vectbf{x}$ and $\vectbf{u}$, and  $\fH^{I}\prl{\vectbf{x}, \vectbf{u}}$ and $\fS^{I}\prl{\vectbf{x}, \vectbf{u}}$ are continuous and piecewise smooth functions of $\vectbf{x}$ and $\vectbf{u}$ since functions defined by the maximum of a finite collection of smooth functions are continuous and piecewise smooth, and the product of continuous and piecewise smooth functions are also continuous and piecewise smooth \cite{chaney_NA1990}.   
 
We now show that,  for any $I \in \visitedcluster$, $\hhat{\purt, \vectbf{b}}^ {I}\prl{\vectbf{x}, \vectbf{u}}$ is continuous and piecewise smooth in $\vectbf{x}$ and $\vectbf{u}$.
First, observe from \reflem{lem.ClusterPartition} that the set $\visitedcluster$ \refeqn{eq.VisitedClusters} can be recursively defined as
\begin{subequations}\label{eq.VisitedClustersRecursion}
\begin{align}
&\hspace{-2mm}\bullet\text{(Base Step) $I \in \visitedcluster$ for all $I \in \purt$.}\\
&\hspace{-2mm}\bullet \text{(Recursion) If $I, \complementLCL{I}{\tree} \sqz{\in} \visitedcluster \sqz{\setminus} \crl{\indexset}$, then $\parentCL{I,\tree} \sqz{\in} \visitedcluster$.} \!\!\!\!
\end{align}
\end{subequations}   
Accordingly, we provide a proof by structural induction \cite{rosen_DiscreteMathematics_2011}. 
For any cluster $I \in \visitedcluster$:
\begin{itemize}[leftmargin = *]
\item (Base Case) If  $I \in \purt$, then we have
\begin{align}
\hhat{\purt, \vectbf{b}}^{I} \prl{\vectbf{x}, \vectbf{u}} = 
\left \{
\begin{array}{@{}l@{}l}
\fA^{I}\prl{\vectbf{x}, \vectbf{u}} & \text{, if } \vect{b}_I = +1, \\
\fS^{I}\prl{\vectbf{x}, \vectbf{u}} & \text{, if } \vect{b}_I = -1,
\end{array}
\right.
\end{align}
which is continuous and piecewise smooth in $\vectbf{x}$ and $\vectbf{u}$.
\item (Induction) Else, we have $\card{I} \geq 2$ and so let $\crl{I_L, I_R} = \childCL{I,\tree}$. 
(Induction hypothesis) Suppose $\hhat{\purt, \vectbf{b}}^{I_L}\prl{\vectbf{x}, \vectbf{y}}$ and $\hhat{\purt, \vectbf{b}}^{I_R}\prl{\vectbf{x}, \vectbf{y}}$ are continuous and piecewise smooth. 
Then, one can verify from \reftab{tab.LocalPolicy} that
\begin{align}
\hhat{\purt, \vectbf{b}}^{I} \prl{\vectbf{x}, \vectbf{u}} = \prl{\fH^{I} \circ \hhat{\purt, \vectbf{b}}^{I_R} \circ \hhat{\purt, \vectbf{b}}^{I_L}}\prl{\vectbf{x}, \vectbf{u}}.
\end{align}
Hence, $\hhat{\purt, \vectbf{b}}^{I}$ is a composition of continuous and piecewise smooth functions, hence it must remain so as well \cite{chaney_NA1990}.
\end{itemize}
Thus,  the result follows since $\prl{\vectbf{x}, \h{\purt, \vectbf{b}}\prl{\vectbf{x}}} = \hhat{\purt, \vectbf{b}}^{\indexset}\prl{\vectbf{x}, \vectbf{0}}$. \qed
\end{proof}

%%%%%%%%%%%%%%%%%%%%%%%%%%%%%%%%%%%%%%%%%%%%%%%%%%
%%%%%%%%%%%%%%%%%%%%%%%%%%%%%%%%%%%%%%%%%%%%%%%%%%
\subsection{Proof of \refprop{prop.DomainInclusion}}
\label{app.DomainInclusion}
%%%%%%%%%%%%%%%%%%%%%%%%%%%%%%%%%%%%%%%%%%%%%%%%%%
%%%%%%%%%%%%%%%%%%%%%%%%%%%%%%%%%%%%%%%%%%%%%%%%%%
%
\begin{proof}
Since the recursion in \reftab{tab.PolicySelection} uses only clusters of $\tree$ and guarantees, in \reftab{tab.PolicySelection}.4), \reftab{tab.PolicySelection}.7) and \reftab{tab.PolicySelection}.13), that the dimension of $\vectbf{b}$ is equal to the cardinality of $\purt$, the output $\prl{\purt, \vectbf{b}} = \p\prl{\vectbf{x}}$ associated with any configuration $\vectbf{x} \in \stratum{\tree}$  always satisfies that $\purt \subset \cluster{\tree}$ and $\vectbf{b} \in \crl{-1, +1}^{\purt}$.

To prove that $\purt$ is a partition of $\indexset$, we shall show that, for any $\vectbf{x} \in \stratum{\tree}$ and $I \in \cluster{\tree}$, $(\hat{\mathcal{I}}, \hat{\vectbf{b}}) = \phat\prl{\vectbf{x}, I}$ yields a partition $\hat{\mathcal{I}}$ of $I$. 
Based on the recursive definition \refeqn{eq.ClusterSetBottomUpRecursion} of $\cluster{\tree}$, we now provide a proof by  structural induction.
For any $\vectbf{x} \in \stratum{\tree}$ and $I \in \cluster{\tree}$ let $(\hat{\mathcal{I}}, \hat{\vectbf{b}}) = \phat\prl{\vectbf{x}, I}$, then:
\begin{itemize}[leftmargin=6mm]
\item (Base Case) If $I$ is singleton, then $\setA\prl{I} = \stratum{\tree}$ and the base condition in \reftab{tab.PolicySelection}.2) holds.
Hence, we have $\hat{\mathcal{I}} = \crl{I}$,  the trivial partition of $I$, and the result follows.
\item (Induction) Otherwise ($\card{I}\geq 2$), we have two possibilities.
\begin{itemize}[leftmargin=4mm]
\item If $I$ satisfies any base condition in \reftab{tab.PolicySelection}.2) and in \reftab{tab.PolicySelection}.5), i.e. $\vectbf{x} \in \setA\prl{I} \cup (\stratum{\tree}\setminus \setH\prl{I})$, then we have $\hat{\mathcal{I}} = \crl{I}$ and the result directly follows.
\item Else(the recursion condition in \reftab{tab.PolicySelection}.8)-14) holds), since $\card{I} \geq 2$, let $\crl{I_L, I_R} = \childCL{I,\tree}$ and $(\hat{\mathcal{I}}_L, \hat{\vectbf{b}}_L) = \phat\prl{\vectbf{x}, I_L}$ and $(\hat{\mathcal{I}}_R, \hat{\vectbf{b}}_R) = \phat\prl{\vectbf{x}, I_R}$.
(Induction Hypothesis) Suppose that $\hat{\mathcal{I}}_L$ and $\hat{\mathcal{I}}_R$ are partitions of $I_L$ and $I_R$, respectively.
Then, since $\hat{\mathcal{I}} = \hat{\mathcal{I}}_L \cup \hat{\mathcal{I}}_R$ (\reftab{tab.PolicySelection}.12)) and $\childCL{I,\tree}$ is a bipartition of $I$, we observe that $\hat{\mathcal{I}}$ is a partition of $I$.
\end{itemize}
\end{itemize} 
Hence, since $\prl{\purt, \vectbf{b}} =\p\prl{\vectbf{x}} = \phat\prl{\vectbf{x}, \indexset}$, the recursion in \reftab{tab.PolicySelection} terminates with a partition $\purt$ of $\indexset$.
Thus, since the policy selection algorithm is deterministic,  $\p$ is a well-defined function from $\stratum{\tree}$ to $\subpolicyset{\indexset}{\tree}$ \refeqn{eq.SubPolicySet}.

Finally, we shall show that $\prl{\purt, \vectbf{b}} =\p\prl{\vectbf{x}}$  is the index of a local control policy whose domain $\domain\prl{\purt, \vectbf{b}}$ contains $\vectbf{x}$, i.e. $\vectbf{x} \in \prl{\domain \circ \p}\prl{\vectbf{x}}$.
Using the base conditions in \reftab{tab.PolicySelection}.2)-7) one can verify that for any $I \in \purt$, if $\vect{b}_I = +1$, then $\vectbf{x} \in \setA\prl{I}$; and if $\vectbf{b}_I = - 1$,  then $\vectbf{x} \in \stratum{\tree}\setminus \prl{\setA\prl{I} \cup \setH\prl{I}} \subset \stratum{\tree}$.   
Hence, the base conditions guarantee that $\vectbf{x} \in \setB\prl{I, \vect{b}_I}$ \refeqn{eq.SetB} for any $I \in \purt$.
Observe that during the recursive computation of $\p$ in \reftab{tab.PolicySelection} to reach any cluster $I \in \purt$ satisfying a base condition every ancestor $K \in \ancestorCL{I,\tree}$ of $I$ must have been recursively visited.
A recursion (\reftab{tab.PolicySelection}.8)-14)) at any ancestor $K \in \ancestorCL{I,\tree}$ of $I$ in $\tree$ implies that  $\vectbf{x} \in \setH\prl{K} \setminus \setA\prl{K} \subset \setH\prl{K}$. 
Thus, by definition \refeqn{eq.LocalPolicyDomain},  we have $\vectbf{x} \in \domain\prl{\purt, \vectbf{b}}$ and  the result follows.  \qed
\end{proof}

%%%%%%%%%%%%%%%%%%%%%%%%%%%%%%%%%%%%%%%%%%%%%%
%%%%%%%%%%%%%%%%%%%%%%%%%%%%%%%%%%%%%%%%%%%%%%
\subsection{Proof of \refprop{prop.SystemEquivalence}}
\label{app.SystemEquivalence}
%%%%%%%%%%%%%%%%%%%%%%%%%%%%%%%%%%%%%%%%%%%%%%
%%%%%%%%%%%%%%%%%%%%%%%%%%%%%%%%%%%%%%%%%%%%%%

\begin{proof}
For any given $\vectbf{x} \in \stratum{\tree}$, the recursions in \reftab{tab.HierInvNav} and \reftab{tab.PolicySelection} traverse the same clusters of $\tree$ in the same order since both recursions have identical base and recursion conditions.

Now observe that the tree traversal pattern used by the recursion in  \reftab{tab.LocalPolicy} is fixed for a given policy index $\prl{\purt, \vectbf{b}} \in \subpolicyset{\indexset}{\tree}$: a base condition is satisfied at any cluster $I \in \purt$, and to reach such cluster $I$ all its ancestors $\ancestorCL{I,\tree}$ must have been recursively visited starting from the root $\indexset$.    
Recall from the proof of \refprop{prop.DomainInclusion} that $\prl{\purt,\vectbf{b}} = \p\prl{\vectbf{x}}$ yields a partition $\purt$ of $\indexset$ such that a base condition in \reftab{tab.PolicySelection} holds for every block $I \in \purt$ and  all its ancestors in $\ancestorCL{I,\tree}$ are recursively visited.   
Hence, if the policy index is selected as $\prl{\purt,\vectbf{b}} = \p\prl{\vectbf{x}}$,  the recursion in \reftab{tab.LocalPolicy} computing $\h{\purt, \vectbf{b}}\prl{\vectbf{x}}$ always follows the tree traversal pattern used by the recursion in \reftab{tab.PolicySelection} computing $\p\prl{\vectbf{x}}$.

Thus, for a given configuration, all recursions in  \reftab{tab.HierInvNav}, \reftab{tab.LocalPolicy} and   \reftab{tab.PolicySelection} share a common tree traversal strategy. 

Let $\vectbf{x} \in \stratum{\tree}$ and  $\prl{\mathcal{\indexset, \vectbf{b}}} = \p\prl{\vectbf{x}}$, and  observe from \reftab{tab.PolicySelection} that for any $I \in \purt$, if $\vect{b}_I = +1$, then $\vectbf{x} \in \setA\prl{I}$; and if $\vectbf{b}_I = - 1$,  then $\vectbf{x} \in \stratum{\tree}\setminus \prl{\setA\prl{I} \cup \setH\prl{I}}$; and  $\vectbf{x} \in \setH\prl{K} \setminus \setA\prl{K}$ for all $K \in \ancestorCL{I,\tree}$.
Using this relation between policy indices and domains, one can conclude that the recursions in   \reftab{tab.HierInvNav} and \reftab{tab.LocalPolicy} use the same vector fields for the identical base and recursive steps.  
Thus, the result follows. \qed
\end{proof}

%%%%%%%%%%%%%%%%%%%%%%%%%%%%%%%%%%%%%%%%%%%%%%%%%%
%%%%%%%%%%%%%%%%%%%%%%%%%%%%%%%%%%%%%%%%%%%%%%%%%%
\subsection{Proof of \refprop{prop.PreparesGraph}}
\label{app.PreparesGraph}
%%%%%%%%%%%%%%%%%%%%%%%%%%%%%%%%%%%%%%%%%%%%%%%%%%
%%%%%%%%%%%%%%%%%%%%%%%%%%%%%%%%%%%%%%%%%%%%%%%%%%

\begin{proof}
According to \refdef{def.SubPreparesGraph}, any pair $\left (\! \!\big. \right. (\purt, \vectbf{b}), (\purt', \vectbf{b}') \! \!\left. \big. \right)$  of policy indices in $\subpreparesedgesetS$ satisfies at least one of Lemmas \ref{lem.PreparesRelationSingleton} - \ref{lem.PreparesRelationPositive}.
Hence, $\h{\purt, \vectbf{b}}$ prepares $\h{\purt', \vectbf{b}'}$ in finite time, and $\priority(\purt', \vectbf{b}') > \priority(\purt, \vectbf{b})$.
Thus, $\left (\! \!\big. \right. (\purt, \vectbf{b}), (\purt', \vectbf{b}') \! \!\left. \big. \right)$ is also an edge of the prepares graph $\subpreparesgraph$. 

Moreover, for any $\prl{\purt, \vectbf{b}} \neq \prl{\crl{\indexset}, +1}$, there always exists a  policy index $(\purt', \vectbf{b}') \neq \prl{\purt, \vectbf{b}}$ such that $\left (\! \!\big. \right. (\purt, \vectbf{b}), (\purt', \vectbf{b}') \! \!\left. \big. \right)$  is an edge of $\subpreparesgraphS$. 
This can be observed as follows.
Since  $\purt$ is compatible with $\tree$, i.e. $\purt \subset \cluster{\tree}$, if $\card{\purt} > 1$, then there exists a nonsingleton cluster $I \in \cluster{\tree}$ such that $\childCL{I,\tree} \subset \purt$ (\reflem{lem.ClusterPartition}).
Hence,  at least one of the following always holds:
\begin{enumerate}[label=(\alph*)]
\item  There exists a cluster $I \in \purt$ with $\vect{b}_I = -1$.
\item  There exists a cluster $I \in \cluster{\tree}$ such that $\childCL{I,\tree}\subset \purt$ and $\vect{b}_{K} = +1$ for all $K \in \childCL{I,\tree}$.
\end{enumerate}
And $(\purt', \vectbf{b}')$ can be selected accordingly to satisfy one of the connectivity conditions of $\subpreparesgraphS$ (\refdef{def.SubPreparesGraph}.\ref{def.subprepgraph1}-\ref{def.subprepgraph3}).

Since every policy index $\prl{\purt, \vectbf{b}} \neq \prl{\crl{\indexset}, +1}$ has an adjacent policy index $(\purt', \vectbf{b}')$ in $\subpreparesgraphS$ and $\priority(\purt', \vectbf{b}') > \priority(\purt, \vectbf{b})$,
$\subpreparesgraphS$ has no cycle and all of its nodes connected to the goal policy index $\prl{\crl{\indexset}, +1}$ through directed paths along which $\priority$ is strictly increasing.
Note that the goal policy index has the highest $\priority$ value which is $\card{\indexset}^2$ \refeqn{eq.PriorityRange}.
Further, since $\priority$ \refeqn{eq.Priority} is integer valued function whose range \refeqn{eq.PriorityRange} is $ [-\!\card{\indexset}^2\!, \card{\indexset}^2]$, the length of a  directed path in $\subpreparesgraphS$ is bounded above by $\bigO{\card{\indexset}^2}$ hops, and the result follows. \qed
\end{proof}

%%%%%%%%%%%%%%%%%%%%%%%%%%%%%%%%%%%%%%%%%%%%%%%%%
%%%%%%%%%%%%%%%%%%%%%%%%%%%%%%%%%%%%%%%%%%%%%%%%%
\subsection{Proof of \refprop{prop.PolicyPriority}}
\label{app.PolicyPriority}
%%%%%%%%%%%%%%%%%%%%%%%%%%%%%%%%%%%%%%%%%%%%%%%%%
%%%%%%%%%%%%%%%%%%%%%%%%%%%%%%%%%%%%%%%%%%%%%%%%%

%
\begin{proof}
If there is only one local controller whose domain contains $\vectbf{x}$, then the result follows from \refprop{prop.DomainInclusion}.

Otherwise, we shall provide a proof by contradiction. 
Let $(\purt, \vectbf{b}) = \p\prl{\vectbf{x}}$, and
$(\purt', \vectbf{b}')$ be the index of a local controller whose domain $\domain(\purt', \vectbf{b}')$ \refeqn{eq.LocalPolicyDomain} contains $\vectbf{x}$, and $(\purt', \vectbf{b}') \neq (\purt, \vectbf{b})$.  
Suppose that the local controller $\h{\purt', \vectbf{b}'}$ has the maximum priority among all local controllers whose domains  contain $\vectbf{x}$. 
We shall show below that there always exists another local controller whose domain contains $\vectbf{x}$ and it has a higher priority than $\priority(\purt', \vectbf{b}')$, which is a contradiction.

%Since $\purt$ and $\hat{\purt}$ are partitions of $\indexset$ compatible with $\tree$, i.e. $\purt \subset \cluster{\tree}$ and $\hat{\purt} \subset \cluster{\tree}$, for each element of $\hat{I} \in \hat{\purt}$ there exist an element of $I \in \purt$ either $\hat{I} \subset I$ or $ \hat{ I}\supset I$ or $\hat{I} = I $. 

It follows from \reflem{lem.FinerCoarserResolution} that  at least one of the followings always holds:
\begin{itemize}[leftmargin=*]
\item \underline{Case 1 ($\purt$ Partially Refines $\purt'$):} There exists a cluster $K' \in \purt'$ with a nontrivial partition  $\mathcal{K}'$ (i.e. $|\mathcal{K}'| \geq 2$) such that  $\mathcal{K}' \subset \purt$.
Since $\vectbf{x} \in \domain\prl{\purt, \vectbf{b}}$ and all the elements of $\mathcal{K}'$ are descendants of $K'$ in $\tree$, the recursive tree traversal in \reftab{tab.PolicySelection} requires that  $\vectbf{x} \in \setH(K') \setminus \setA(K')$. 
Hence, $\vect{b}'_{K'} = -1$.     

Since $\prl{a + b}^2 > a^2 + b^2$ for any $a,b \in \R_{+}$, one can observe that replacing $K'$ of $\purt'$ with the elements of $\mathcal{K}'$ and updating $\vectbf{b}'$ with the associated binary values from $\vectbf{b}$ yields the index $(\purt'', \vectbf{b}'')$ of another local controller, 
\begin{align}
\purt'' & = \mathcal{K}' \cup \purt' \setminus \{K'\}, \\
\vectbf{b}'' &= (\vect{b}''_{I})_{I \in \purt''} \,\, \text{ s.t. } \,\, \vect{b}''_{I} = \left \{ 
\begin{array}{@{}l@{}l@{}}
\vect{b}_{I}  & \text{, if } I \in \mathcal{K}',\\ 
\vect{b}'_{I} & \text{, if } I \in \purt' \setminus \{K'\},
\end{array}
\right. \!\!
\end{align}
at a strictly higher priority,
\begin{align}
\hspace{-1mm}\priority(\purt'', \vectbf{b}'') & = \priority (\purt', \vectbf{b}') \sqz{+} \underbrace{|K'|^2 \sqz{+} \!\! \sum_{I' \in \mathcal{K}'} \! \vect{b}_{I'} |I'|^2}_{> \;0}\!, \!\!  \nonumber\\
&> \priority (\purt', \vectbf{b}'). 
\end{align}

Note that  we still have $\vectbf{x} \in \domain(\purt'',\vectbf{b}'')$ since $\vectbf{x} \in \domain(\purt,\vectbf{b}) \cap \domain(\purt', \vectbf{b}')$.   

\item \underline{Case 2 ($\purt'$ Partially Refines $\purt$):} There exists a cluster $K \in \purt$ with a nontrivial partition  $\mathcal{K}$ (i.e. $|\mathcal{K}| \geq 2$) such that  $\mathcal{K} \subset \purt'$.
Since $K \in \purt$, one of the base conditions in \reftab{tab.PolicySelection} at cluster $K$ holds, and so we have either $\vectbf{x} \in \setA\prl{K}$ or $\vectbf{x} \not \in \setH\prl{K}$. 
Further, since $\vectbf{x} \in \domain(\purt', \vectbf{b}')$ and $K$ is an ancestor of all the elements of $\mathcal{K}$ in $\tree$, we have $\vectbf{x} \in \setH\prl{K}$.
Therefore, $\vectbf{x} \in \setA\prl{K}$ and $\vect{b}_K = +1$.
  
Once again, since $\prl{a + b}^2 > a^2 + b^2$ for any $a,b \in \R_{+}$ and $\vectbf{x} \in \domain(\purt,\vectbf{b}) \cap \domain(\purt', \vectbf{b}')$, one can verify that the following local policy index
\begin{align}
\purt'' &= \crl{K} \cup \purt' \setminus \mathcal{K}, \\
\vectbf{b}'' &= (\vect{b}''_{I})_{I \in \purt''} \,\, \text{ s.t. } \,\, \vect{b}''_{I} = \left \{
\begin{array}{ll}
+1 & \text{, if } I = K, \\
\vect{b}'_{I} & \text{, if } I \in \purt' \setminus \mathcal{K},
\end{array}
\right. \!\!
\end{align}
has a strictly higher priority,
\begin{align}
\hspace{-1mm}
\priority(\purt'', \vectbf{b}'') & = \priority (\purt', \vectbf{b}')  \sqz{+} \underbrace{|K|^2 \sqz{-} \! \sum_{I \in \mathcal{K}} \! \vect{b}_{I} |I|^2}_{> \;0}\!, \!\!  \nonumber
\\
&> \priority (\purt', \vectbf{b}'), 
\end{align}
and its domain contains $\vectbf{x}$, i.e. $\vectbf{x} \in \domain(\purt'', \vectbf{b}'')$.

\item \underline{Case 3 (Identical Resolution):} $\purt' = \purt$ and $\vectbf{b}' \neq \vectbf{b}$.
Since $\purt' = \purt$, one can maximize $\priority(\purt', \vectbf{b}')$ \refeqn{eq.Priority} by maximizing the binary vector $\vectbf{b}'$, which is achieved by setting $\vect{b}'_{I} = +1$ for any $I \in \purt'$ whenever $\vectbf{x} \in \setA(I)$.
The base conditions in \reftab{tab.PolicySelection} guarantee such an optimal selection of $\vectbf{b}'$. 
However, since $\vectbf{b}' \neq \vectbf{b}$, we have
\begin{align}
\priority(\purt, \vectbf{b}) &> \priority (\purt', \vectbf{b}'), 
\end{align}
which  completes the proof. \qed  
\end{itemize}  
\end{proof}

%%%%%%%%%%%%%%%%%%%%%%%%%%%%%%%%%%%%%%%%%%%%%%
%%%%%%%%%%%%%%%%%%%%%%%%%%%%%%%%%%%%%%%%%%%%%%
\subsection{Proof of \refprop{prop.LocalPolicyExistenceUniqueness}}
\label{app.LocalPolicyExistenceUniqueness}
%%%%%%%%%%%%%%%%%%%%%%%%%%%%%%%%%%%%%%%%%%%%%%
%%%%%%%%%%%%%%%%%%%%%%%%%%%%%%%%%%%%%%%%%%%%%%
%
\begin{proof}
The continuity and piecewise smoothness of $\h{\purt, \vectbf{b}}$ (\refprop{prop.LocalPolicyProperty}) implies its locally Lipschitz continuity in $\stratum{\tree}$ \cite{kuntz_scholtes_JMAA1994}; and the existence of at least one trajectory of $\h{\purt, \vectbf{b}}$ starting at $\vectbf{x}$ follows from its continuity.

Let $\vectbf{x}^t$ denote a trajectory of $\h{\purt, \vectbf{b}}$ starting at any $\vectbf{x}^0 \in \domain\prl{\purt, \vectbf{b}}$ for all $t \geq 0$.
We have from \refprop{prop.DomainInvariance} that $\vectbf{x}^t$ remains in $\domain\prl{\purt, \vectbf{b}}$ for all $t \geq 0$. 
Further, by \reflem{lem.CentroidDynamics}, the centroidal trajectory $\ctrd{\vectbf{x}^t|\indexset}$ is guaranteed to lie on the line segment  joining $\ctrd{\vectbf{x}^0| \indexset}$ and $\ctrd{\vectbf{y}|\indexset}$; and, by \reflem{lem.ConfRadiusBound}, the centroidal configuration radius $\radiusCL{}\prl{\vectbf{x}^t|\indexset}$ \refeqn{eq.ConfRadius} is bounded above by a certain finite value, $R\prl{\vectbf{x}^0, \vectbf{y}}$, depending only on the initial and desired configurations, $\vectbf{x}^0$ and $\vectbf{y}$, respectively. 
Thus, all trajectories of $\h{\purt, \vectbf{b}}$ stay in  a compact subset  $W$ of  $\domain\prl{\purt, \vectbf{b}}$ and the compact set defined by the Minkowski sum of the line segment joining $\ctrd{\vectbf{x}^0|\indexset}$ and $\ctrd{\vectbf{y}|\indexset}$ and the closed ball centered at the origin with radius of $R\prl{\vectbf{x}^0, \vectbf{y}}$.

Given that all trajectories of $\h{\purt, \vectbf{b}}$ starting at any $\vectbf{x} \in \domain\prl{\purt, \vectbf{b}}$ lie in a compact subset $W$ of $\domain\prl{\purt, \vectbf{b}}$, the uniqueness of its flow follows from  the Lipschitz continuity of $\h{\purt, \vectbf{b}}$ in $W$ since a locally Lipschitz function on $\stratum{\tree}$ is Lipschitz on every compact subset of $\stratum{\tree}$, also refer to Theorem 3.3 in \cite{khalil_NonlinearSystems_2001}.  
Moreover, this unique flow is continuous and piecewise smooth since 
it is the integral of the continuous and piecewise smooth vector field $\h{\purt, \vectbf{b}}$ \cite{shilov_RealAnalysis1996}, which completes the proof.\qed
\end{proof}

%%%%%%%%%%%%%%%%%%%%%%%%%%%%%%%%%%%%%%%%%%%%%%%%%%%%%%
%%%%%%%%%%%%%%%%%%%%%%%%%%%%%%%%%%%%%%%%%%%%%%%%%%%%%%
\subsection{Proof of \refprop{prop.DomainInvariance}}
\label{app.DomainInvariance}
%%%%%%%%%%%%%%%%%%%%%%%%%%%%%%%%%%%%%%%%%%%%%%%%%%%%%%
%%%%%%%%%%%%%%%%%%%%%%%%%%%%%%%%%%%%%%%%%%%%%%%%%%%%%%

Before proceeding with the proof of \refprop{prop.DomainInvariance}, we find it useful to emphasize some critical properties of a trajectory $\vectbf{x}^t$ of $\h{\purt, \vectbf{b}}\prl{\vectbf{x}}$ starting at any $\vectbf{x}^0 \in \domain\prl{\purt, \vectbf{b}}$. 

\begin{lemma}\label{lem.DomainInvarianceBase1}
A trajectory $\vectbf{x}^t$ of $\h{\purt, \vectbf{b}}\prl{\vectbf{x}}$ (\reftab{tab.LocalPolicy}) starting at any initial configuration $\vectbf{x}^0 \in \domain\prl{\purt, \vectbf{b}}$ \refeqn{eq.LocalPolicyDomain} satisfies the following properties for any $I \in \purt$ with $\vect{b}_I = + 1$ and $t \geq 0$,
\begin{enumerate}[label=(\roman*), itemsep=0.5mm]
\item $\lie_{\overrightarrow{\vectbf{y}}}\frac{1}{2}\!\norm{\vect{x}_i^t - \vect{x_j}^t}^2 \geq \prl{r_i + r_j}^2\!, \quad  \forall i \neq j \in I,$
\item $\lie_{\overrightarrow{\vectbf{y}}} \vectprod{\prl{\vect{x}_k^t \sqz{-} \ctrdmid{K\!}{\vectbf{x}^t}\!}\!}{\ctrdsep{K\!}{\vectbf{x}^t}} \sqz{\geq} 0,  \quad \hspace{-2mm}\forall k \sqz{\in} K\!,  K \!\sqz{\in} \descendantCL{I,\tree},$
\item $\sepmag{k,K\!}{\vectbf{x}^t}\geq 0, \quad   \forall k \in K, K  \in \descendantCL{I,\tree},$
\item $\norm{\vect{x}_i^t - \vect{x}_j^t}^2 > \prl{r_i + r_j}^2, \quad \forall i \neq j \in I$. 
\end{enumerate}
\end{lemma}
\begin{proof}
See \refapp{app.DomainInvarianceBase1}. \qed
\end{proof}

\begin{lemma} \label{lem.DomainInvarianceBase2}
A trajectory $\vectbf{x}^t$ of $\h{\purt, \vectbf{b}}\prl{\vectbf{x}}$ (\reftab{tab.LocalPolicy}) starting at any initial configuration $\vectbf{x}^0 \in \domain\prl{\purt, \vectbf{b}}$ \refeqn{eq.LocalPolicyDomain} satisfies the following properties for any $I \in \purt$ with $\vect{b}_I = - 1$ and $t \geq 0$,
\begin{enumerate}[label=(\roman*), itemsep=0.5mm]
\item $\sepmag{k,K\!}{\vectbf{x}^t}\geq 0, \quad   \forall k \in K\!, K \in \descendantCL{I,\tree},$
\item $\norm{\vect{x}_i^t - \vect{x}_j^t}^2 > \prl{r_i + r_j}^2, \quad  \forall i \neq j \in I.$
\end{enumerate}
\end{lemma}
\begin{proof}
See \refapp{app.DomainInvarianceBase2}. \qed
\end{proof}

\begin{lemma} \label{lem.DomainInvarianceRecursion}
Let $\visitedcluster$ \refeqn{eq.VisitedClusters} be the set of clusters visited during the recursive computation of  $\h{\purt, \vectbf{b}}\prl{\vectbf{x}}$ in \reftab{tab.LocalPolicy}.

Then a trajectory $\vectbf{x}^t$ of $\h{\purt, \vectbf{b}}\prl{\vectbf{x}}$ starting at any initial configuration $\vectbf{x}^0 \in \domain\prl{\purt, \vectbf{b}}$ \refeqn{eq.LocalPolicyDomain} satisfies the following properties for any $I \in \visitedcluster \setminus \purt$ and $t \geq 0$,
\begin{enumerate}[label=(\roman*), itemsep=0.5mm]
\item $\sepmag{k,K\!}{\vectbf{x}^t}\geq  r_k + \alpha, \quad   \forall k \in K\!, K \in \childCL{I,\tree},$
\item $\norm{\vect{x}_i^t \sqz{-} \vect{x}_j^t}^2 > \prl{r_i \sqz{+} r_j}^2\!\!\!, \quad  \forall i \sqz{\in} K, j \sqz{\in} I \sqz{\setminus} K, K \sqz{\in} \childCL{I,\tree}.$
\end{enumerate}
\end{lemma}
\begin{proof}
See \refapp{app.DomainInvarianceRecursion}. \qed
\end{proof}

Accordingly, we conclude the positive invariance of the domain $\domain\prl{\purt, \vectbf{b}}$ of $\h{\purt, \vectbf{b}}$ as follows:

\begin{proof}[Proof of \refprop{prop.DomainInvariance}]

By Lemma \ref{lem.DomainInvarianceBase1}.(iii)-(iv) and Lemmas \ref{lem.DomainInvarianceBase2}-\ref{lem.DomainInvarianceRecursion}.(i)-(ii), a trajectory $\vectbf{x}^t$ of $\h{\purt, \vectbf{b}}$ starting at any $\vectbf{x}^0 \in \domain\prl{\purt, \vectbf{b}}$ is guaranteed to remain in $\stratum{\tree}$ for all future time.
Given $\vectbf{x}^t \in \stratum{\tree}$ for all $t \geq 0$, \reflem{lem.DomainInvarianceBase1}.(i)-(ii) imply $\vectbf{x}^t \in \setA\prl{I}$ for any $I \in \purt$ with $\vect{b}_I = +1$; and  \reflem{lem.DomainInvarianceRecursion}.(iii) implies $\vectbf{x}^t \in \setH\prl{K}$ for every ancestor $K \in \ancestorCL{I,\tree}$ of any $I \in \purt$.
Thus, by definition \refeqn{eq.LocalPolicyDomain}, we have $\vectbf{x}^t \in \domain\prl{\purt, \vectbf{b}}$ for all $t \geq 0$. \qed

% Key for the invariance and stability analysis is the observation about the recursion step in \reftab{tab.LocalPolicy} that for any cluster $I \in \mathcal{R}_{\tree}\prl{\purt}$  the split preserving field $\fH$ \refeqn{eq.SplitPreservingField} steers  agents of each child cluster of $I$ with the same velocity vector yielding a continuous rigid translation of each child cluster.
%Since the set of rigid translations form a group with the addition operation, let $\vect{v}_I \in \R^d$ denote the total rigid translation field on a cluster $I \in \purt \cup \mathcal{R}_{\tree}\prl{\purt} $ due to all its ancestor clusters $\ancestorCL{I,\tree}$. 
\end{proof}

%%%%%%%%%%%%%%%%%%%%%%%%%%%%%%%%%%%%%%%%%%%%%%%%%%%%
%%%%%%%%%%%%%%%%%%%%%%%%%%%%%%%%%%%%%%%%%%%%%%%%%%%%
\subsection{Proof of \refprop{prop.PreparesRelation}}
\label{app.PreparesRelation}
%%%%%%%%%%%%%%%%%%%%%%%%%%%%%%%%%%%%%%%%%%%%%%%%%%%%
%%%%%%%%%%%%%%%%%%%%%%%%%%%%%%%%%%%%%%%%%%%%%%%%%%%%

Here we first establish finite-time prepares relations between pairs of local policies whose indices are related to each other in a certain way as specified in \refdef{def.SubPreparesGraph}; and then we continue with the proof of \refprop{prop.PreparesRelation}.

\begin{lemma}[The Case of \refdef{def.SubPreparesGraph}.\ref{def.subprepgraph1}]\label{lem.PreparesRelationSingleton}
Let $\purt \in \purtset_{\indexset}\prl{\tree}$ be a partition of $\indexset$ and  $\vectbf{b}, \vectbf{b}' \in \crl{-1,+1}^{\purt}$.
If $\vect{b}_I = \vect{b}'_I$ for all $I \in \purt$ but a singleton cluster $D \in \purt$ where $\vect{b}_D = -1$ and $\vect{b}'_D = +1$, then the domains \refeqn{eq.LocalPolicyDomain} of  local control policies $\h{\purt, \vectbf{b}}$ and $\h{\purt, \vectbf{b}'}$ are identical, i.e.
\begin{align}
\domain(\purt, \vectbf{b}') = \domain\prl{\purt, \vectbf{b}},
\end{align}
and their priorities \refeqn{eq.Priority} satisfy
\begin{align}
\priority\prl{\purt, \vectbf{b}'} = \priority\prl{\purt, \vectbf{b}} + 2.
\end{align}
\end{lemma}
\begin{proof}
See \refapp{app.PreparesRelationSingleton}.\qed
\end{proof}

\begin{lemma}[The Case of \refdef{def.SubPreparesGraph}.\ref{def.subprepgraph2}] \label{lem.PreparesRelationNegative}
Let $\purt \in \purtset_{\indexset}\prl{\tree}$ be a partition of $\indexset$ and $\vectbf{b} \in \crl{-1,+1}^{\purt}$ such that $\vect{b}_I = -1$ for a nonsingleton cluster $I \in \purt$; and let $\purt' =  \purt \setminus \crl{I} \cup \childCL{I,\tree}$ and $\vectbf{b}' \in \crl{-1, +1}^{\purt'}$ with $\vect{b}'_K = -1$ for all $ K  \in \childCL{I,\tree}$ and $ \vect{b}'_D = \vect{b}_D$ for all $D \in \purt \setminus \crl{I}$.

Then all trajectories of the local control policy $\h{\purt, \vectbf{b}}$ starting in its domain $\domain\prl{\purt, \vectbf{b}}$  reach in finite time  the domain $\domain(\purt', \vectbf{b}')$ of the local controller $\h{\purt', \vectbf{b}'}$ which has a higher $\priority$ \refeqn{eq.Priority} than $\h{\purt, \vectbf{b}}$ does, i.e.
\begin{align}
\priority (\purt', \vectbf{b}') > \priority \prl{\purt, \vectbf{b}}.
\end{align}  
\end{lemma}
\begin{proof}
See \refapp{app.PreparesRelationNegative}. \qed
\end{proof}

\begin{lemma}[The Case of \refdef{def.SubPreparesGraph}.\ref{def.subprepgraph3}] \label{lem.PreparesRelationPositive}
Let $\purt \in \purtset_{\indexset}\prl{\tree}$ be a partition of $\indexset$ and $\vectbf{b} \in \crl{-1,+1}^{\purt}$ such that $\childCL{I,\tree} \subset \purt$  for a nonsingleton cluster $I \in \cluster{\tree}$   and $\vect{b}_{K} = +1$ for all $K \in \childCL{I,\tree}$; and let $\purt' =  \purt \setminus \childCL{I,\tree} \cup \crl{I}$ and $\vectbf{b}' \in \crl{-1, +1}^{\purt'}$ with $ \vect{b}'_I = +1$ and $\vect{b}'_D = \vect{b}_D $ for all $D \in \purt \setminus \childCL{I,\tree}$.

Then the local control policy $\h{\purt, \vectbf{b}}$ steers (almost) all configurations in its domain $\domain\prl{\purt, \vectbf{b}}$ in finite time to the domain $\domain(\purt', \vectbf{b}')$ of the local controller $\h{\purt', \vectbf{b}'}$ which has a higher $\priority$ \refeqn{eq.Priority} than $\h{\purt, \vectbf{b}}$ does, i.e.
\begin{align}
\priority (\purt', \vectbf{b}') > \priority \prl{\purt, \vectbf{b}}.
\end{align}
\end{lemma}
\begin{proof}
See \refapp{app.PreparesRelationPositive}. \qed
\end{proof}

\begin{proof}[Proof of \refprop{prop.PreparesRelation}]

Since $\purt$ is a partition of $\indexset$ compatible with $\tree$, i.e. $\purt \subset \cluster{\tree}$, observe that  if $\card{\purt} > 1$, then there exists a cluster $I \in \cluster{\tree}$ such that $\childCL{I,\tree} \subset \purt$ (\reflem{lem.ClusterPartition}).
Hence, since  $\prl{\purt, \vectbf{b}} \neq \prl{\crl{\indexset}, +1}$, at least one of the followings always holds:
\begin{enumerate}[leftmargin=7mm, label = (\alph*)]
\item There exists $I \in \purt$ such that $\vect{b}_I = -1$.
If $\card{I} = 1$, then we have the result by \reflem{lem.PreparesRelationSingleton}; otherwise ($\card{I}> 1$), the results follows from \reflem{lem.PreparesRelationNegative}.

\item There exist a cluster $I \in \cluster{\tree}$ such that $\childCL{I,\tree} \subset \purt$  and $\vect{b}_{K} = +1$ for all $K\in \childCL{I,\tree}$.
Accordingly, the results follows from \reflem{lem.PreparesRelationPositive} and this completes the proof. \qed
\end{enumerate}
\end{proof}

%%%%%%%%%%%%%%%%%%%%%%%%%%%%%%%%%%%%%%%%%%%%%%%%%%%%%%%%
%%%%%%%%%%%%%%%%%%%%%%%%%%%%%%%%%%%%%%%%%%%%%%%%%%%%%%%%
\subsection{Proof of \reflem{lem.CentroidDynamics}}
\label{app.CentroidDynamics}
%%%%%%%%%%%%%%%%%%%%%%%%%%%%%%%%%%%%%%%%%%%%%%%%%%%%%%%%
%%%%%%%%%%%%%%%%%%%%%%%%%%%%%%%%%%%%%%%%%%%%%%%%%%%%%%%%

\begin{proof}
For any cluster $I \in \purt$ the recursion in \reftab{tab.LocalPolicy} employs a vector field satisfying the associated base condition, and then recursively constructs an additive repulsion field at every ancestor  $\ancestorCL{I,\tree}$ of $I$, which can be explicitly written as follows: for any $i \in I$  and $I \in \purt$,
\begin{itemize}[leftmargin = *]
\item if $\vect{b}_I = +1$, then we have
\begin{align}
\hspace{-4mm}\vect{u}_i = \fA\prl{\vectbf{x}, \vectbf{0}, I}_i \sqz{+} \hspace{-9mm}\sum_{\substack{K \in \ancestorCL{I,\tree} \cup \crl{I} \setminus \crl{\indexset} \\ R = \parentCL{K,\tree}}} \hspace{-10mm}2 \falpha{R}{\vectbf{x}, \vectbf{v}_R} \! \frac{\card{\complementLCL{K}{\tree}}}{\card{R}} \! \frac{\ctrdsep{K\!}{\vectbf{x}}}{\norm{\ctrdsep{K\!}{\vectbf{x}}}}, \!\!\!\!
\end{align}

\item else ($\vect{b}_I = -1$), 
\begin{align}
\hspace{-4mm}\vect{u}_i = \fS\prl{\vectbf{x}, \vectbf{0}, I}_i \sqz{+} \hspace{-9mm}\sum_{\substack{K \in \ancestorCL{I,\tree} \cup \crl{I} \setminus \crl{\indexset} \\ R = \parentCL{K,\tree}}} \hspace{-10mm}2 \falpha{R}{\vectbf{x}, \vectbf{v}_R} \! \frac{\card{\complementLCL{K}{\tree}}}{\card{R}} \! \frac{\ctrdsep{K\!}{\vectbf{x}}}{\norm{\ctrdsep{K\!}{\vectbf{x}}}}, \!\!\!\!
\end{align}
\end{itemize}
for some $\vectbf{v}_R \in \prl{\R^d}^{\indexset}$ associated with cluster $R \in \ancestorCL{I,\tree}$.

\medskip

Now, using \refeqn{eq.AttractiveField} and \refeqn{eq.SeparationField}, one can verify that for any $I \in \purt$, 
{\small
\begin{align}\label{eq.CentroidDynamicsBase}
\hspace{-2mm}\ctrd{\vectbf{u}|I} \sqz{=} -\ctrd{\vectbf{x} \sqz{-} \vectbf{y}|I}  \sqz{+} \hspace{-10mm}\sum_{\substack{K \in \ancestorCL{I,\tree} \cup \crl{I} \setminus \crl{\indexset} \\ R = \parentCL{K,\tree}}} \hspace{-10mm}2 \falpha{R}{\vectbf{x}, \vectbf{v}_R} \!\! \frac{\card{\complementLCL{K}{\tree}}}{\card{R}} \! \frac{\ctrdsep{K\!}{\vectbf{x}}}{\norm{\ctrdsep{K\!}{\vectbf{x}}}}, \!\!\!\! 
\end{align}
}%
which can be generalized to other clusters in $\visitedcluster \setminus \purt$. 
That is to say, we now show that for any $I \in \visitedcluster$ the centroidal dynamics $\ctrd{\vectbf{u}|I}$ satisfies \refeqn{eq.CentroidDynamicsBase}.
Based on the recursive definition \refeqn{eq.VisitedClustersRecursion} of $\visitedcluster$, we provide a proof by  structural induction.
For any $I \in \visitedcluster$,
\begin{itemize}[leftmargin=*]
\item (Base Case) If $I \in \purt$, then the result is shown above in \refeqn{eq.CentroidDynamicsBase}.
\item (Induction) Otherwise, $\card{I} > 2$ and let $\crl{I_L, I_R} = \childCL{I,\tree}$.
(Induction hypothesis) Suppose that $\ctrd{\vectbf{u}|I_L}$ and $\ctrd{\vectbf{u}|I_R}$ satisfy \refeqn{eq.CentroidDynamicsBase}.
Then using
\begin{align}
\ctrd{\vectbf{u}|I} = \frac{\card{I_L}}{\card{I}}\ctrd{\vectbf{u}|I_L} + \frac{\card{I_R}}{\card{I}}\ctrd{\vectbf{u}|I_R},
\end{align} 
one can obtain \refeqn{eq.CentroidDynamicsBase} for cluster $I$ as well.
\end{itemize}

\smallskip

Observe that for the root cluster $\indexset$ the equation  \refeqn{eq.CentroidDynamicsBase} simplifies and yields \refeqn{eq.CentroidDynamicsRoot}.
Further, using \refeqn{eq.CentroidDynamicsBase}, we obtain \refeqn{eq.CentroidDynamicsGeneral}  for any $I \in \visitedcluster \setminus \crl{\indexset}$ with parent $P = \parentCL{I,\tree}$ as follows:
{
\begin{align}
 \ctrd{\vectbf{u}|I} &= -\ctrd{\vectbf{x} \sqz{-} \vectbf{y}|I} \sqz{+} 2 \falpha{P}{\vectbf{x}, \vectbf{v}_P} \! \frac{\card{\complementLCL{I}{\tree}}}{\card{P}} \! \frac{\ctrdsep{I}{\vectbf{x}}}{\norm{\ctrdsep{I}{\vectbf{x}}}}  \nonumber \\
 & \hspace{12mm}+ \hspace{-5mm}\underbrace{\sum_{\substack{K \in \ancestorCL{I,\tree} \setminus \crl{\indexset} \\ R = \parentCL{K,\tree}}} \hspace{-5mm}2 \falpha{R}{\vectbf{x}, \vectbf{v}_R} \! \frac{\card{\complementLCL{K}{\tree}}}{\card{R}} \! \frac{\ctrdsep{K}{\vectbf{x}}}{\norm{\ctrdsep{K}{\vectbf{x}}}}}_{= \ctrd{\vectbf{u}|P} + \ctrd{\vectbf{x} - \vectbf{y}|P}}, \!\!\!\\
 & = -\ctrd{\vectbf{x} \sqz{-} \vectbf{y}|I} \sqz{+} 2 \falpha{P}{\vectbf{x}, \vectbf{v}_P} \! \frac{\card{\complementLCL{I}{\tree}}}{\card{P}} \! \frac{\ctrdsep{I}{\vectbf{x}}}{\norm{\ctrdsep{I}{\vectbf{x}}}\!\!}  \nonumber \\
 & \hspace{20mm} + \ctrd{\vectbf{u}|P} + \ctrd{\vectbf{x} - \vectbf{y}|P},
\end{align}
}%
which completes the proof. \qed
\end{proof}

%%%%%%%%%%%%%%%%%%%%%%%%%%%%%%%%%%%%%%%%
%%%%%%%%%%%%%%%%%%%%%%%%%%%%%%%%%%%%%%%%
\subsection{Proof of \reflem{lem.ConfRadiusBound}}
\label{app.ConfRadiusBound}
%%%%%%%%%%%%%%%%%%%%%%%%%%%%%%%%%%%%%%%
%%%%%%%%%%%%%%%%%%%%%%%%%%%%%%%%%%%%%%%

\begin{proof}
Since the domain $\domain\prl{\purt, \vectbf{b}}$ of $\h{\purt, \vectbf{b}}$ is positive invariant (\refprop{prop.DomainInvariance}), the existence of $\vectbf{x}^t$ for $t \geq 0$ simply follows from the continuity of $\h{\purt, \vectbf{b}}$ (\refprop{prop.LocalPolicyProperty}).
We now show that for any $I \in \visitedcluster$ \refeqn{eq.VisitedClusters} visited during the recursive computation of $\h{\purt, \vectbf{b}}$ in \reftab{tab.LocalPolicy} the  centroidal radius $\radiusCL{}\prl{\vectbf{x}^t|I}$ is bounded above by a certain value, $R_{I}\prl{\vectbf{x}^0, \vectbf{y}}$, depending only on $\vectbf{x}^0$ and $\vectbf{y}$.

Based on the recursive definition of $\visitedcluster$ in \refeqn{eq.VisitedClustersRecursion}, we now provide a proof of the result by  structural induction.
For any $I \in \visitedcluster$,
\begin{itemize}[leftmargin = *]
\item (Base Case 1) If $I \in \purt$ and $\card{I} = 1$, then the result simply follows since $\radiusCL{}\prl{\vectbf{x}^t|I} = r_i$ for all $t \geq 0$, where $I = \crl{i}$.
\item (Base Case 2) If $I \in \purt$, $\card{I}\geq 2$ and $\vect{b}_I = +1$, then, using \reftab{tab.LocalPolicy}, one can verify that for any $i \in I$
\begin{subequations} 
\begin{align}
\dot{\vect{x}}_i & = \h{\purt, \vectbf{b}}\prl{\vectbf{x}}_i  = \fA\prl{\vectbf{x}, \vectbf{u}, I}_i + \vect{v}_I,\\ 
& =   - \prl{\vect{x}_i - \vect{y}_i} + \vect{v}_I,
\end{align}
\end{subequations}
for some $\vectbf{u} \in \prl{\R^d}^{\indexset}$ and  $\vect{v}_I \in \R^d$, where $\vect{v}_I$  represents the accumulated rigid translation due to all ancestors of $I$ in $\tree$.

\smallskip

Accordingly, we obtain for any $i \in I$ that 
\begin{align}
\frac{d}{dt} \! \norm{\vect{x}_i  \sqz{-} \ctrd{\vectbf{x}|I}}^2 &\sqz{=} -2\! \norm{\vect{x}_i  \sqz{-} \ctrd{\vectbf{x}|I}}^2 \nonumber \\ 
& \hspace{10mm}\sqz{+} \vectprod{\prl{\vect{x}_i  \sqz{-} \ctrd{\vectbf{x}|I}\!}}{\!\prl{\vect{y}_i  \sqz{-} \ctrd{\vectbf{y}|I}\!}}\!, \!\!\!
\end{align}
from which one can conclude that
{
\begin{align}
\hspace{-3mm}\norm{\vect{x}^t_i  \sqz{-} \ctrd{\vectbf{x}^t|I}} \sqz{\leq} \max\!\prl{\norm{\vect{x}^0_i  \sqz{-} \ctrd{\vectbf{x}^0|I}}\!, \norm{\vect{y}_i  \sqz{-} \ctrd{\vectbf{y}|I}}}\! . \!\!\!
\end{align}
}%
Thus, by definition, it follows that the centroidal radius $\radiusCL{}\prl{\vectbf{x}^t|I}$ is bounded above as 
\begin{align}
\radiusCL{}\prl{\vectbf{x}^t|I} \leq R_{I}\prl{\vectbf{x}^0, \vectbf{y}} = \max\prl{\radiusCL{}\prl{\vectbf{x}^0|I}, \radiusCL{}\prl{\vectbf{y}|I}}.
\end{align}

\item (Base Case 3) If $I \in \purt$, $\card{I} \geq 2$ and $\vect{b}_I = -1$, then, using \reftab{tab.LocalPolicy}, one can verify that for any $k \in K$ and $K \in \childCL{I,\tree}$
\begin{subequations}
\begin{align}
\hspace{-2mm}\dot{\vect{x}}_k &= \h{\purt, \vectbf{b}}\prl{\vectbf{x}}_k = \fS\prl{\vectbf{x}, \vectbf{u}, I}_k + \vect{v}_I, \\
& = - \ctrd{\vectbf{x} \sqz{-} \vectbf{y}|I} \sqz{+} 2 \fbeta{I}{\vectbf{x}} \!\frac{\card{\complementLCL{K}{\tree}}}{\card{I}}  \frac{\ctrdsep{K\!}{\vectbf{x}}}{\norm{\ctrdsep{K\!}{\vectbf{x}}}} \sqz{+} \vect{v}_I, \!\!
\end{align}  
\end{subequations}
for some $\vectbf{u} \in \prl{\R^d}^{\indexset}$ and  $\vect{v}_I \in \R^d$.

\smallskip

Accordingly, we obtain for any $K \in \childCL{I,\tree}$ that
\begin{align}
\frac{d}{dt} \radiusCL{}\prl{\vectbf{x}|K} &= 0,\\
\frac{d}{dt} \norm{\ctrdsep{K\!}{\vectbf{x}}}^2 &= 2 \fbeta{I}{\vectbf{x}}.
\end{align}
Observe from \refeqn{eq.Beta} that 
{\small
\begin{align}
\norm{\ctrdsep{K\!}{\vectbf{x}}} \geq 2\prl{\!\!\beta \sqz{+} \hspace{-2.5mm}\max_{D \in \childCL{I,\tree}} \hspace{-2mm} \radiusCL{}\prl{\vectbf{x}|D}\!\!} \! &\sqz{\Longrightarrow} \hspace{-2.5mm} 
\min_{\substack{d \in D \\ D \in \childCL{I,\tree}}} \hspace{-3mm}(\sepmag{d,D\!}{\vectbf{x}}  \sqz{-} r_d) \geq  \beta, \nonumber \\
&\sqz{\Longrightarrow} \, \fbeta{I}{\vectbf{x}} = 0, \!\! 
\end{align}
}%
Thus, since $\radiusCL{}\prl{\vectbf{x}^t|K} = \radiusCL{}\prl{\vectbf{x}^0|K}$ for all $t \geq 0$, it follows that 
{\small
\begin{align}
\norm{\ctrdsep{K\!}{\vectbf{x}^t}} \leq \max\!\prl{\! \norm{\ctrdsep{K\!}{\vectbf{x}^0}} \!, 2\!\prl{\! \beta \sqz{+} \hspace{-2mm} \max_{D \in \childCL{I,\tree}}  \hspace{-2mm}\radiusCL{}\prl{\vectbf{x}^0|D}\!\!}  \!\!\!}\!\!, \! \!\! 
\end{align}
}%
and, since $\radiusCL{}\prl{\vectbf{x}|I} \leq  \max_{K \in \childCL{I,\tree}} \norm{\ctrdsep{K\!}{\vectbf{x}}} + \radiusCL{}\prl{\vectbf{x},K}$ for any $\vectbf{x} \in \stratum{\tree}$, we have
\begin{align}
\radiusCL{}\prl{\vectbf{x}^t|K} \leq R_{I}\prl{\vectbf{x}^0, \vectbf{y}},
\end{align}
where
{
\begin{align}
\hspace{-4mm} R_{I}\prl{\vectbf{x}^0\!, \vectbf{y}}  &\sqz{=} \max_{K \in \childCL{I,\tree}} \max \! \prl{ \!\!\bigg. \norm{\ctrdsep{K\!}{\vectbf{x}^0}\!}\!, 2\!\prl{ \Big. \!\beta \sqz{+} \radiusCL{}\prl{\vectbf{x}^0|K}\!\!} \!\!} \nonumber \\
& \hspace{25mm}\sqz{\sqz{+}} \hspace{-1mm}\max_{K \in \childCL{I,\tree}} \hspace{-1mm}\radiusCL{}\prl{\vectbf{x}^0|K}\!, 
\end{align} 
}%

\item (Induction) Otherwise, $\card{I} \geq 2$ and suppose that $\radiusCL{}\prl{\vectbf{x}^t|K} \leq R_K\prl{\vectbf{x}^0, \vectbf{y}}$ for all $K \in \childCL{I,\tree}$.
Then, using \reflem{lem.CentroidDynamics}, one can obtain for any $K \in \childCL{I,\tree}$ that
{\small
\begin{align}\label{eq.CtrdDyTemp1}
\hspace{-3mm}\frac{d}{dt} \! \norm{\ctrdsep{K\!}{\vectbf{x}}\!}^2 \sqz{=} - 2 \norm{\ctrdsep{K\!}{\vectbf{x}}\!}^2 \sqz{+} 2 \, \vectprod{\ctrdsep{K\!}{\vectbf{x}}\!}{\ctrdsep{K\!}{\vectbf{y}}} \sqz{+} 2\falpha{I\!}{\vectbf{x}, \vectbf{v}_I}\!, \!\!
\end{align}
}%
for some $\vect{v}_I \in \R^d$.
Now observe from \refeqn{eq.Alpha}  that
{\small
\begin{align}\label{eq.CtrdDyTemp2}
\hspace{-4mm}\norm{\ctrdsep{K\!}{\vectbf{x}}} \geq 2\!\prl{\!\!\beta \sqz{+} \hspace{-2mm}\max_{D \in \childCL{I,\tree}} \hspace{-2mm} \radiusCL{}\prl{\vectbf{x}|D}\!\!} &\sqz{\Longrightarrow} \hspace{-2mm} 
\min_{\substack{d \in D \\ D \in \childCL{I,\tree}}} \hspace{-2mm}(\sepmag{d,D\!}{\vectbf{x}} \sqz{-} r_d) \geq  \beta, \nonumber \\
&\sqz{\Longrightarrow} \, \falpha{I}{\vectbf{x}, \vectbf{v}_I} = 0, \!\! 
\end{align}
}%
Hence, using \refeqn{eq.CtrdDyTemp1} and \refeqn{eq.CtrdDyTemp2}, one can conclude that
{\small
\begin{align}
\norm{\ctrdsep{K\!}{\vectbf{x}^t}\!} \sqz{\leq} \max\prl{\!\!\norm{\ctrdsep{K\!}{\vectbf{x}^0}\!}\!, \norm{\big.\ctrdsep{K\!}{\vectbf{y}}\!}\!, 2\!\prl{\!\beta \sqz{+} \hspace{-3mm}\max_{D \in \childCL{I,\tree}} \hspace{-3mm} R_{D\!}\prl{\vectbf{x}^0\!, \vectbf{y}\!}\!\!}\!\!\!} 
\end{align}
}%
and since $\radiusCL{}\prl{\vectbf{x}|I} \leq  \max_{K \in \childCL{I,\tree}} \norm{\ctrdsep{K\!}{\vectbf{x}}}_2 + \radiusCL{}\prl{\vectbf{x},K}$ for any $\vectbf{x} \in \stratum{\tree}$, we have
\begin{align}
\radiusCL{}\prl{\vectbf{x}^t|I} \leq R_{I}\prl{\vectbf{x}^0, \vectbf{y}},
\end{align}
where
{\small
\begin{align}
\hspace{-3mm}R_{I\!}\prl{\vectbf{x}^0\!, \vectbf{y}} &\sqz{=} \hspace{-2mm} \max_{K \in \childCL{I,\tree}} \hspace{-2.5mm} \max\prl{\bigg.\!\!\norm{\ctrdsep{K\!}{\vectbf{x}^0}\!}\!, \norm{\ctrdsep{K\!}{\vectbf{y}}\!}\!, 2\!\prl{\Big.\!\beta \sqz{+}  R_{K\!}\prl{\vectbf{x}^0\!, \vectbf{y}\!}\!\!}\!\!\!} \nonumber \\
& \hspace{30mm}+ \max_{K \in \childCL{I,\tree}} \hspace{-3mm} R_{K\!}\prl{\vectbf{x}^0\!, \vectbf{y}\!}.
\end{align}
}
\end{itemize} 
Thus, the result follows with $R\prl{\vectbf{x}^0, \vectbf{y}} = R_{J\!}\prl{\vectbf{x}^0, \vectbf{y}}$. \qed
\end{proof}

%%%%%%%%%%%%%%%%%%%%%%%%%%%%%%%%%%%%%%%%%%%%%%%%%%%%
%%%%%%%%%%%%%%%%%%%%%%%%%%%%%%%%%%%%%%%%%%%%%%%%%%%%
\subsection{Proof of \reflem{lem.DomainInvarianceBase1}}
\label{app.DomainInvarianceBase1}
%%%%%%%%%%%%%%%%%%%%%%%%%%%%%%%%%%%%%%%%%%%%%%%%%%%%%
%%%%%%%%%%%%%%%%%%%%%%%%%%%%%%%%%%%%%%%%%%%%%%%%%%%%%

\begin{proof}
By definition of $\domain\prl{\purt, \vectbf{b}}$ \refeqn{eq.LocalPolicyDomain}, $\vectbf{x}^0 \in \setA\prl{I}$ for any $I \in \purt$ with $\vect{b}_I = + 1$, and one can verify using \reftab{tab.LocalPolicy} that for any $i \in I$ and  $I \in \purt$ with $\vect{b}_I = + 1$ 
\begin{subequations} \label{eq.LocalSystemModelInvarianceBase1}
\begin{align}
\dot{\vect{x}}_i &= \h{\purt, \vectbf{b}} \prl{\vectbf{x}}_i = \fA\prl{\vectbf{x}, \vectbf{u}, I}_i  +  \vect{v}_I, \\
&= - \prl{\vect{x}_i - \vect{y}_i} + \vect{v}_I, 
\end{align}
\end{subequations}
for some  $\vectbf{u} \in \prl{\R^d}^{\indexset}$ and  $\vect{v}_I \in \R^d$, where $\vect{v}_I$  represents the accumulated rigid translation due to ancestors of $I$ in $\tree$.  

Accordingly, \reflem{lem.DomainInvarianceBase1}.(i)-(iv) can be shown as follows:
\begin{enumerate}[leftmargin = 5mm, label=(\roman*)]

\item Using \refeqn{eq.SetACond1} and \refeqn{eq.LocalSystemModelInvarianceBase1}, one can verify  that for any $i \neq j \in I$
\begin{align}
\hspace{-3mm}\frac{d}{dt} \lie_{\overrightarrow{\vectbf{y}}}\scalebox{1}{$\frac{1}{2}$}\norm{\vect{x}_i \sqz{-} \vect{x}_i}^2 &= - \lie_{\overrightarrow{\vectbf{y}}}\scalebox{1}{$\frac{1}{2}$}\norm{\vect{x}_i \sqz{-} \vect{x}_i}^2 + \hspace{-8mm}\underbrace{\norm{\vect{y}_i \sqz{-} \vect{y}_i}^2}_{> \; \prl{r_i + r_j}^2 \text{, since } \vectbf{y} \in \stratum{\tree}}\hspace{-7mm},  \!\!\!\\
& > - \lie_{\overrightarrow{\vectbf{y}}}\scalebox{1}{$\frac{1}{2}$}\norm{\vect{x}_i - \vect{x}_i}^2  + \prl{r_i + r_j}^2\!\!,\!\! \!
\end{align}
and so for any $t \geq 0$
\begin{align}
\hspace{-2mm}\lie_{\overrightarrow{\vectbf{y}}}\scalebox{1}{$\frac{1}{2}$}\!\norm{\vect{x}^t_i \sqz{-} \vect{x}^t_i}^2 & \geq e^{-t}\hspace{-5mm}\underbrace{\lie_{\overrightarrow{\vectbf{y}}}\scalebox{1}{$\frac{1}{2}$}\!\norm{\vect{x}^0_i \sqz{-} \vect{x}^0_i}^2}_{\geq \; \prl{r_i + r_j}^2 \text{, since } \vectbf{x}^0 \in \setA\prl{I}} \hspace{-4mm}  \sqz{+} \prl{1 \sqz{-} e^{-t}} \!\prl{r_i \sqz{+} r_j}^2 \!\!,  \nonumber \\
& \geq \prl{r_i + r_j}^2.
\end{align}

\item Similarly, using \refeqn{eq.SetACond2} and \refeqn{eq.LocalSystemModelInvarianceBase1}, we obtain for any $k \in K, K \in \descendantCL{I,\tree}$
\begin{align}
\hspace{-4mm}
\frac{d}{dt} \lie_{\overrightarrow{\vectbf{y}}} \vectprod{\prl{\vect{x}_k \sqz{-} \ctrdsep{K\!}{\vectbf{x}}\!}\!}{\ctrdsep{K\!}{\vectbf{x}}} & \sqz{=} - \lie_{\overrightarrow{\vectbf{y}}} \vectprod{\prl{\vect{x}_k \sqz{-} \ctrdsep{K\!}{\vectbf{x}}\!}\!}{\ctrdsep{K\!}{\vectbf{x}}} \nonumber \\
& \hspace{7mm}+ \underbrace{\sepmag{k,K\!}{\vectbf{y}} \norm{\ctrdsep{K\!}{\vectbf{y}}}}_{\geq 0 \text{, since } \vectbf{y} \in \stratum{\tree}}, \!\!\! 
\end{align}
from which we conclude for any $t \geq 0$
{
\begin{align}
\hspace{-2mm} \lie_{\overrightarrow{\vectbf{y}}} \vectprod{\prl{\vect{x}_k^t \sqz{-} \ctrdsep{K\!}{\vectbf{x}^t}\!}\!}{\ctrdsep{K\!}{\vectbf{x}^t}} & \sqz{\geq} e^{-t} \underbrace{\lie_{\overrightarrow{\vectbf{y}}} \vectprod{\prl{\vect{x}_k^0 \sqz{-} \ctrdsep{K\!}{\vectbf{x}^0}\!}\!}{\ctrdsep{K\!}{\vectbf{x}^0}}}_{\geq \; 0 \text{, since } \vectbf{x}^0 \in \setA\prl{I}},  \nonumber \\
& \sqz{\geq} 0 \,.
\end{align}
}%

\item Now observe from \refeqn{eq.SetACond2}, \refeqn{eq.sepmagLie}  and \refeqn{eq.LocalSystemModelInvarianceBase1} that for any $k \in K$ and $K \in \descendantCL{I,\tree}$ 
\begin{align}
\hspace{-2mm}\frac{d}{dt} \sepmag{k, K\!}{\vectbf{x}} &=  -\sepmag{k, K\!}{\vectbf{x}} \prl{\!1 \sqz{+} \frac{\vectprod{\ctrdsep{K\!}{\vectbf{x}}}{\ctrdsep{K\!}{\vectbf{y}}}}{\norm{\ctrdsep{K\!}{\vectbf{x}}}^2}} \nonumber \\
& \hspace{15mm} +   \underbrace{\frac{\lie_{\overrightarrow{\vectbf{y}}} \vectprod{\prl{\vect{x}_k \sqz{-} \ctrdmid{K\!}{\vectbf{x}}}}{\ctrdsep{K\!}{\vectbf{x}}}}{\norm{\ctrdsep{K\!}{\vectbf{x}}}}}_{\geq \; 0 \text{ by \reflem{lem.DomainInvarianceBase1}.(ii) }  }\!. \!  \!\!
\end{align}
As a result, since $\frac{d}{dt}\sepmag{k, K\!}{\vectbf{x}} \geq 0$ whenever $\sepmag{k, K\!}{\vectbf{x}} = 0$,  we have the invariance of local cluster structure, i.e. $\sepmag{k, K\!}{\vectbf{x}^t} \geq 0$ for all $t \geq 0$.   

\item The relative displacement of any pair of agents, $i \neq j \in I$, satisfies
\begin{align} 
\dot{\vect{x}}_i - \dot{\vect{x}}_j = - \prl{\vect{x}_i - \vect{x}_j} + \prl{\vect{y}_i - \vect{y}_j}.
\end{align}
whose solution for  $t \geq 0$ is explicitly given by
\begin{align}
\vect{x}_i^t - \vect{x}_j^t = e^{-t}\prl{\vect{x}^0_i - \vect{x}^0_j} + \prl{1 - e^{-t}}\prl{\vect{y}_i - \vect{y}_j}.
\end{align} 
Hence, since $\vectbf{x}^0 \in \setA\prl{I}$ and $\vectbf{y} \in \stratum{\tree}$, one can verify   the intra-cluster collision avoidance as follows: 
\begin{align}
\hspace{-2mm}
\norm{\vect{x}_i^t - \vect{x}_j^t}^2 &= e^{-2t} \underbrace{\norm{\vect{x}_i^0 \sqz{-} \vect{x}_j^0}^2}_{> \; \prl{r_i + r_j}^2} +  \prl{1 \sqz{-} e^{-t}}^2 \underbrace{\norm{\vect{y}_i \sqz{-} \vect{y}_j}^2}_{> \; \prl{r_i + r_j}^2} \nonumber \\
& \hspace{13mm} +  e^{-t}\prl{1 \sqz{-} e^{-t}}\underbrace{\lie_{\overrightarrow{\vectbf{y}}} \norm{\vect{x}_i^0 \sqz{-} \vect{x}_j^0}^2}_{\geq \; 2\;\prl{r_i + r_j}^2}, \!\!\\
& > \prl{r_i + r_j}^2,
\end{align}
and this completes the proof \qed
\end{enumerate}
\end{proof}

%%%%%%%%%%%%%%%%%%%%%%%%%%%%%%%%%%%%%%%%%%%%%%%%%%%
%%%%%%%%%%%%%%%%%%%%%%%%%%%%%%%%%%%%%%%%%%%%%%%%%%%
\subsection{Proof of \reflem{lem.DomainInvarianceBase2}}
\label{app.DomainInvarianceBase2}
%%%%%%%%%%%%%%%%%%%%%%%%%%%%%%%%%%%%%%%%%%%%%%%%%%%
%%%%%%%%%%%%%%%%%%%%%%%%%%%%%%%%%%%%%%%%%%%%%%%%%%%

\begin{proof}
For any singleton  $I \in \purt$ the results simply follow  since a singleton cluster contains no pair of indices and has an empty set of descendants.
Otherwise,  for any nonsingleton $I \in \purt$ with $\vect{b}_I = -1$, one can obtain from \reftab{tab.LocalPolicy} that for any $k \in K$ and $K \in \childCL{I,\tree}$
\begin{subequations} \label{eq.LocalSystemModelInvarianceBase2}
\begin{align}
\dot{\vect{x}}_k &= \h{\purt, \vectbf{b}}\prl{\vectbf{x}}_k =  \fS\prl{\vectbf{x}, \vectbf{u}, I}_k  + \vect{v}_I, \\
& = - \ctrd{\vectbf{x} \sqz{-} \vectbf{y}|I} \sqz{+} 2 \fbeta{I}{\vectbf{x}} \!\frac{\card{\complementLCL{K}{\tree}}}{\card{I}} \! \frac{\ctrdsep{K\!}{\vectbf{x}}}{\norm{\ctrdsep{K\!}{\vectbf{x}}}} \sqz{ +} \vect{v}_I, \!\!\! 
\end{align}
\end{subequations}
for some $\vectbf{u} \in\prl{\R^d}^{\indexset}$ and $\vect{v}_I \in \R^d$, where $\vect{v}_I$ models the overall rigid translation due to ancestors of $I$ in $\tree$.

Accordingly, using \refeqn{eq.LocalSystemModelInvarianceBase2}, we will show the results as follows:
\begin{enumerate}[leftmargin = 5mm, label=(\roman*) ]

\item The preservation of local cluster structure can be observed in two steps.
First, since $\dot{\vect{x}}_i - \dot{\vect{x}}_j = 0$ for any $i \neq j \in K$ and $K \in \childCL{I,\tree}$, we have $\sepmag{d,D\!}{\vectbf{x}^t} = \sepmag{d,D\!}{\vectbf{x}^0} \geq 0$ for all $t \geq 0 $  and $d \in D$, $D \in \descendantCL{K,\tree}$ and $K \in \childCL{K,\tree}$.
Second, using \refeqn{eq.sepmagLie} and  \refeqn{eq.Beta}, we obtain that for any $k \in K$, $K \in \childCL{I,\tree}$ 
\begin{align}
\frac{d}{dt}\sepmag{k,K\!}{\vectbf{x}} &\sqz{=} \fbeta{I}{\vectbf{x}}  \geq - \prl{\sepmag{k,K\!}{\vectbf{x}} \sqz{-} r_k \sqz{-} \beta}, \!\!\!
\end{align}
where $\beta > 0$.
Hence, $\frac{d}{dt}\sepmag{k,K\!}{\vectbf{x}} > 0$  whenever $\sepmag{k,K\!}{\vectbf{x}} = 0$, and so $\sepmag{k,K\!}{\vectbf{x}^t} \geq 0$   for any $t \geq 0$.

\item Likewise, we conclude the intra-cluster collision avoidance between individuals in $I$ in two steps. 
First, we have for any $i \neq j \in K$, $K \in \childCL{I,\tree}$
\begin{align}
\dot{\vect{x}_i} - \dot{\vect{x}}_j = 0, 
\end{align} 
guaranteeing that for all $t \geq 0$
\begin{align}
\norm{\vect{x}_i^t - \vect{x}_j^t}^2 = \norm{\vect{x}_i^0 - \vect{x}_j^0}^2 > \prl{r_i + r_j}^2.
\end{align}
Second, for any $i \in K$, $j \in I \setminus K$ and $K \in \childCL{I,\tree}$  we have
\begin{align}
\dot{\vect{x}}_i - \dot{\vect{x}}_j = 2 \fbeta{I}{\vectbf{x}} \frac{\ctrdsep{K\!}{\vectbf{x}}}{\norm{\ctrdsep{K\!}{\vectbf{x}}}} ,
\end{align}
yielding 
\begin{align}
\frac{d}{dt} \norm{\vect{x}_i \sqz{-} \vect{x}_j}^2 & \sqz{=} 2 \underbrace{\fbeta{I}{\vectbf{x}}}_{\geq 0} \underbrace{\vectprod{\prl{\vect{x}_i \sqz{-} \vect{x}_j}}{\frac{\ctrdsep{K\!}{\vectbf{x}}}{\norm{\ctrdsep{K\!}{\vectbf{x}}}}}}_{=\sepmag{i,K}{\vectbf{x}} + \sepmag{j,I\setminus K}{\vectbf{x}}\geq 0 }, \\
& \sqz{\geq} 0,
\end{align}
and so for $t \geq 0$
\begin{align}
\norm{\vect{x}_i^t - \vect{x}_j^t}^2 \geq \norm{\vect{x}_i^0 - \vect{x}_j^0}^2 > \prl{r_i + r_j}^2. 
\end{align}
 \qed
\end{enumerate}
\end{proof}

%%%%%%%%%%%%%%%%%%%%%%%%%%%%%%%%%%%%%%%%%%%%%%%%%%%%%%%%%
%%%%%%%%%%%%%%%%%%%%%%%%%%%%%%%%%%%%%%%%%%%%%%%%%%%%%%%%%
\subsection{Proof of \reflem{lem.DomainInvarianceRecursion}}
\label{app.DomainInvarianceRecursion}
%%%%%%%%%%%%%%%%%%%%%%%%%%%%%%%%%%%%%%%%%%%%%%%%%%%%%%%%%
%%%%%%%%%%%%%%%%%%%%%%%%%%%%%%%%%%%%%%%%%%%%%%%%%%%%%%%%%

\begin{proof}
By definition of $\domain\prl{\purt, \vectbf{b}}$ \refeqn{eq.LocalPolicyDomain},
for any $I \in \visitedcluster \setminus \purt$ we have $\vectbf{x}^0 \in \setH\prl{I}$ \refeqn{eq.SetR}   and one can verify from \reftab{tab.LocalPolicy} that for any $k \in K$ and $K \in \childCL{I,\tree}$
\begin{subequations}\label{eq.LocalSystemInvarianceRecursion}
\begin{align}
\dot{\vect{x}}_k &= \h{\purt, \vectbf{b}}\prl{\vectbf{x}}_k =  \fH\prl{\vectbf{x}, \vectbf{u}, I}_k + \vect{v}_I, \\
& = \vect{u}_k + 2\falpha{I}{\vectbf{x}, \vectbf{u}} \frac{\card{\complementLCL{K}{\tree}}}{\card{I}} \frac{\ctrdsep{K\!}{\vectbf{x}}}{\norm{\ctrdsep{K\!}{\vectbf{x}}}} + \vect{v}_I,
\end{align} 
\end{subequations}
for some $\vectbf{u} \in \prl{\R^d}^{\indexset}$ and $\vect{v}_I \in \R^d$.
Here, $\vect{v}_I$ represents the total rigid translation due ancestors of $I$ in $\tree$.

With these observations in place, we now achieve claimed results as follows:
\begin{enumerate}[label = (\roman*), leftmargin = 5mm]
\item The maintenance of cluster separation (\reflem{lem.DomainInvarianceRecursion}.(i)) can be observed, using \refeqn{eq.sepmagLie} and \refeqn{eq.LocalSystemInvarianceRecursion}, as follows: for any $k \in K$ and $K \in \childCL{I,\tree}$
\begin{align}
\frac{d}{dt} \sepmag{k,K\!}{\vectbf{x}} &= \lie_{\overrightarrow{\vectbf{u}}} \sepmag{k,K\!}{\vectbf{x}} + \falpha{I}{\vectbf{x}, \vectbf{u}},
\end{align}
and, since $\vectbf{x}^0 \in \setH\prl{I}$ and $\falpha{I}{\vectbf{x}, \vectbf{u}} \sqz{\geq}  -\lie_{\overrightarrow{\vectbf{u}}} \sepmag{k,K\!}{\vectbf{x}}$ whenever $\sepmag{k,K\!}{\vectbf{x}} = r_k + \alpha$, we have $\sepmag{k,K\!}{\vectbf{x}^t} \geq r_k + \alpha$ for all $t \geq 0$.

\item The inter-cluster collision avoidance (\reflem{lem.DomainInvarianceRecursion}.(ii)) directly follows from the maintenance of certain cluster separation (\reflem{lem.DomainInvarianceRecursion}.(i)) since 
\begin{align}
\begin{array}{c}
\sepmag{k,K\!}{\vectbf{x}^t}\geq  r_k, \\
\forall \, k \in K,\\
 K \sqz{\in} \childCL{I,\tree},
\end{array}
\Longrightarrow
\begin{array}{c}
\norm{\vect{x}_i^t \sqz{-} \vect{x}_j^t}^2 > \prl{r_i \sqz{+} r_j}^2\!\!, \\
\forall \, i \sqz{\in} K, j \sqz{\in} I \sqz{\setminus} K, \\
K \sqz{\in} \childCL{I,\tree}.
\end{array}
\end{align} 
\qed
\end{enumerate}
\end{proof}

%%%%%%%%%%%%%%%%%%%%%%%%%%%%%%%%%%%%%%%%%%%%%%%%%%%%%%
%%%%%%%%%%%%%%%%%%%%%%%%%%%%%%%%%%%%%%%%%%%%%%%%%%%%%%
\subsection{Proof of \reflem{lem.PreparesRelationSingleton}}
\label{app.PreparesRelationSingleton}
%%%%%%%%%%%%%%%%%%%%%%%%%%%%%%%%%%%%%%%%%%%%%%%%%%%%%%
%%%%%%%%%%%%%%%%%%%%%%%%%%%%%%%%%%%%%%%%%%%%%%%%%%%%%%

\begin{proof}
Since $\setA\prl{I} = \stratum{\tree}$ for any singleton cluster  $I \in \cluster{\tree}$, we have from \refeqn{eq.SetB} that $\setB\prl{I, -1} = \setB\prl{I, +1} = \stratum{\tree}$ for any singleton cluster $I \in \cluster{\tree}$.
Hence, by definition \refeqn{eq.LocalPolicyDomain}, the first part of the result holds.

Likewise, using \refeqn{eq.Priority}, one can observe the second part of the result  because the binary vectors $\vectbf{b}$ and $\hat{\vectbf{b}}$ only differ at a singleton cluster $D \in \purt$ where $\vect{b}_D = -1$ and $\vect{b}'_D = +1$. \qed
\end{proof}

%%%%%%%%%%%%%%%%%%%%%%%%%%%%%%%%%%%%%%%%%%%%%%%%%%%%%%%%%%
%%%%%%%%%%%%%%%%%%%%%%%%%%%%%%%%%%%%%%%%%%%%%%%%%%%%%%%%%%
\subsection{Proof of \reflem{lem.PreparesRelationNegative}}
\label{app.PreparesRelationNegative}
%%%%%%%%%%%%%%%%%%%%%%%%%%%%%%%%%%%%%%%%%%%%%%%%%%%%%%%%%%
%%%%%%%%%%%%%%%%%%%%%%%%%%%%%%%%%%%%%%%%%%%%%%%%%%%%%%%%%%

\begin{proof}
For any nonsingleton $I \in \purt$ with $\vect{b}_I = -1$, one can verify from \reftab{tab.LocalPolicy} that for any $k \in K$ and $K \in \childCL{I,\tree}$
\begin{subequations} \label{eq.SystemModelPRN}
\begin{align}
\dot{\vect{x}}_k &= \h{\purt, \vectbf{b}}\prl{\vectbf{x}}_k = \fS\prl{\vectbf{x}, \vectbf{u}, I}, \\
 &=  - \ctrd{\vectbf{x} \sqz{-} \vectbf{y} | I} \sqz{+} 2 \fbeta{I}{\vectbf{x}} \! \frac{\card{\complementLCL{K}{\tree}}}{\card{I}} \! \frac{\ctrdsep{K\!}{\vectbf{x}}}{\norm{\ctrdsep{K\!}{\vectbf{x}}}} \sqz{+} \vect{v}_I, \!\!\!
\end{align}  
\end{subequations}
for some $\vectbf{u} \in \prl{\R^d}^{\indexset}$ and $\vect{v}_I \in \R^d$.

Accordingly, using \refeqn{eq.sepmagLie} and \refeqn{eq.Beta}, we obtain that
\begin{align}
\frac{d}{dt} \sepmag{k,K\!}{\vectbf{x}} &\sqz{=} \fbeta{I}{\vectbf{x}}  \geq  - \sepmag{k, K\!}{\vectbf{x}} + r_k + \beta.
\end{align}
Hence, a trajectory $\vectbf{x}^t$  of $\h{\purt, \vectbf{b}}$ starting at any $\vectbf{x}^0 \in \domain\prl{\purt, \vectbf{b}}$ satisfies 
\begin{align}
\sepmag{k,K\!}{\vectbf{x}^t} \geq e^{-t} \sepmag{k,K\!}{\vectbf{x}^0} + \prl{1 - e^{-t}} \prl{r_k + \beta},
\end{align}
for all $t \geq 0$.
Thus, since $\beta > \alpha > 0$ and  $\frac{d}{dt} \sepmag{k,K\!}{\vectbf{x}} > 0$ whenever $\sepmag{k,K\!}{\vectbf{x}} < r_k + \beta$, using LaSalle's Invariance Principle \cite{LaSalle_1976}, one can conclude that the local policy $\h{\purt, \vectbf{b}}$ asymptotically steers all the configurations in its domain  $\domain\prl{\purt, \vectbf{b}}$ to a subset $\setH^{\beta}\prl{I}$ of the interior $\setHI\prl{I}$ of $\setH\prl{I}$ \refeqn{eq.SetR},
{\small
\begin{align}
\setH^{\beta}\!\prl{I} \; &\sqz{\sqz{\ldf}} \crl{ \vectbf{x} \sqz{\in} \stratum{\tree} \! \Big| \sepmag{k,K\!}{\vectbf{x}} \sqz{\geq} r_k \sqz{+} \beta, \forall k \sqz{\in} K, K \sqz{\in}\childCL{I,\tree} \!} \!,  \nonumber \\
& \subset \setH\prl{I}.
\end{align}
}%
In particular, since $\beta > \alpha$, the system in \refeqn{eq.SystemModelPRN}  starting at any configuration in $\domain\prl{\purt, \vectbf{y}}$ enters $\setH\prl{I}$ in finite time.

\medskip

\noindent Now  observe from  \refeqn{eq.LocalPolicyDomain} and \refeqn{eq.Priority} that
\begin{align}
\priority\prl{\purt', \vectbf{b}'} &= \priority \prl{\purt, \vectbf{b}}  + \underbrace{\card{I} ^2 \sqz{-} \!\!\! \!\sum_{D \in \childCL{I,\tree}} \!\!\!\! \card{D}^2}_{> 0}, \nonumber \\
& > \priority \prl{\purt, \vectbf{b}},
\end{align}
and
\begin{align}
\domain\prl{\purt', \vectbf{b}'} & = \domain\prl{\purt, \vectbf{b}} \cap \setH\prl{I}, \\
%& \supset \domain\prl{\purt, \vectbf{b}} \cap \setHI\prl{I},\\
& \supset \domain\prl{\purt, \vectbf{b}} \cap \setH^{\beta}\prl{I}.
\end{align}
Thus, since $\setH^{\beta}\prl{I} \subset \setHI\prl{I} $ and its domain  $\domain\prl{\purt, \vectbf{b}}$ is positively invariant (\refprop{prop.DomainInvariance}), $\h{\purt, \vectbf{b}}$ prepares $\h{\purt', \vectbf{b}'}$  in finite time, and the result follows.  \qed
\end{proof}

%%%%%%%%%%%%%%%%%%%%%%%%%%%%%%%%%%%%%%%%%%%%%%%%%%%%%%%%%%%%%%
%%%%%%%%%%%%%%%%%%%%%%%%%%%%%%%%%%%%%%%%%%%%%%%%%%%%%%%%%%%%%%
\subsection{Proof of \reflem{lem.PreparesRelationPositive}}
\label{app.PreparesRelationPositive}
%%%%%%%%%%%%%%%%%%%%%%%%%%%%%%%%%%%%%%%%%%%%%%%%%%%%%%%%%%%%%%
%%%%%%%%%%%%%%%%%%%%%%%%%%%%%%%%%%%%%%%%%%%%%%%%%%%%%%%%%%%%%%

\begin{proof}
%Since $\purt$ is a partition of $\indexset$ compatible with $\tree$, i.e. $\purt \subset \cluster{\tree}$, there always exists a cluster $I \in \tree$ such that $\childCL{I,\tree} \subset \purt$.   

Since $\childCL{I,\tree} \subset \purt$  and $\vect{b}_{K} = +1$ for any $K \in \childCL{I,\tree}$, every child $K \in \childCL{I,\tree}$ of $I$ in $\tree$ satisfies the base condition in \reftab{tab.LocalPolicy}.2)-4) whereas cluster $I$ satisfies the recursion conditions in \reftab{tab.LocalPolicy}.9)-12).
Hence, using \reftab{tab.LocalPolicy}, one can verify that for any $k \in K $ and $K \in \childCL{I, \tree}$ 
\begin{subequations}\label{eq.SystemModelPRP}
\begin{align}
\hspace{-1.5mm}\dot{\vect{x}}_{k} &= \h{\purt, \vectbf{b}}\prl{\vectbf{x}}_k = (\fH \circ \fA)\prl{\vectbf{x}, \vectbf{u}, I}_k + \vect{v}_I, \\ 
&=  - \prl{\vect{x}_k \sqz{-} \vect{y}_k} \sqz{+} 2\falpha{I}{\vectbf{x},\vectbf{u}}\! \frac{\card{\complementLCL{K}{\tree}}}{\card{I}}\frac{\ctrdsep{K\!}{\vectbf{x}}}{\norm{\ctrdsep{K\!}{\vectbf{x}}}} \sqz{+} \vect{v}_I, \!\!\!
\end{align}
\end{subequations}
for some $\vectbf{u} \in \prl{\R^d}^{\indexset}$  and $\vect{v}_I \in \R^d$.

\medskip

We now show in three steps that $\h{\purt, \vectbf{b}}$ asymptotically steers (almost) all configuration in its domain $\domain\prl{\purt, \vectbf{b}}$ to 
\begin{align}
\setG\prl{I} \sqz{\ldf} \! \left \{ \Big. \right .
\!\! \vectbf{x} \sqz{\in} \stratum{\tree} \! \Big | & \scalebox{1.1}{$\frac{\ctrdsep{K}{\vectbf{x}}}{\norm{\ctrdsep{K}{\vectbf{x}}}} \sqz{=} \frac{\ctrdsep{K}{\vectbf{y}}}{\norm{\ctrdsep{K}{\vectbf{y}}}}$}, \norm{\ctrdsep{K\!}{\vectbf{x}}} \sqz{\geq} \norm{\ctrdsep{K\!}{\vectbf{y}}}\!, \nonumber \\
& \hspace{8mm} \vect{x}_k \sqz{-} \ctrd{\vectbf{x}|K} \sqz{=} \vect{y}_k \sqz{-} \ctrd{\vectbf{y}|K}\!, \nonumber \\ 
& \hspace{12mm}\forall \; k \sqz{\in} K, K \sqz{\in} \childCL{I,\tree} \!\! \left. \Big. \right \}\!, \!
\end{align}
which is a subset of $\setAhat\prl{I_L, I_R}$ \refeqn{eq.SetAHat} associated with children clusters $\crl{I_L,I_R} = \childCL{I,\tree}$
because  for any $\vectbf{x} \sqz{\in} \setG\prl{I}$ and  $i \sqz{\in} K$, $j \sqz{\in} I\setminus K$ and $K \sqz{\in} \childCL{I,\tree}$
\begin{align}
\hspace{-2mm}\lie_{\overrightarrow{\vectbf{y}}}\scalebox{1}{$\frac{1}{2}$}\norm{\vect{x}_i - \vect{x}_j}^2 & \sqz{=} \, \vectprod{\prl{\vect{x}_i \sqz{-} \vect{x}_j}}{\prl{\vect{y}_i \sqz{-} \vect{y}_j}},\\
 & \sqz{=} \, \vectprod{\underbrace{\prl{\big.\vect{x}_i \sqz{-} \vect{x}_j \sqz{-} \ctrdsep{K\!}{\vectbf{x}} \sqz{+} \ctrdsep{K\!}{\vectbf{y}}\!}\!}_{= \vect{y}_i - \vect{y_j}}} {\!\prl{\vect{y}_i \sqz{-} \vect{y}_j}}  \nonumber 
\\
& \hspace{5mm}\sqz{+}  \underbrace{\vectprod{\underbrace{\prl{\ctrdsep{K\!}{\vectbf{x}} \sqz{-} \ctrdsep{K\!}{\vectbf{y}}}}_{\substack{ =\epsilon \;\ctrdsep{K}{\vectbf{y}} \text{ for some } \epsilon \geq 0}}\!\!} {\prl{\vect{y}_i \sqz{-} \vect{y}_j}}}_{\geq 0 \text{ since } \vectbf{y} \in \stratum{\tree}}, \!\! 
\\
&\!\geq \norm{\vect{y}_i - \vect{y}_j}^2 > \prl{r_i + r_j}^2,
\end{align}
and for any $\vectbf{x} \sqz{\in} \setG\prl{I}$ and $k \in K$ and $K \in \childCL{I,\tree}$
\begin{align}
\hspace{-2mm}\lie_{\overrightarrow{\vectbf{y}}} \vectprod{\prl{\vect{x}_k \sqz{-} \ctrdmid{K\!}{\vectbf{x}}\!}\!}{\!\ctrdsep{K\!}{\vectbf{x}}}   & = \!\! \underbrace{\vectprod{\prl{\vect{y}_k \sqz{-} \ctrdmid{K\!}{\vectbf{y}}\!}\!}{\ctrdsep{K\!}{\vectbf{x}}}}_{\geq 0 \text{ since } \vectbf{x} \in \setG\prl{I} \text{ and } \vectbf{y} \in \stratum{\tree}} \nonumber 
\\
& \hspace{6mm} + \underbrace{\vectprod{\prl{\vect{x}_k \sqz{-} \ctrdmid{K\!}{\vectbf{x}} \!}\!}{\ctrdsep{K\!}{\vectbf{y}}}}_{\geq 0 \text{ since } \vectbf{x} \in \setG\prl{I}} , \!\!\! \\
& \geq 0.
\end{align}
Likewise, one can observe that $\setG\prl{I} \sqz{\cap} \setH\prl{I}$ is a subset of the interior of $\setAhat\prl{I_L, I_R}$, i.e.  $\setG\prl{I} \sqz{\cap} \setH\prl{I} \sqz{\subset} \setAhatI\prl{I_L, I_R}$.

\medskip

First, using \refeqn{eq.SystemModelPRP}, we obtain that for any $K \in \childCL{I,\tree}$
{
\begin{align}
\hspace{-1mm}\frac{d}{dt}  \frac{\vectprod{\ctrdsep{K\!}{\vectbf{x}}}{\ctrdsep{K\!}{\vectbf{y}}}}{\norm{\ctrdsep{K\!}{\vectbf{x}}}\norm{\ctrdsep{K\!}{\vectbf{y}}}} 
&= 
\frac{\norm{\ctrdsep{K\!}{\vectbf{y}}}}{\norm{\ctrdsep{K\!}{\vectbf{x}}}} \sqz{-} \frac{\vectprod{\ctrdsep{K\!}{\vectbf{x}}}{\ctrdsep{K\!}{\vectbf{y}}}}{\norm{\ctrdsep{K\!}{\vectbf{x}}}^3  \norm{\ctrdsep{K\!}{\vectbf{y}}}}, \nonumber\\
&  =  \frac{\norm{\ctrdsep{K\!}{\vectbf{y}}}}{\norm{\ctrdsep{K\!}{\vectbf{x}}}} \prl{\! \! 1 \sqz{-}\frac{\vectprod{\ctrdsep{K\!}{\vectbf{x}}}{\ctrdsep{K\!}{\vectbf{y}}}}{\norm{\ctrdsep{K\!}{\vectbf{x}}}\norm{\ctrdsep{K\!}{\vectbf{y}}}}\!}^2\!\!\!\!, \!\! \nonumber \\
&  \geq 0
\end{align}
}%
where the equality only holds if $\frac{\ctrdsep{K}{\vectbf{x}}}{\norm{\ctrdsep{K}{\vectbf{x}}}} = \pm \frac{\ctrdsep{K}{\vectbf{y}}}{\norm{\ctrdsep{K}{\vectbf{y}}}}$.
Thus, $\h{\purt, \vectbf{b}}$ asymptotically aligns the separating hyperplane normals  of complementary clusters $\childCL{I,\tree}$ of (almost) any configuration in $\domain\prl{\purt, \vectbf{b}}$ with the desired ones. 
Note that the set of configurations $\vectbf{x} \in \domain\prl{\purt, \vectbf{b}}$ with $\frac{\ctrdsep{K}{\vectbf{x}}}{\norm{\ctrdsep{K}{\vectbf{x}}}} = -\frac{\ctrdsep{K}{\vectbf{y}}}{\norm{\ctrdsep{K}{\vectbf{y}}}}$ has  measure zero and are saddle points.

\medskip

Next, let $\vectbf{x} \in \domain\prl{\purt, \vectbf{b}}$ with $\frac{\ctrdsep{K}{\vectbf{x}}}{\norm{\ctrdsep{K}{\vectbf{x}}}} = \frac{\ctrdsep{K}{\vectbf{y}}}{\norm{\ctrdsep{K}{\vectbf{y}}}}$  for all $K \in \childCL{I,\tree}$. 
Then,  using \refeqn{eq.SystemModelPRP}, observe that     
\begin{align}
\hspace{-2mm}\frac{d}{dt}\norm{\ctrdsep{K\!}{\vectbf{x}}}^2 &= -2 \norm{\ctrdsep{K\!}{\vectbf{x}}}^2 + 2 \vectprod{\ctrdsep{K\!}{\vectbf{x}}}{\ctrdsep{K\!}{\vectbf{y}}}  \nonumber \\
& \hspace{15mm}+ 4\underbrace{\falpha{I\!}{\vectbf{x}, \vectbf{u}}}_{\geq \; 0} \norm{\ctrdsep{K\!}{\vectbf{x}}}, \\
& \geq -2 \norm{\ctrdsep{K\!}{\vectbf{x}}}^2 + 2 \vectprod{\ctrdsep{K\!}{\vectbf{x}}}{\ctrdsep{K\!}{\vectbf{y}}}\!, \!\! \!\\
&= -2 \norm{\ctrdsep{K\!}{\vectbf{x}}}\prl{\norm{\ctrdsep{K\!}{\vectbf{x}}} \sqz{-} \norm{\ctrdsep{K\!}{\vectbf{y}}}}\!. \!\!
\end{align}
Hence, $\frac{d}{dt}\norm{\ctrdsep{K\!}{\vectbf{x}}}^2 \sqz{>} 0$  whenever $\norm{\ctrdsep{K\!}{\vectbf{x}}} \sqz{<} \norm{\ctrdsep{K\!}{\vectbf{y}}}_2$. 
Thus, the stable configurations of  $\h{\purt, \vectbf{b}}$ also satisfies  $\norm{\ctrdsep{K\!}{\vectbf{x}}} \geq \norm{\ctrdsep{K\!}{\vectbf{y}}}$.

\medskip

Finally, we have from \refeqn{eq.SystemModelPRP} that for any $k \sqz{\in} K$, $K \sqz{\in} \childCL{K,\tree}$
\begin{align}
\hspace{-1mm}\frac{d}{dt} \prl{\big. \vect{x}_k \sqz{-} \ctrd{\vectbf{x}|K}\!} \sqz{=} -\prl{\big. \vect{x}_k \sqz{-} \ctrd{\vectbf{x}|K}\!} \sqz{+} \prl{\big. \vect{y}_k \sqz{-} \ctrd{\vectbf{y}|K}\!}\!, \!\!
\end{align}
and so a trajectory $\vectbf{x}^t$ of $\h{\purt, \vectbf{b}}$ starting at any $\vectbf{x}^0 \in \domain\prl{\purt, \vectbf{b}}$ satisfies 
{\small
\begin{align}
\hspace{-1.5mm}\vect{x}_k^t \sqz{-} \ctrd{\vectbf{x}^t|K} &\sqz{=} e^{-t}\prl{\big. \vect{x}^0_k \! \sqz{-} \ctrd{\vectbf{x}^0|K}\!}  \sqz{+} \prl{1 \sqz{-} e^{-t}}\!\prl{\big. \vect{y}_k \! \sqz{-} \ctrd{\vectbf{y}|K}\!}\!\!. \!\!\!
\end{align}
}%
for all $t \geq 0$. 
Hence, the centroidal displacements, $\vect{x}_k^t - \ctrd{\vectbf{x}^t|K}$, of any configuration $\vectbf{x} \in \domain\prl{\purt, \vectbf{b}}$ asymptotically matches the centroidal displacement, $\vect{y}_k - \ctrd{\vectbf{y}|K}$, of the desired configuration $\vectbf{y}$.  

Thus, it follows from LaSalle's Invariance Principle \cite{LaSalle_1976} that (almost) all configurations in the domain  $\domain\prl{\purt, \vectbf{b}}$ of $\h{\purt, \vectbf{b}}$ asymptotically reach $\setG\prl{I}$. 

\bigskip

Now observe from \refeqn{eq.LocalPolicyDomain} and \refeqn{eq.Priority} that
\begin{align}
\priority\prl{\purt', \vectbf{b}'} &= \priority\prl{\purt, \vectbf{b}} \sqz{+} \underbrace{\card{I}^2 \sqz{-} \!\!\! \sum_{D \in \childCL{I,\tree}}\!\! \!\! \card{D}^2}_{> \; 0}\!\!, \nonumber \\
& >  \priority\prl{\purt, \vectbf{b}},
\end{align}
and
\begin{subequations}\label{eq.AttractivePrepares}
\begin{align}
\domain(\purt', \vectbf{b}') &\supset \domain\prl{\purt, \vectbf{b}} \cap \setAhat\prl{I_L, I_R}, \\
& \supset \domain\prl{\purt, \vectbf{b}} \cap \setG\prl{I},
\end{align}
\end{subequations}
which follows from  that $\setG\prl{I} \subset \setAhat\prl{I_L, I_R}$ and $\setA\prl{I} =  \setA\prl{I_L} \cap \setA\prl{I_R} \cap \setAhat\prl{I_L, I_R}$ \refeqn{eq.SetARecursion}.

Thus, one can conclude from \refeqn{eq.AttractivePrepares} and $\setG\prl{I} \cap \setH\prl{I} \subset \setAhatI\prl{I_L, I_R} $ that the disks starting at almost any configuration in the positively invariant  $\domain\prl{\purt, \vectbf{b}}$  reach $\domain(\purt', \vectbf{b}') $ in finite time, and this completes the proof.\qed
\end{proof}

% use section* for acknowledgement
%%%%%%%%%%%%%%%%%%%%%%%%%%%%%%%%%%%%%%
%%%%%%%%%%%%%%%%%%%%%%%%%%%%%%%%%%%%%%
\section*{Acknowledgment}
%%%%%%%%%%%%%%%%%%%%%%%%%%%%%%%%%%%%%%
%%%%%%%%%%%%%%%%%%%%%%%%%%%%%%%%%%%%%%

This work was supported in part by AFOSR under the CHASE MURI FA9550-10-1-0567 and in part by ONR under the HUNT MURI N00014–07–0829.
The authors would like to thank Yuliy Baryshnikov and Fred Cohen for discussions on the topology of configuration spaces.

% Can use something like this to put references on a page
% by themselves when using endfloat and the captionsoff option.
\ifCLASSOPTIONcaptionsoff
  \newpage
\fi

% trigger a \newpage just before the given reference
% number - used to balance the columns on the last page
% adjust value as needed - may need to be readjusted if
% the document is modified later
%\IEEEtriggeratref{8}
% The "triggered" command can be changed if desired:
%\IEEEtriggercmd{\enlargethispage{-5in}}

% references section

% can use a bibliography generated by BibTeX as a .bbl file
% BibTeX documentation can be easily obtained at:
% http://www.ctan.org/tex-archive/biblio/bibtex/contrib/doc/
% The IEEEtran BibTeX style support page is at:
% http://www.michaelshell.org/tex/ieeetran/bibtex/
\bibliographystyle{IEEEtran}
% argument is your BibTeX string definitions and bibliography database(s)
\bibliography{IEEEabrv,hiernav}

\vfill

%\enlargethispage{-10in}

%% if you will not have a photo at all:
%\begin{IEEEbiographynophoto}{John Doe}
%Biography text here.
%\end{IEEEbiographynophoto}

% insert where needed to balance the two columns on the last page with
% biographies
%\newpage

%\begin{IEEEbiographynophoto}{Jane Doe}
%Biography text here.
%\end{IEEEbiographynophoto}

% You can push biographies down or up by placing
% a \vfill before or after them. The appropriate
% use of \vfill depends on what kind of text is
% on the last page and whether or not the columns
% are being equalized.

%\vfill

% Can be used to pull up biographies so that the bottom of the last one
% is flush with the other column.
%\enlargethispage{-5in}

% that's all folks
\end{document}